\documentclass[10pt]{article}
\usepackage[left=1in,right=1in,top=1in,bottom=1in]{geometry}
\usepackage{parskip}
\usepackage{setspace}
\usepackage{soul}
\usepackage{pifont}
\usepackage{comment}

\usepackage[table]{xcolor}
\definecolor{darkred}{rgb}{0.6,0,0}
\definecolor{darkgreen}{rgb}{0,0.5,0}
\definecolor{darkblue}{rgb}{0,0,0.5}
\definecolor{notered}{HTML}{d62728}
\definecolor{noteblue}{HTML}{1f77b4}
\definecolor{notegreen}{HTML}{2ca02c}
\definecolor{noteorange}{HTML}{ff7f0e}
\usepackage[framemethod=TikZ]{mdframed}
\usepackage[many]{tcolorbox}
\usepackage{tabto}
\usepackage{makecell}
\usepackage{multirow} % set multiple rows
\usepackage{multicol}
\usepackage{booktabs}
\usepackage{threeparttable}
\usepackage{enumitem}
\usepackage[linesnumbered,ruled,vlined]{algorithm2e}
\usepackage{titletoc}
\usepackage{authblk}
\usepackage{hyperref}
\hypersetup{
	colorlinks=true,
	linkcolor={darkred},
	citecolor={darkgreen}
}
\usepackage{graphicx}
\usepackage{wrapfig}
\usepackage{subcaption}
\usepackage{tikz}
\usetikzlibrary{tikzmark}
\newcommand{\tikzhighlight}[4]{
	\tikz[overlay, remember picture]{
		\fill[#2, opacity=0.5] 
		([xshift={#3[0]}, yshift={#3[1]}]pic cs:#1-start) rectangle ([xshift={#4[0]}, yshift={#4[1]}]pic cs:#1-end);
	}
}
\newcommand{\dottedbox}[1]{%
	\tikz[baseline=(X.base)] \node[draw=black, dash pattern=on 2pt off 2pt, line width=0.6pt, rectangle, inner xsep=2pt, inner ysep=3pt] (X) {#1};
}
\usepackage[round]{natbib}
\usepackage{amsthm}
\usepackage{amsmath,amsfonts,amssymb,bm,bbm}
\usepackage{mathtools}
\usepackage{mathrsfs} % script fonts
\usepackage{nicefrac}
\usepackage{color}

\def\floor#1{\lfloor #1 \rfloor}
\def\1{\bm{1}} % for vector, there is another definition with the same meaning

\def\rvc{{\mathbf{c}}}
\def\rvd{{\mathbf{d}}}

\def\rvm{{\mathbf{m}}}

\def\rvp{{\mathbf{p}}}

\def\rvs{{\mathbf{s}}}

\def\rvw{{\mathbf{w}}}
\def\rvx{{\mathbf{x}}}
\def\rvy{{\mathbf{y}}}
\def\rvz{{\mathbf{z}}}

\def\vzero{{\bm{0}}}

\DeclareMathAlphabet{\mathsfit}{\encodingdefault}{\sfdefault}{m}{sl}
\SetMathAlphabet{\mathsfit}{bold}{\encodingdefault}{\sfdefault}{bx}{n}

\def\gF{{\mathcal{F}}}

\def\gN{{\mathcal{N}}}
\def\gO{{\mathcal{O}}}

\def\gS{{\mathcal{S}}}

\newcommand{\E}{\mathbb{E}}

\newcommand{\R}{\mathbb{R}}

\newcommand{\norm}[1]{\left\lVert#1\right\rVert} % norm
\newcommand{\normsq}[1]{\left\Vert #1 \right\Vert^2} % norm square
\newcommand{\innerprod}[1]{\left\langle #1 \right\rangle} % inner product
\newcommand{\abs}[1]{\left\lvert#1\right\rvert} % absoluate value

\newcommand{\Term}[2]{\text{T$_{#1}$ in #2}}

\usepackage{thmtools}

\declaretheoremstyle[
headfont=\normalfont\bfseries,
notefont=\mdseries, notebraces={(}{)},
bodyfont=\normalfont,
preheadhook=\vspace{1ex},
postheadhook=,
postheadspace=0.5em,
spaceabove=1pt,
spaceabove=1pt,
mdframed={
	roundcorner=2pt,
	hidealllines=true,
	backgroundcolor={gray!20},
	innerleftmargin=2pt,
	innerrightmargin=2pt}
]{shaded}

\theoremstyle{definition}

\declaretheorem[style=shaded]{definition}
\declaretheorem[style=shaded]{theorem}

\declaretheorem[style=shaded]{assumption}

\declaretheorem[style=shaded]{lemma}

\declaretheorem[style=shaded]{example}

\declaretheorem[style=shaded, name=Definition, numbered=no]{definition*}
\declaretheorem[style=shaded, name=Theorem, numbered=no]{theorem*}
\declaretheorem[style=shaded, name=Proposition, numbered=no]{proposition*}
\declaretheorem[style=shaded, name=Lemma, numbered=no]{lemma*}

\usepackage[noabbrev]{cleveref}

\title{\LARGE \textbf{A Unified Analysis of Stochastic Gradient Descent with Arbitrary Data Permutations and Beyond}}

\author[1]{Yipeng Li}
\author[2]{Xinchen Lyu\thanks{Corresponding author.}}
\author[3]{Zhenyu Liu}

\affil[1,2]{\small{Beijing University of Posts and Telecommunications, Beijing, China}. \texttt{\{liyipeng, lvxinchen\}@bupt.edu.cn}}
\affil[3]{\small{Shenzhen International Graduate School, Tsinghua University, Shenzhen, China. \texttt{zhenyuliu@sz.tsinghua.edu.cn}}}

\date{}

\begin{document}
%%%%%% Title page %%%%%%

\maketitle

\begin{abstract}
	We aim to provide a unified convergence analysis for permutation-based Stochastic Gradient Descent (SGD), where data examples are permuted before each epoch. By examining the relations among permutations, we categorize existing permutation-based SGD algorithms into four categories: Arbitrary Permutations, Independent Permutations (including Random Reshuffling), One Permutation (including Incremental Gradient, Shuffle One and Nice Permutation) and Dependent Permutations (including GraBs \citealp{lu2022grab,cooper2023coordinating}). Existing unified analyses failed to encompass the Dependent Permutations category due to the inter-epoch dependencies in its permutations. In this work, we propose a general assumption that captures the inter-epoch permutation dependencies. Using the general assumption, we develop a unified framework for permutation-based SGD with arbitrary permutations of examples, incorporating all the aforementioned representative algorithms. Furthermore, we adapt our framework on example ordering in SGD for client ordering in Federated Learning (FL). Specifically, we develop a unified framework for regularized-participation FL with arbitrary permutations of clients.
\end{abstract}

\section{Introduction}

We study the finite-sum minimization problem
\begin{align*}\textstyle
	\min_{\rvx \in \R^d} \left[ f(\rvx) \coloneqq \frac{1}{N}\sum_{n=0}^{N-1} f_n (\rvx)\right],%\label{eq:problem}
\end{align*}
where each $f_n: \R^d \to \R$ is assumed to be differentiable. One popular way to solve this problem is Stochastic Gradient Descent (SGD). It updates the parameter vector iteratively according to the rule
\begin{align*}
	\rvx^{n+1} = \rvx^n - \gamma \nabla f_{\pi(n)} (\rvx^n)\,,
\end{align*}
where $\gamma$ denotes the step size and $\pi(n)$ denotes the index of the local objective function at iteration $n$. For classic SGD (cSGD), $\pi(n)$ is chosen uniformly with replacement from $\{0,1,\ldots, N-1\}$; for permutation-based SGD, $\pi(n)$ is the $(n+1)$-th element of a permutation $\pi$ of $\{0,1,\ldots, N-1\}$. Permutation-based SGD is more common in practice \citep{bottou2012stochastic}, and thus attracts much attention recently \citep{ahn2020sgd,mishchenko2020random,nguyen2021unified}. It is also the focus of this work. In what follows, unless otherwise stated, ``SGD'' refers to ``permutation-based SGD''.

The convergence rate of permutation-based SGD is determined by example orders. Thus, to study it, we need a measure of the quality of example orders. Note that we say that an example order is good if it leads to a high convergence rate of permutation-based SGD, and vice versa. For a small finite step size $\gamma$, the cumulative updates in any epoch $q$ are
\begin{align}
	\rvx_{q+1} - \rvx_q &\approx - \gamma N \nabla f(\rvx_q) + \gamma^2 \sum_{n=0}^{N-1}\sum_{i<n} \nabla^2 f_{\pi(n)}(\rvx_q) \nabla f_{\pi(i)}(\rvx_q) \nonumber\\ %\label{eq:order error analysis}\\
	&\approx \underbrace{- \gamma N \nabla f(\rvx_q) + \gamma^2 \sum_{n=0}^{N-1}\sum_{i<n} \nabla^2 f_{\pi(n)}(\rvx_q) \nabla f(\rvx_q)}_{\text{optimization vector}}+ \underbrace{\gamma^2 \sum_{n=0}^{N-1}\sum_{i<n} \nabla^2 f_{\pi(n)}(\rvx_q) \left( \nabla f_{\pi(i)}(\rvx_q) -\nabla f(\rvx_q) \right)}_{\text{error vector}}\nonumber,% + \gO( \gamma^3 N^3)\nonumber.
\end{align}
where the first equation is from \citet[Eq.~13]{smith2021origin} (we replace $=$ with $\approx$ as we omit $\gO( \gamma^3 N^3)$), and it can be proved by Taylor expansion (see Appendix~\ref{apx-subsec:order error}). Here, we additionally assume that each $f_n$ is twice differentiable. The optimization vector is beneficial; the error vector is detrimental and depends on the order of examples. Thus, the goal is to suppress the error vector (for instance, we use Lebesgue $2$-norm for both vectors and matrices):
\begin{align*}
	\norm{\text{Error vector}}_{}
	&= \gamma^2 \norm{\sum_{n=0}^{N-1}\sum_{i<n} \nabla^2 f_{\pi(n)}(\rvx_q) \left( \nabla f_{\pi(i)}(\rvx_q) -\nabla f(\rvx_q) \right)}_{} \\
	&\leq \gamma^2 \sum_{n=0}^{N-1} \norm{\nabla^2 f_{\pi(n)}(\rvx_q)}_{} \norm{ \sum_{i<n}\left( \nabla f_{\pi(i)}(\rvx_q) -\nabla f(\rvx_q) \right) }_{} \\
	&\leq \gamma^2 L N \bar \phi_q\,,
\end{align*}
where the last inequality is due to $L$-smoothness (see Definition~\ref{def:smoothness}) and
\begin{align}\textstyle
	\bar \phi_q \coloneq \underset{n\in[N]}{\max} \norm{ \sum_{i=0}^{n-1} \left(\nabla f_{\pi(i)}(\rvx_q) - \nabla f(\rvx_q)\right) }.\label{eq:intro:order error}
\end{align}
This implies \textit{the order error $\bar \phi_q$ can be used as a measure of the quality of example orders}: a smaller $\bar \phi_q$ means a faster convergence rate, and a better example order, and vice versa. Even though the order error is proposed as early as \citet{lu2021general}, where the authors justify its validity on synthetic experiments empirically, the rationale behind it (that is, the above analysis) has not been well understood until this work.

As shown in Table~\ref{tab:SGD}, based on \textit{the relations among permutations}, we classify existing permutation-based SGD algorithms into the following categories:
\begin{itemize}%[leftmargin=3ex]
	\item Arbitrary Permutations (AP). Permutations are generated without any specific structure, allowing for completely arbitrary permutation in each epoch.
	\item Independent Permutations (IP). All the permutations are generated independently. This category includes Random Reshuffling (RR), where permutations are generated independently and randomly for each epoch. It also includes Greedy Ordering \citep{lu2021general,mohtashami2022characterizing}, where permutations are generated by a greedy algorithm.
	\item One Permutation (OP). The initial (first-epoch) permutation is used repeatedly for all the subsequent epochs. In particular, when the initial permutation is arbitrary, it is called Incremental Gradient (IG); when the initial permutation is random, it is called Shuffle Once (SO); when the initial permutation is designed meticulously, it is called Nice Permutation (NP).
	\item Dependent Permutations (DP). Permutations are dependent across epochs, with the order in one epoch affected by the order in previous epochs (explicitly). This category includes FlipFlop \citep{rajput2022permutationbased} and GraBs (including GraB \citealp{lu2022grab} and PairGraB \citealp{lu2022grab, cooper2023coordinating}). In particular, GraB has been proven to outperform RR \citep{lu2022grab}, and even be a theoretically optimal permutation-based SGD algorithm \citep{cha2023tighter}. See \Cref{apx-sec:algs}.
\end{itemize}
We exclude Greedy Ordering from our discussion due to its lack of practicality and theoretical justification \citep{chelidze2010greedy,lu2022grab}. We also exclude FlipFlop as its superiority (over RR) is proved on quadratic functions \citep{rajput2022permutationbased}. See Table~\ref{tab:SGD}.

For AP/RR/OP, the relation among permutations is arbitrary/independent/identical, and thus we can bound the order error for any epoch and then apply this bound for all the epochs. To deal with these cases, \citet{lu2021general} proposed one assumption (\citealt{lu2021general} consider an interval of arbitrary length, not necessarily an epoch): There exist nonnegative constants $B$ and $D$ such that for all $\rvx_q$ (the outputs of Algorithm~\ref{alg:SGD}),
\begin{align}
	\left(\bar \phi_q \right)^2 \leq B \norm{\nabla f(\rvx_q)}^2 + D\,.\label{eq:order error-raw bound}
\end{align}
By proving that this assumption (Ineq.~\ref{eq:order error-raw bound}) holds for AP, RR and SO with specific values of $B$ and $D$ (under some standard assumptions in SGD), previous works \citep{lu2021general,mohtashami2022characterizing, koloskova2024convergence} successfully incorporate them into one framework. However, none of the unified frameworks of permutation-based SGD has successfully incorporated GraBs. The main reason for the failure can be that, existing works implicitly deal with the order error $\bar \phi_q$ separately across epochs (as in Ineq.~\ref{eq:order error-raw bound}), while in GraBs, the example orders across consecutive epochs are dependent. This limitation sparked our initial motivation for this work---developing a unified framework of permutation-based SGD that includes GraBs.

To achieve this, we propose a more general assumption than Ineq.~\eqref{eq:order error-raw bound} (see Assumption~\ref{asm:order error}): There exist nonnegative constants $\{A_i\}$, $\{B_i\}$ and $D$ such that for all $\rvx_q$,
\begin{align}
	\left(\bar \phi_q \right)^2 &\leq  \sum_{i=1}^{q} A_{i}\left(\bar \phi_{q-i} \right)^2  + \sum_{i=0}^{q} B_{i}\normsq{\nabla f(\rvx_{q-i})}  + D\label{eq:order error-new bound}
\end{align}
This assumption explicitly demonstrates the \textit{dependence} between permutations across different epochs. In particular, when $A_i= 0$ and $B_i= 0$ for all $i\in \{1,2,\ldots q\}$, it reduces to Ineq.~\eqref{eq:order error-raw bound}. Our goal now is to prove that Ineq.~\eqref{eq:order error-new bound} holds for existing algorithms by identifying the relation between order errors. For instance, for OP, the main task is to establish the relation between $\bar \phi_q$ and $\bar \phi_0$ for $q\geq 1$; for GraBs, the main task is to establish the relation between $\bar \phi_q$ and $\bar \phi_{q-1}$ for $q\geq 1$. This is the key idea of our framework.

\begin{table}[ht]
	\caption{Upper bounds of permutation-based SGD (the numerical constants and polylogarithmic factors are hided). The ``Relation'' column shows the relation among permutations. The ``Upper Bound'' column shows the upper bound of $\min_{q\in \{0,1,\ldots,Q-1\}} \normsq{\nabla f(\rvx_q)}$ (see Theorem~\ref{thm:SGD}). The upper bounds of AP and RR match the prior best known upper bounds \citep{lu2021general,mishchenko2020random}. The upper bound of OP is new; see Section~\ref{subsec:cases:others} for details. The upper bound of GraBs matches that in \citet{lu2022grab}.}
	\label{tab:SGD}
	\setlength{\tabcolsep}{0.15em}
	\tikzhighlight{highlightGraB}{noteblue!30}{{+1em,0.2em}}{{8em, -2.2em}}
	\centering{\small
		\resizebox{0.8\linewidth}{!}{
			\begin{threeparttable}
				\begin{tabular}{llll}
					\toprule
					\textbf{Method}  &\textbf{Relation}  &\textbf{Upper Bound}\\
					\midrule
					
					Arbitrary Permutations  &Arbitrary &$\frac{L F_0}{Q} + \left(\frac{LF_0 N \varsigma }{ NQ}\right)^{\frac{2}{3}}$\\\midrule
					
					Independent Permutations &Independent &--\vspace{-0.5ex}\\
					\quad Random Reshuffling &Independent &$\frac{LF_0 }{Q} +  \left(\frac{LF_0\sqrt{N}\varsigma}{NQ}\right)^{\frac{2}{3}}$ \\\midrule

					\rowcolor{darkgreen!20}
					One Permutation &Identical &$\frac{L F_0}{Q} + \left(\frac{LF_0 \bar\phi_0 }{ NQ}\right)^{\frac{2}{3}}$\tnote{\textcolor{blue}{(1)}}\\\midrule

					Dependent Permutations &Dependent &--\vspace{-0.5ex}\\
					\tikzmark{highlightGraB-start}%
					\quad GraBs &Dependent &$\frac{\tilde L  F_0 + (L_{2,\infty}F_0 \varsigma)^{\frac{2}{3}}}{Q}  + \left(\frac{L_{2,\infty}F_0 \varsigma }{ NQ}\right)^{\frac{2}{3}}$\tnote{\textcolor{blue}{(2)}}\tikzmark{highlightGraB-end}\\
					\bottomrule
				\end{tabular}
				\begin{tablenotes}
					\footnotesize
					\item[1] It requires $\theta \lesssim \frac{\bar\phi_0}{LN}$ (see Assumption~\ref{asm:parameter deviation} for $\theta$). Notably, $\bar \phi_0$ depends on the initial permutation. Specifically, $\bar \phi_0 = \gO(N \varsigma)$ for IG; $\bar \phi_0 = \tilde \gO(\sqrt{N}\varsigma)$ for SO; $\bar \phi_0 = \tilde \gO(\varsigma)$ for NP.
					\item[2] For GraBs, $\tilde L = L +L_{2,\infty}+ L_{\infty}$. See Definition~\ref{def:smoothness} for $L$, $L_{2,\infty}$ and $L_{\infty}$.
				\end{tablenotes}
			\end{threeparttable}
		}
	}
\end{table}

Beyond SGD, we adapt our theory on \textit{example ordering in SGD} for \textit{client ordering in Federated Learning (FL)} \citep{mcmahan2017communication}, one of the most popular distributed machine learning paradigms. FL aims to learn from data distributed across multiple clients. In cross-device FL, only a small fraction of clients can participate in the training process simultaneously. In this work, we study client ordering in FL with regularized client participation (regularized-participation FL, Algorithm~\ref{alg:FL}), where each client participates once before any client is reused; it can be caused by diurnal variation \citep{eichner2019semi}. Compared to SGD, the main challenge stems from its partially parallel training manner, where clients update their models locally (in parallel) in a round of federated training. To address it, we propose a variant of the order error $\bar\phi$ of SGD for FL (see Definition~\ref{def:FL:order error}). With it, we develop a unified framework for regularized-participation FL with arbitrary permutation of clients, including regularized-participation FL with AP, RR, OP and GraBs (see Table~\ref{tab:FL}), which correspond to AP, RR, OP and GraBs in SGD, respectively. Among them, regularized-participation FL with GraB is introduced in this paper for the first time.

The main contributions are as follows:
\begin{itemize}%[leftmargin=3ex]
	\item Example ordering in SGD. We propose a new assumption (Assumption~\ref{asm:order error}) to bound the order error, which explicitly demonstrates the \textit{dependence} between permutations across different epochs. With this assumption, we develop a unified framework for permutation-based SGD with arbitrary permutations of examples (Section~\ref{subsec:main theorem}). At last, we prove that it includes AP, RR, OP (including IG, SO and NP) and GraBs (Section~\ref{sec:cases}). This is the first unified framework that includes GraBs.
	\item Client ordering in FL with regularized participation (Section~\ref{sec:FL}). We propose a variant of the order error of SGD for FL (\Cref{def:FL:order error}). With it, we develop a unified framework for regularized-participation FL with arbitrary permutations of clients.
\end{itemize}

\section{Convergence Analysis}

\textit{Notations.} We use $\norm{\cdot}_p$ to denote the Lebesgue $p$-norm; For simplicity, we use $\norm{\cdot}_{}$ to denote the Lebesgue $2$-norm. See Appendix~\ref{apx-sec:notations}.

\subsection{Setup}

We consider the finite-sum minimization problem
\begin{align}\textstyle
	\min_{\rvx \in \R^d} \left[ f(\rvx) \coloneqq \frac{1}{N}\sum_{n=0}^{N-1} f_n (\rvx)\right],\label{eq:problem}
\end{align}
where $d$ denotes the dimension of the parameter vector and $N$ denotes the number of the local objective functions $\{f_n\}$. Notably, for SGD, the local objective functions represent the data examples.

We study permutation-based SGD (see Algorithm~\ref{alg:SGD}). During any epoch $q$, it updates parameters as
\begin{align*}
	\rvx_{q}^{n+1} \gets \rvx_{q}^{n} - \gamma \nabla f_{\pi_q(n)}(\rvx_q^n),
\end{align*}
where $\gamma$ denotes the step size and $\pi$ denotes a permutation of $\{0,1,\ldots,N-1\}$ (at the same time, it serves as the training order of examples). At the end of each epoch, it produces the next-epoch permutation by some permuting algorithm (\Cref{alg-line:SGD:permute}).

\IncMargin{1em}
\begin{algorithm}%[H]
	\SetNoFillComment
	\SetAlgoNoEnd
	\DontPrintSemicolon
	\caption{Permutation-based SGD}
	\label{alg:SGD}
	
	\SetKwFunction{Permute}{Permute}
	\KwIn{$\pi_0$, $\rvx_0$; \textbf{Output}: $\{\rvx_q\}$}
	\For{$q = 0, 1,\ldots, Q-1$}{
		\For{$n=0,1,\ldots, N-1$}{
			$\rvx_{q}^{n+1} \gets \rvx_{q}^{n} - \gamma \nabla f_{\pi_q(n)}(\rvx_q^n)$
		}
		$\pi_{q+1} \gets$ \Permute{$\cdots$}\label{alg-line:SGD:permute}\\
	}
\end{algorithm}
\DecMargin{1em}

\subsection{Order Error}
\label{subsec:order error}

We use the order error $\bar \phi_q$ as a measure of the quality of example orders. Recall that we say that an example order is good if it leads to a fast convergence rate, and vice versa, and thus its quality is dynamic, and depends on the factors like gradients.

\begin{definition}[Order Error, \citealt{lu2021general,lu2022grab}]
	\label{def:order error}
	The order error $\bar \phi_q$ in any epoch $q$ is defined as
	\begin{align*}
		\bar \phi_q \coloneqq \max_{n\in[N]} \left\{\phi_{q}^n \coloneqq \norm{ \sum_{i=0}^{n-1} \left(\nabla f_{\pi_q(i)}(\rvx_q) - \nabla f(\rvx_q)\right) }_p\right\}\,.
	\end{align*}
\end{definition}

We propose Assumption~\ref{asm:order error}, which explicitly demonstrates the dependence between permutations across different epochs. With it, we can incorporate the existing permutation-based SGD algorithms like AP, RR, OP and GraBs into one framework. In Section~\ref{sec:cases}, we will prove that Assumption~\ref{asm:order error} holds for AP, RR, OP and GraBs under some given assumptions.
\begin{assumption}
	\label{asm:order error}
	There exist nonnegative constants $\{A_i\}_{i=1}^q$, $\{B_i\}_{i=0}^q$ and $D$ such that for all $\rvx_q$ (the outputs of Algorithm~\ref{alg:SGD}),
	\begin{align*}
		\left(\bar \phi_q \right)^2 &\leq  \sum_{i=1}^{q} A_{i}\left(\bar \phi_{q-i} \right)^2  + \sum_{i=0}^{q} B_{i}\normsq{\nabla f(\rvx_{q-i})}  + D\,.
	\end{align*}
\end{assumption}

\subsection{Main Theorem}
\label{subsec:main theorem}

Definition~\ref{def:smoothness} will help us deal with the multiple smoothness constants in GraBs. 
\begin{definition}[$L_{p,p'}$-smoothness]
	\label{def:smoothness}
	We say $f$ is $L_{p,p'}$-smooth, if it is differentiable and for any $\rvx, \rvy\in \R^d$,
	\begin{align*}
		\norm{ \nabla f \left(\rvx \right) - \nabla f \left(\rvy \right) }_{p} \leq L_{p,p'} \norm{\rvx - \rvy}_{p'}\,.
	\end{align*}
	If $p=p'$, we write $L_{p,p'}$ as $L_p$; if $p=p'=2$, we write $L_{p,p'}$ as $L$ for convenience.
\end{definition}

We also assume that the global objective function $f$ is lower bounded by $f_\ast$. We let $F_0 = f(\rvx_0) - f_\ast$. The main theorem is presented in Theorem~\ref{thm:SGD}.

\begin{theorem}
	\label{thm:SGD}
	Let the global objective function $f$ be $L$-smooth and each local objective functions $f_n$ be $L_{2,p}$-smooth and $L_p$-smooth ($p\geq 2$). Suppose that Assumption~\ref{asm:order error} holds with $A_i = B_i = 0$ for $i> \nu$ ($\nu$ is a very small constant compared to $Q$). If $\gamma \leq \min \left\{ \frac{1}{ L N}, \frac{1}{32L_{2,p} N}, \frac{1}{32L_p N} \right\}$ and $\frac{\sum_{i=0}^{\nu} B_i }{N^2\left(1-\sum_{i=1}^{\nu} A_i\right)}<255$, then
	\begin{align*}
		\min_{q\in \{0,1,\ldots,Q-1\}} \normsq{ \nabla f(\rvx_q) } \leq c_1 \cdot \frac{ F_0 }{\gamma NQ} + c_2 \cdot \gamma^2L_{2,p}^2  \frac{1}{Q}\sum_{i=0}^{\nu-1}\left( \bar \phi_{i}\right)^2 + c_2\cdot\gamma^2L_{2,p}^2 D,
	\end{align*}
	where $c_1$ and $c_2$ are numerical constants such that $c_1 \geq \nicefrac{1}{\left( \frac{255}{512} - \frac{\sum_{i=0}^{\nu} B_i }{512N^2\left(1-\sum_{i=1}^{\nu} A_i\right)}\right)}$ and $c_2 \geq \left(\frac{2}{1-\sum_{i=1}^{\nu} A_i}\right)\cdot c_1$.
\end{theorem}

\section{Case studies}
\label{sec:cases}

In this section, we prove that Assumption~\ref{asm:order error} holds for AP, OP, RR and GraBs under some given assumptions, and then provide the corresponding upper bounds (see Table~\ref{tab:SGD}). Details are in Appendix~\ref{apx-sec:SGD:special cases}.

To prove Assumption~\ref{asm:order error}, in addition to the smoothness assumptions, Assumption~\ref{asm:local deviation} is required to bound the deviation of the local gradient from the global gradient.

\begin{assumption}\label{asm:local deviation}
	There exist a nonnegative constant $\varsigma$ such that for any $n\in \{0,1,\ldots,N-1\}$ and $\rvx\in \R^d$, 
	\begin{align*}%\textstyle
		\norm{\nabla f_n(\rvx) - \nabla f(\rvx) }^2 \leq \varsigma^2\,.
	\end{align*}
\end{assumption}

For OP, to derive the tighter bounds, we also need Assumption~\ref{asm:parameter deviation} to restrict the change of model parameters. Though this assumption seems stringent, it can be reasonable in some scenarios where the parameter change is not so large (for instance, fine-tuning). In addition, we can restrict the change by performing a proximal step at the end of each epoch \citep{mishchenko2022proximal,liu2024last}.

\begin{assumption}%[Parameter deviation]
	\label{asm:parameter deviation}
	There exists a nonnegative constant $\theta$ such that for all $\rvx_q$ (the outputs of Algorithm~\ref{alg:SGD}),
	\begin{align*}%\textstyle
		\norm{\rvx_q - \rvx_0 }^2 \leq \theta^2\,.
	\end{align*}
\end{assumption}

\subsection{Analyses of AP, RR and OPs}
\label{subsec:cases:others}

\textit{Analysis of Arbitrary Permutation.} The bound in Example~\ref{ex:arbitrary permutation} applies to all the methods discussed in the following. It matches that of \citet{lu2021general}.

\begin{example}[Arbitrary Permutations, AP]
	\label{ex:arbitrary permutation}
	For AP, all the permutations $\{\pi_q\}$ in Algorithm~\ref{alg:SGD} are generated arbitrarily. Under Assumption~\ref{asm:local deviation}, Assumption~\ref{asm:order error} holds as
	\begin{align*}
		\left(\bar\phi_{q}\right)^2 \leq N^2\varsigma^2\,.
	\end{align*}
	Applying Theorem~\ref{thm:SGD}, we get
	\begin{align*}
		\min_{q\in \{0,1\ldots, Q-1\}} \normsq{ \nabla f(\rvx_q) } = \gO \left(  \frac{F_0}{\gamma NQ} + \gamma^2 L^2 N^2 \varsigma^2 \right)\,.
	\end{align*}
	After we tune the step size, the upper bound becomes $\gO\left( \frac{L F_0}{Q} + \left(\frac{LF_0 N \varsigma }{ NQ}\right)^{\frac{2}{3}} \right)$.
\end{example}

\textit{Analysis of Random Reshuffling.} We consider the high-probability bound for RR \citep{lu2021general,yu2023high} rather than the in-expectation bounds \citep{mishchenko2020random,koloskova2024convergence}. This is mainly to maintain consistency with the high-probability bounds of GraBs. Theorem~\ref{thm:SGD} can be modified for in-expectation bounds readily by using the expectation version of Assumption~\ref{asm:order error} (taking expectations on both sides of the inequality in Assumption~\ref{asm:order error}). As shown in Example~\ref{ex:random reshuffling}, our high-probability bound of RR matches the prior best known bounds \citep{lu2021general,yu2023high}.

\begin{example}[Random Reshuffling, RR]
	\label{ex:random reshuffling}
	For RR, all the permutations $\{\pi_q\}$ in Algorithm~\ref{alg:SGD} are generated independently and randomly. Under Assumption~\ref{asm:local deviation}, Assumption~\ref{asm:order error} holds with probability at least $1-\delta$:
	\begin{align*}
		\left(\bar\phi_{q}\right)^2 \leq 4 N \varsigma^2 \log^2 \left( \nicefrac{8}{\delta}\right) \,.
	\end{align*}
	Applying Theorem~\ref{thm:SGD}, we get that, with probability at least $1-Q\delta$,
	\begin{align*}
		\min_{q\in \{0,1\ldots, Q-1\}} \normsq{ \nabla f(\rvx_q) } = \tilde\gO \left(  \frac{F_0}{\gamma NQ} + \gamma^2 L^2 N \varsigma^2 \right)\,.
	\end{align*}
	After we tune the step size, the upper bound becomes $\tilde\gO\left( \frac{LF_0 }{Q} +  \left(\frac{LF_0\sqrt{N}\varsigma}{NQ}\right)^{\frac{2}{3}} \right)$.
\end{example}

\textit{Analysis of One Permutation.} In OP, the key characteristic is that \textit{the initial permutation is reused} for the subsequent epochs. This avoids the repeated loading of the data examples, and thus leads to a faster implementation. To highlight this characteristic of OP, we try to establish the relation between $\bar\phi_q$ and $\bar\phi_0$. Specifically, for all $q\geq 1$ and $n\in[N]$,
\begin{align*}
	\textstyle
	\phi_{q}^n %&= \norm{ \sum_{i=0}^{n-1} \left( \nabla f_{\pi_{q}(i)}(\rvx_{q}) - \nabla f(\rvx_{q})  \right) }  \\
	&\leq 2LN\norm{\rvx_{q} - \rvx_{0}} + \norm{\sum_{i=0}^{n-1} \left( \nabla f_{\pi_{q}(i)}(\rvx_{0}) - \nabla f(\rvx_{0})  \right) }\\
	&\leq 2LN \norm{\rvx_{q} - \rvx_{0}} + \bar\phi_0\,,
\end{align*}
where the last inequality is due to $\pi_q = \pi_0$ for all $q\geq 1$. Then, since it holds for all $n\in [N]$, we get
\begin{align}
	\bar \phi_{q} \leq 2LN \norm{\rvx_{q} - \rvx_{0}} + \bar \phi_0 \leq 2LN \theta + \bar \phi_0\,,\nonumber%\label{eq:ex-op:connection}
\end{align}
where the last inequality is due to Assumption~\ref{asm:parameter deviation}. With this relation, we derive the upper bound of OP in Example~\ref{ex:one permutation}, along with some concrete instances.

\begin{example}[One Permutation, OP]
	\label{ex:one permutation}
	For OP, in Algorithm~\ref{alg:SGD}, the first-epoch permutation $\pi_0$ is generated arbitrarily/randomly/meticulously; the subsequent permutations are the same as the first-epoch permutation: $\pi_q = \pi_0$ for any $q\geq 1$.
	Let all $f_n$ be $L$-smooth and Assumptions~\ref{asm:local deviation},~\ref{asm:parameter deviation} hold. Then, Assumption~\ref{asm:order error} holds as
	\begin{align*}
		\left(\bar\phi_{q}\right)^2 \leq 2 \left( \bar\phi_0 \right)^2 + 8 L^2 N^2 \theta^2 \,.
	\end{align*}
	Applying Theorem~\ref{thm:SGD}, we get
	\begin{align*}
		\min_{q\in \{0,1,\ldots,Q-1\}}\normsq{ \nabla f(\rvx_q) }
		&= \gO \left(  \frac{F_0}{\gamma NQ} + \gamma^2 L^2 \left( \bar\phi_0 \right)^2 + \gamma^2 L^4 N^2\theta^2 \right)
	\end{align*}
	After we tune the step size, the upper bound becomes $\gO\left( \frac{LF_0 }{Q} +  \left(\frac{LF_0\bar\phi_0 + L^2 F_0 N\theta}{NQ}\right)^{\frac{2}{3}} \right)$. Furthermore, if $\theta \lesssim \frac{\bar \phi_0}{LN}$, it becomes $\gO\left( \frac{LF_0 }{Q} +  \left(\frac{LF_0\bar\phi_0}{NQ}\right)^{\frac{2}{3}} \right)$.
	\begin{itemize}
		\item Incremental Gradient (IG). If the initial permutation is generated arbitrarily (it implies that $\bar\phi_0 = \gO\left( N \varsigma \right)$), then the bound will be $\gO\left( \frac{LF_0 }{Q} +  \left(\frac{LF_0N\varsigma}{NQ}\right)^{\frac{2}{3}} \right)$.
		\item Shuffle Once (SO). If the initial permutation is generated randomly (it implies that $\bar\phi_0 = \tilde\gO( \sqrt{N} \varsigma )$), then the bound will be $\tilde\gO\left( \frac{LF_0 }{Q} +  \left(\frac{LF_0\sqrt{N}\varsigma}{NQ}\right)^{\frac{2}{3}} \right)$. It holds with probability at least $1-\delta$.
		\item Nice Permutation (NP). If the initial permutation is generated meticulously (it implies that $\bar\phi_0 = \tilde\gO\left( \varsigma \right)$), then the bound will be $\tilde\gO\left( \frac{LF_0 }{Q} +  \left(\frac{LF_0\varsigma}{NQ}\right)^{\frac{2}{3}} \right)$.
	\end{itemize}
\end{example}

Example~\ref{ex:one permutation} states that OP methods show great potentials in the scenarios where the parameter change is not so large. Specifically, if the initial permutation is produced meticulously, \textit{OP (NP) can even converge faster than RR}. This finding is aligned with that in \citet{yun2021open}, whose result depends on the refined matrix AM-GM inequality conjecture. Intuitively, the advantage of NP in Example~\ref{ex:one permutation} comes from that the ``nice'' order in the first epoch \textit{is still ``nice''} in the subsequent epochs, which in fact relies on that $\nabla f_{\pi_0 (i)}(\rvx_0)$ is a nice estimator of $\nabla f_{\pi_q (i)}(\rvx_q)$ for $q\geq 1$. However, when the parameter changes drastically, the estimate becomes inaccurate, and the initial ``nice'' order becomes ``worse'' subsequently. This is also the reason why we use Assumption~\ref{asm:parameter deviation} to restrict the drastic change of parameters.

\subsection{Analyses of GraBs}
\label{subsec:cases:GraBs}

Recall that our goal is to find a permutation to minimize the order error (Notably, in GraBs, $\bar \phi_q$ is defined by $\norm{\cdot}_{\infty}$)
\begin{align*}\textstyle
	\bar \phi_q \coloneqq \max_{n\in[N]} \norm{ \sum_{i=0}^{n-1} \left(\nabla f_{\pi_q(i)}(\rvx_q) - \nabla f(\rvx_q)\right) }_{\infty}\,,
\end{align*}
which is aligned with the goal of herding \citep{welling2009herding}. With this insight, \citet{lu2022grab} proposed GraB (to produce good permutations online) based on the theory of herding and balancing \citep{harvey2014near,alweiss2021discrepancy}: Consider $N$ vectors $\{\rvz_n \}_{n=0}^{N-1}$ such that $\sum_{n=0}^{N-1} \rvz_n = 0$ and $\norm{\rvz_n} \leq 1$. First, for any permutation $\pi$, assign the signs $\{\epsilon_{n} \}_{n=0}^{N-1}$ ($\epsilon_n \in \{-1,+1\}$) to the permuted vectors $\{\rvz_{\pi(n)}\}_{n=0}^{N-1}$ using the \textit{balancing} algorithms (such as Algorithm~\ref{alg:balancing} in Appendix~\ref{apx-sec:algs}). Second, with the assigned signs and the old permutation $\pi$, produce a new permutation $\pi'$ by the \textit{reordering} algorithm (that is, Algorithm~\ref{alg:reordering} in Appendix~\ref{apx-sec:algs}). Then,% the herding errors satisfy the relation
\begin{flalign}
	\max_{n \in [N]} \norm{\sum_{i=0}^{n-1}\rvz_{\pi'(i)}}_{\infty} \leq \frac{1}{2} \max_{n \in [N]} \norm{\sum_{i=0}^{n-1}\rvz_{\pi(i)}}_{\infty} +  \frac{1}{2} \max_{n \in [N]} \norm{\sum_{i=0}^{n-1}\epsilon_i\cdot  \rvz_{\pi(i)}}_{\infty} \label{eq:ex-grab:key relation}
\end{flalign}
where we call the three terms, the herding error under $\pi'$, the herding error under $\pi$, and the signed herding error under $\pi$, respectively (see Lemma~\ref{lem:basic-balancing-reordering}). Ineq.~\eqref{eq:ex-grab:key relation} ensures that the herding error will be reduced (from $\pi$ to $\pi'$) as long as the signed herding error is small. That is, the herding error can be progressively reduced by balancing and reordering the vectors. By iteratively applying this process (balancing and then reordering), the herding error will approach the signed herding error, which is proved to be $\tilde\gO\left( 1\right)$, if the signs are assigned by Algorithm~\ref{alg:balancing} \citep{alweiss2021discrepancy}.

\textit{Analysis of GraB-proto.} To present the key idea of GraBs, as well as our theory, we start from GraB-proto, the simplified version of the original GraB \citep{lu2022grab}. The key characteristic of GraB-proto (and other variants) is that the example order depends on the example order of previous epochs. Thus, the goal is to find the relation between $\bar \phi_{q}$ and $\bar \phi_{q-1}$. Specifically, for all $q\geq1$ and $n \in [N]$,
\begin{align*}
	\phi_{q}^n %&= \norm{ \sum_{i=0}^{n-1} \left( \nabla f_{\pi_{q}(i)}(\rvx_{q}) - \nabla f(\rvx_{q})  \right) }_{\infty}  \\
	&\leq 2L_{\infty} N \norm{\rvx_{q} - \rvx_{q-1}}_{\infty} + \max_{n\in [N] }\norm{\sum_{i=0}^{n-1} \left( \nabla f_{\pi_{q}(i)}(\rvx_{q-1}) - \nabla f(\rvx_{q-1})  \right) }_{\infty}.
\end{align*}
First, note that the first term is the well-studied ``parameter deviation'' \citep{mishchenko2020random}, whose upper bound is provided in Lemma~\ref{lem:parameter drift}. Second, since GraB-proto uses $\pi_{q-1}$, $\{\nabla f_{\pi_{q-1}(n)}(\rvx_{q-1})- \nabla f(\rvx_{q-1})\}_{n=0}^{N-1}$ to generate $\pi_q$ in epoch $(q-1)$, we can apply Ineq.~\eqref{eq:ex-grab:key relation} to the second term:
\begin{align*}
	&\max_{n\in [N] }\norm{\sum_{i=0}^{n-1} \left( \nabla f_{\pi_{q}(i)}(\rvx_{q-1}) - \nabla f(\rvx_{q-1})  \right) }_{\infty} \\
	&\leq \frac{1}{2}\max_{n\in [N] }\norm{\sum_{i=0}^{n-1} \left( \nabla f_{\pi_{q-1}(i)}(\rvx_{q-1}) - \nabla f(\rvx_{q-1})  \right) }_{\infty} + \frac{1}{2}C \varsigma\\
	&= \frac{1}{2}\bar\phi_{q-1} + \frac{1}{2}C \varsigma\,,
\end{align*}
where $C = \gO \left(\log\left(\frac{d N}{\delta}\right) \right) = \tilde\gO\left( 1\right)$ is from \citet[Theorem~1.1]{alweiss2021discrepancy}. We also use Assumption~\ref{asm:local deviation} to scale the vector length to be no greater than $1$. Now, combining them gives the relation in Example~\ref{ex:GraB-proto}.

\begin{example}[GraB-proto]
	\label{ex:GraB-proto}
	Let each $f_n$ be $L_{\infty}$-smooth and Assumption~\ref{asm:local deviation} hold. Then, if $\gamma \leq \frac{1}{32L_{\infty}N}$, Assumption~\ref{asm:order error} holds with probability at least $1-\delta$:
	\begin{align*}
		\left(\bar \phi_{q}\right)^2 \leq \frac{3}{4}\left( \bar \phi_{q-1} \right)^2 + \frac{1}{50} N^2 \normsq{\nabla f(\rvx_{q-1})}_{} + C^2 \varsigma^2,
	\end{align*}
	where $C = \gO \left(\log\left(\frac{d N}{\delta}\right) \right) = \tilde\gO\left( 1\right)$. Applying Theorem~\ref{thm:SGD}, we get that, with probability at least $1-Q\delta$,
	\begin{align*}
		&\min_{q\in \{0,1\ldots, Q-1\}} \normsq{ \nabla f(\rvx_q) } = \gO\left(  \frac{F_0}{\gamma NQ} + \gamma^2 \frac{1}{Q} L_{2,\infty}^2 N^2 \varsigma^2 + \gamma^2 L_{2,\infty}^2C^2\varsigma^2 \right).
	\end{align*}
	After we tune the step size, the upper bound becomes $\gO\left( \frac{\tilde LF_0 + \left(L_{2,\infty}F_0\varsigma\right)^{\frac{2}{3}} }{Q} + \left(\frac{L_{2,\infty}F_0 C \varsigma}{NQ}\right)^{\frac{2}{3}} \right)$, where $\tilde L = L+L_{2,\infty}+L_{\infty}$.
\end{example}

\textit{Analysis of GraB and PairGraB.} We also give the upper bounds of GraB and PairGraB in Examples~\ref{ex:GraB} and~\ref{ex:PairGraB} (The proofs are deferred to Appendix~\ref{apx-sec:SGD:special cases} due to their complexity). See Appendix~\ref{apx-sec:algs} for details of these two practical algorithms. Though PairGraB has appeared in the public code of \citet{lu2022grab}, its upper bound is still missing before this paper. See Examples~\ref{ex:GraB} and \ref{ex:PairGraB}. First, the upper bounds of GraB and PairGraB are almost identical to that of GraB-proto, with a more stringent constraint of the step size and some differences of numerical constants. Second, the $\bar \phi_q$ of GraB is affected by the factors from the previous two epochs (such as $\bar \phi_{q-1}$ and $\bar \phi_{q-2}$). This is because GraB uses the average of the stale gradients for centering, while PairGraB is free of centering (see Appendix~\ref{apx-sec:algs}).

\begin{example}[GraB]
	\label{ex:GraB}
	Let each $f_n$ be $L_{2,\infty}$-smooth and $L_{\infty}$-smooth, and Assumption~\ref{asm:local deviation} hold. Then, if $\gamma \leq \min\{\frac{1}{128 L_{2,\infty} C},\frac{1}{128 L_{\infty} N} \}$, Assumption~\ref{asm:order error} holds with probability at least $1-\delta$:
	\begin{align*}
		&\left(\bar \phi_{q}\right)^2
		\leq \frac{3}{5}\left( \bar \phi_{q-1} \right)^2 + \frac{1}{50}\left( \bar \phi_{q-2} \right)^2 + \frac{1}{50}N^2 \normsq{\nabla f(\rvx_{q-1})}_{} +\frac{1}{50}N^2 \normsq{\nabla f(\rvx_{q-2})}_{} + 2C^2 \varsigma^2,
	\end{align*}
	where $C = \gO \left(\log\left(\frac{d N}{\delta}\right) \right) = \tilde\gO\left( 1\right)$. Applying Theorem~\ref{thm:SGD} (with a tighter constraint $\gamma \leq \min \{ \frac{1}{L N}, \frac{1}{128L_{2,\infty} (N+C)}, \frac{1}{128L_\infty N} \}$), we get that, with probability at least $1-Q\delta$,
	\begin{align*}
		&\min_{q\in \{0,1\ldots, Q-1\}} \normsq{ \nabla f(\rvx_q) } = \gO\left(  \frac{F_0}{\gamma NQ} + \gamma^2 \frac{1}{Q} L_{2,\infty}^2 N^2 \varsigma^2 + \gamma^2 L_{2,\infty}^2C^2\varsigma^2 \right).
	\end{align*}
	After we tune the step size, the upper bound becomes $\gO\left( \frac{\tilde LF_0 + \left(L_{2,\infty}F_0\varsigma\right)^{\frac{2}{3}} }{Q} + \left(\frac{L_{2,\infty}F_0 C \varsigma}{NQ}\right)^{\frac{2}{3}} \right)$, where $\tilde L = L+L_{2,\infty}\left(1+\frac{C}{N}\right)+L_{\infty}$.
\end{example}

\begin{example}[PairGraB]
	\label{ex:PairGraB}
	Let each $f_n$ be $L_{2,\infty}$-smooth and $L_{\infty}$-smooth, and Assumption~\ref{asm:local deviation} hold. Assume that $N\mod 2=0$. Then, if $\gamma \leq \min\{\frac{1}{64 L_{2,\infty} C},\frac{1}{64 L_{\infty} N} \}$, Assumption~\ref{asm:order error} holds with probability at least $1-\delta$:
	\begin{align*}
		\left(\bar \phi_{q}\right)^2 \leq \frac{4}{5}\left( \bar \phi_{q-1} \right)^2 + \frac{3}{50} N^2 \normsq{\nabla f(\rvx_{q-1})}_{} + 4C^2 \varsigma^2,
	\end{align*}
	where $C = \gO \left(\log\left(\frac{d N}{\delta}\right) \right) = \tilde\gO\left( 1\right)$. Applying Theorem~\ref{thm:SGD} (with a tighter constraint $\gamma \leq \min \{ \frac{1}{L N}, \frac{1}{64L_{2,\infty} (N+C)}, \frac{1}{64L_\infty N} \}$), we get that, with probability at least $1-Q\delta$,
	\begin{align*}
		&\min_{q\in \{0,1\ldots, Q-1\}} \normsq{ \nabla f(\rvx_q) }= \gO\left(  \frac{F_0}{\gamma NQ} + \gamma^2 \frac{1}{Q} L_{2,\infty}^2 N^2 \varsigma^2 + \gamma^2 L_{2,\infty}^2C^2\varsigma^2 \right).
	\end{align*}
	After we tune the step size, the upper bound becomes $\gO\left( \frac{\tilde LF_0 + \left(L_{2,\infty}F_0\varsigma\right)^{\frac{2}{3}} }{Q} + \left(\frac{L_{2,\infty}F_0 C \varsigma}{NQ}\right)^{\frac{2}{3}} \right)$, where $\tilde L = L+L_{2,\infty}\left(1+\frac{C}{N}\right)+L_{\infty}$.
\end{example}

\section{Federated Learning}
\label{sec:FL}

\textit{Setup.} In this section, we adapt our theory on \textit{example ordering in SGD for client ordering in FL.} For FL, we consider the same problem as that in SGD (that is, Eq.~\ref{eq:problem}). Notably, in the context of FL, the local objective functions represent the clients in FL. We focus on FL with regularized client participation (regularized-participation FL), where each client participate once before any client is reused \citep{wang2022unified}. More concretely, see Algorithm~\ref{alg:FL}. During each epoch, it selects $S$ clients at a time from the permuted clients (under the permutation $\pi$) to complete a round of federated training, until all the clients have participated. Pay attention that one ``epoch'' may include multiple ``rounds''. At the end of each epoch, it produces the next-epoch permutation by some permuting algorithm. Here we also consider the global update \citep{karimireddy2020scaffold, wang2022unified} (see Line~\ref{alg-line:FL:amplify}). Considering that we mainly study the client ordering of FL in this paper, we use Gradient Descent (GD) as the local solver of FL (see Lines~\ref{alg-line:FL:local-update}--\ref{alg-line:FL:local-update^2}) for simplicity. We assume $N \mod S = 0$.

\IncMargin{1em}
\begin{algorithm}%[H]
	\SetNoFillComment
	\DontPrintSemicolon
	\SetAlgoNoEnd
	\caption{Regularized-participation FL}
	\label{alg:FL}
	\SetKwFunction{Permute}{Permute}
	
	\KwIn{$\pi_0$, $\rvx_0$; \textbf{Output}: $\{\rvx_q\}$}
	\For{$q = 0, 1,\ldots, Q-1$}{
		$\rvw \gets \rvx_q$\\
		\For{$n=0,1,\ldots, N-1$}{
			Initilize $\rvx_{q,0}^{n} \gets \rvw $\\
			\For{$k=0,1,\ldots, K-1$}{\label{alg-line:FL:local-update}
				$\rvx_{q,k+1}^{n} \gets \rvx_{q,k}^{n} - \gamma \nabla f_{\pi_q(n)}(\rvx_{q,k}^n)$\label{alg-line:FL:local-update^2}\\
			}
			$\rvp_q^n \gets \rvx_{q,0}^{n} - \rvx_{q,K}^{n}$\\
			\If{$(n+1) \mod S = 0$}{
				$\rvw \gets \rvw - \frac{1}{S}\sum_{s=0}^{S-1}\rvp_q^{n-s} $
			}	
		}
		$\rvx_{q+1} \gets \rvx_{q} - \eta(\rvx_{q}-\rvw)$\label{alg-line:FL:amplify}\\
		$\pi_{q+1} \gets$ \Permute{$\cdots$}
	}
\end{algorithm}
\DecMargin{1em}

\textit{Main theorem.} Compared to SGD, the main challenges or differences lie in the following two aspects:
\begin{enumerate}%[label={(\arabic*)}]
	\item Partially parallel updates. In a round of federated training, the selected $S$ clients are in parallel.
	\item Local updates. It performs multiple local updates on each local objective function.
\end{enumerate}
First, we obtain Definition~\ref{def:FL:order error} by a similar analysis of Definition~\ref{def:order error} in Section~\ref{subsec:order error}, which exactly addresses the first challenge. Then, with the help of Definition~\ref{def:FL:order error}, we propose Assumption~\ref{asm:FL:order error} and prove Theorem~\ref{thm:FL}. Notably, the third term (containing $\varsigma$) on the right hand side in Ineq.~\eqref{eq:FL:main theorem} is not subsumed into Assumption~\ref{asm:FL:order error}. This is because this term is from the local updates, which is affected by the example order rather than the client order in FL. In our setting (GD is used as the local solver), the second challenge is relatively manageable. However, it may significantly complicate the analysis if permutation-based SGD is used as the local solver in FL, which we leave for future work.

\begin{definition}
	\label{def:FL:order error}
	The order error $\bar \varphi_q$ in any epoch $q$ in FL is defined as ($v(n) \coloneqq \floor{\frac{n}{S}}\cdot S$)
	\begin{align*}
		\bar\varphi_{q} \coloneq \max_{n \in[N]}\left\{\varphi_{q}^{v(n)} \coloneqq \norm{ \sum_{i=0}^{\textcolor{blue}{v(n)}-1} \left(\nabla f_{\pi_q(i)}(\rvx_{q}) - \nabla f(\rvx_{q})\right) }_p\right\}.
	\end{align*}
\end{definition}

\begin{assumption}
	\label{asm:FL:order error}
	There exist nonnegative constants $\{A_i\}_{i=1}^q$, $\{B_i\}_{i=0}^q$ and $D$ such that for all $\rvx_q$ (the outputs of Algorithm~\ref{alg:FL}),
	\begin{align*}
		\left(\bar \varphi_q \right)^2 &\leq  \sum_{i=1}^{q} A_{i}\left(\bar \varphi_{q-i} \right)^2  + \sum_{i=0}^{q} B_{i}\normsq{\nabla f(\rvx_{q-i})}  + D\,.
	\end{align*}
\end{assumption}

\begin{theorem}
	\label{thm:FL}
	Let the global objective function $f$ be $L$-smooth, each local objective function $f_n$ be $L_{2,p}$-smooth and $L_p$-smooth ($p\geq 2$), and Assumption~\ref{asm:local deviation} hold. Suppose $N \mod S = 0$. Suppose that Assumption~\ref{asm:FL:order error} holds with $A_i = B_i = 0$ for $i> \nu$ ($\nu$ is a very small constant compared to $Q$). If $\gamma \leq \left\{ \frac{1}{32L_{2,p} KN\frac{1}{S}}, \frac{1}{\eta L KN\frac{1}{S}}, \frac{1}{32L_p KN\frac{1}{S}} \right\}$ and $\frac{\sum_{i=0}^{\nu} B_i }{N^2\left(1-\sum_{i=1}^{\nu} A_i\right)} <255$, then
	\begin{align}
		&\min_{q\in \{0,1,\ldots,Q-1\}} \normsq{\nabla f(\rvx_q)}\nonumber\\
		&\leq c_1 \cdot \frac{ F_0 }{\gamma \eta KN\frac{1}{S}Q} + c_2 \cdot \gamma^2 L_{2,p}^2 K^2\frac{1}{S^2}  \frac{1}{Q}\sum_{i=0}^{\nu-1}\left( \bar \varphi_{i}\right)^2 + 2c_1\cdot \gamma^2 L_{2,p}^2 K^2 \varsigma^2 +  c_2\cdot\gamma^2 L_{2,p}^2 K^2\frac{1}{S^2} D\,.\label{eq:FL:main theorem}
	\end{align}
	where $c_1$ and $c_2$ are numerical constants such that $c_1 \geq \nicefrac{1}{\left(\frac{255}{512} - \frac{\sum_{i=0}^{\nu} B_i }{512N^2 \left(1-\sum_{i=1}^{\nu} A_i\right)}\right)}$ and $c_2 \geq \left(\frac{2}{1-\sum_{i=1}^{\nu} A_i}\right)\cdot c_1$.
\end{theorem}

\begin{table*}[t]
	\caption{Upper bounds of FL with regularized client participation (the numerical constants and polylogarithmic factors are hided). The global step size is set to $\eta=1$ for comparison.}
	\label{tab:FL}
	\setlength{\tabcolsep}{1.3em}
	\centering{
			\resizebox{\linewidth}{!}{
				\begin{threeparttable}
					\begin{tabular}{lll}
						\toprule
						{\bf Method}  & \textbf{Corr.} & \textbf{Upper Bound}  \\\midrule

						\makecell[l]{FL-AP\\ \citep{wang2022unified}}  &AP &$\frac{L F_0}{Q} + \left(\frac{LF_0 S\varsigma }{ NQ}\right)^{\frac{2}{3}} + \left(\frac{LF_0 N \varsigma }{ NQ}\right)^{\frac{2}{3}}$\tnote{\textcolor{blue}{(1)}} \vspace{0.5ex} \\
						
						\midrule
						FL-AP (Ex.~\ref{ex:FL-AP}) &AP & $\frac{L F_0}{Q} + \left(\frac{LF_0 S\varsigma }{ NQ}\right)^{\frac{2}{3}} + \left(\frac{LF_0 N \varsigma }{ NQ}\right)^{\frac{2}{3}}$  \vspace{0.5ex} \\
						
						FL-RR (Ex.~\ref{ex:FL-RR}) &RR & $\frac{L F_0}{Q} + \left(\frac{LF_0 S\varsigma }{ NQ}\right)^{\frac{2}{3}} + \left(\frac{LF_0 \sqrt{N} \varsigma }{ NQ}\right)^{\frac{2}{3}}$  \vspace{0.5ex} \\
						
						FL-OP (Ex.~\ref{ex:FL-OP}) &OP & $\frac{L F_0}{Q} + \left(\frac{LF_0 S\varsigma }{ NQ}\right)^{\frac{2}{3}} + \left(\frac{LF_0 \bar\varphi_0 }{ NQ}\right)^{\frac{2}{3}}$\tnote{\textcolor{blue}{(2)}} \vspace{0.5ex} \\
						
						FL-GraB (Ex.~\ref{ex:FL-GraB}) &PairGraB & $\frac{\tilde L  F_0+\left (L_{2,\infty}F_0 \varsigma \right)^{\frac{2}{3}}}{Q} + \left(\frac{L_{2,\infty}F_0 S\varsigma }{ NQ}\right)^{\frac{2}{3}} + \left(\frac{L_{2,\infty}F_0 \varsigma }{ NQ}\right)^{\frac{2}{3}}$\tnote{\textcolor{blue}{(3)}}\\[0.5ex]
						
						\bottomrule
					\end{tabular}
					\begin{tablenotes}
						\scriptsize
						\item[1] In \citet[Theorem~3.1]{wang2022unified}, let $d$ be $\varsigma$; let $\tilde \nu$ and $\tilde \beta$ be $\varsigma$ (see their Proposition 4.1); let $\sigma$ be $0$; let $\gF$ be $F_0$; let $I$ be $K$; let $P = \frac{N}{S}$. Then letting $\eta =1$, tuning the step size with Lemma~\ref{lem:step-size}, we can recover the bound in the table.
						\item[2] It requires $\theta \lesssim \frac{\bar\varphi_0}{LN}$ (see Assumption~\ref{asm:parameter deviation} for $\theta$). The $\bar \varphi_0$ can be $\gO(N\varsigma)$, $\tilde\gO(\sqrt{N}\varsigma)$ and $\tilde \gO(\varsigma)$, depending on the initial permutation.
						\item[3] Here $\tilde L = L +L_{2,\infty} + L_{\infty}$. See Definition~\ref{def:smoothness} for $L$, $L_{2,\infty}$ and $L_{\infty}$.
					\end{tablenotes}
				\end{threeparttable}
			}
		}
\end{table*}

\textit{Case studies.} Our unified framework covers regularized-participation FL with Arbitrary Permutations (FL-AP), with Random Reshuffling (FL-RR), with One Permutation (FL-OP) and with GraBs (FL-GraBs). They correspond to AP, RR, OP, and GraBs in SGD, respectively. In particular, we propose regularized-participation FL with GraB (FL-GraB), whose corresponding algorithm in SGD is PairGraB (the most advanced GraB algorithm). See Appendix~\ref{apx-sec:algs}.

The upper bounds are summarized in Table~\ref{tab:FL}, and the details are in Examples~\ref{ex:FL-AP}--\ref{ex:FL-GraB} (Appendix~\ref{apx-sec:FL:special cases}). As shown in Table~\ref{tab:FL}, the main difference lies in the last term: the upper bound of FL-GraB $\tilde \gO ((\frac{1}{NQ})^{\frac{2}{3}})$ dominants those of the other algorithms in terms of the number of epochs $Q$ and the number of clients $N$; when the parameter change is small and the initial permutation is nice, FL-OP can achieve the best rate of $\tilde \gO ((\frac{1}{NQ})^{\frac{2}{3}})$. These conclusions are aligned with those in SGD.

\section{Experiments}
\label{sec:exp}

\begin{figure}[h]
	\centering
	\includegraphics[width=0.49\linewidth]{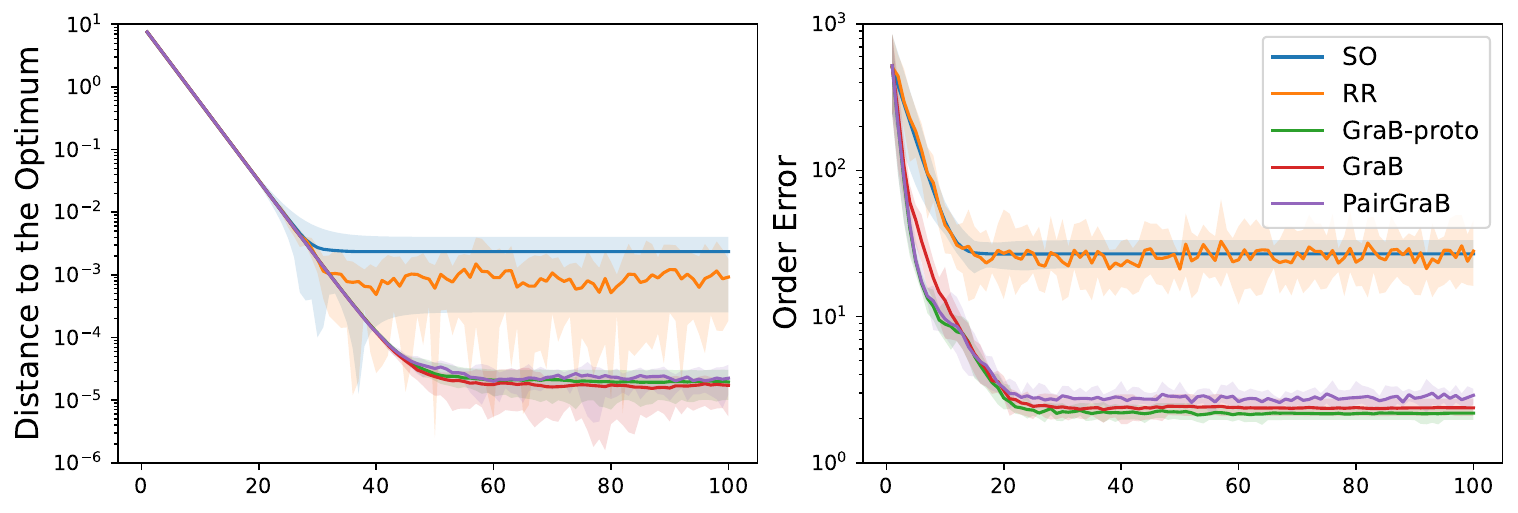}
	\includegraphics[width=0.49\linewidth]{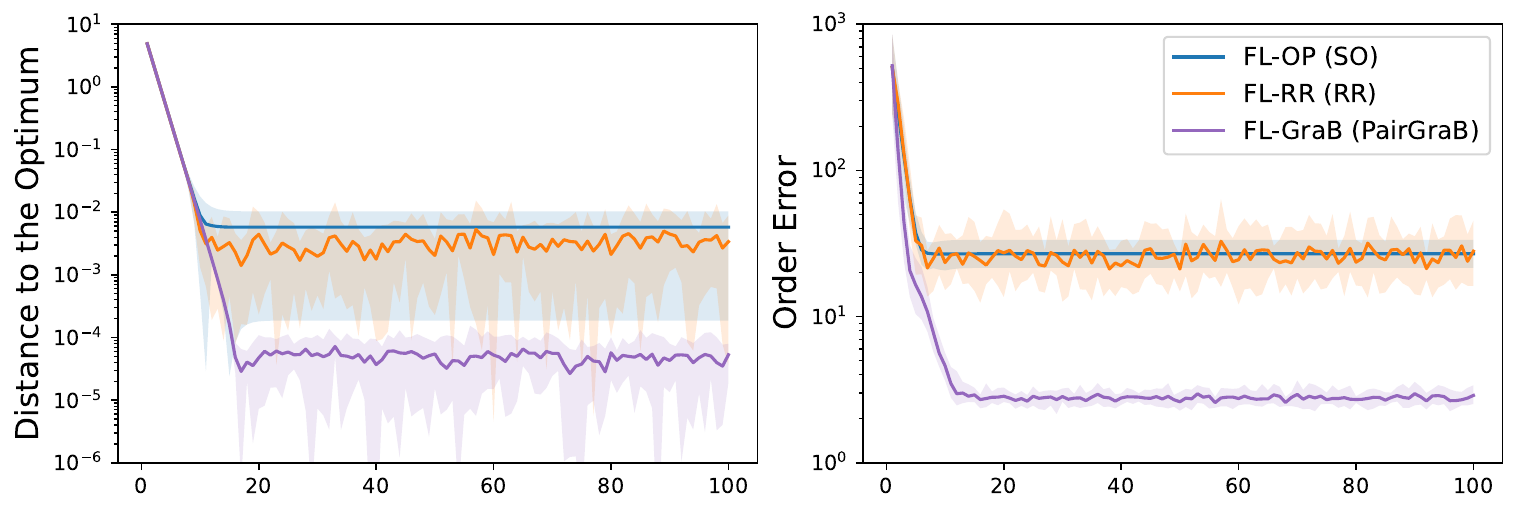}
	\caption{Simulations on quadratic functions. Shaded areas show the min-max values across 10 different random seeds. The left two figures are for SGD; the right two figures are for FL (The corresponding algorithms in SGD are in the parentheses.). For both SGD and FL, $\gamma$ is set to be the same for the algorithms; $N=1000$. For FL, $K=5$ and $S=2$.}
	\label{fig:quadratic}
\end{figure}

In this section, we run experiments on quadratic functions to validate the theory. Refer to \citet{lu2022grab,cooper2023coordinating} for the experimental results of SGD on real data sets; refer to Appendix~\ref{apx-sec:exps} for the experimental results of FL on real data sets.

We use the One-dimensional quadratic functions with the form of $f_n(\rvx) = a_n \rvx^2 + b_n \rvx $ for all $n\in \{ 0,1,\ldots,N-1 \}$ as the local objective functions. We model $a_n \sim \gN (0.5, 1)$ and $b_n \sim \gN (0,1)$ ($\gN$ is the normal distribution). Here $a_n$ and $b_n$ control the heterogeneity of the local objective functions. The experimental results are shown in Figure~\ref{fig:quadratic}. First, we see that the distance between the parameter $\rvx$ and the optimum $\rvx^\ast$ (that is, $\norm{\rvx-\rvx^\ast}$) and the order error $\bar\phi$ have the same trend, which validates that $\bar\phi$ can measure the convergence rate. Second, we see that the GraB algorithms are better than RR and SO in both SGD and FL.

\section{Conclusion}

We study example ordering in permutation-based SGD and client ordering in regularized-participation FL. For SGD, we propose a more general assumption (Assumption~\ref{asm:order error}) to bound the order error. Using it, we develop a unified framework for permutation-based SGD with arbitrary permutations of examples, including AP, RR, OP and GraBs. Furthermore, we develop a unified framework for regularized-participation FL with arbitrary permutations of clients, including FL-AP, FL-RR, FL-OP and FL-GraBs.

Possible future directions are as follows. First, explore new algorithm for SGD (no new algorithms are proposed for SGD in this work). Second, extend the framework to more practical scenarios for FL (our theory is for FL with regularized participation). Third, study example ordering in local updates for FL (we use GD as the local solver).

\bibliographystyle{plainnat}
\bibliography{refs}

\clearpage
\begin{center}
	\LARGE {Appendix}
\end{center}
\appendix
\vskip 3ex\hrule\vskip 2ex
{
		\hypersetup{linktoc=page}
		\parskip=0.5ex
		\startcontents[sections]
		\printcontents[sections]{l}{1}{\setcounter{tocdepth}{3}}
	}
\vskip 3ex\hrule\vskip 5ex
\clearpage

\section{Related Work}
\label{apx-sec:related-works}

Convergence analyses of permutation-based SGD. Up to now, there have been a wealth of works to analyze the convergence of permutation-based SGD \citep{ahn2020sgd, mishchenko2020random,mishchenko2022proximal,nguyen2021unified,liu2024last, safran2020good,safran2021random,rajput2020closing,rajput2022permutationbased,yun2021open,yun2022minibatch,cha2023tighter}. Among them, the most relevant works are the unified analyses of permutation-based SGD \citep{lu2021general,mohtashami2022characterizing,koloskova2024convergence}. They all rely on Assumption~\ref{asm:order error} (they may consider an interval of arbitrary length, not necessarily an epoch); this assumption has been widely adopted in the subsequent works \citep{even2023stochastic,islamov2024asgrad,li2024provably} for other settings beyond this paper. Let us use the upper bounds in \citet{mishchenko2020random} as the baselines. The framework of \citet{lu2021general} includes AP, RR, SO (and so on). Their upper bounds of AP and RR match the baselines; the error term of their upper bound of SO is $\gO \left(\frac{LF_0\sqrt{Nd}\varsigma}{NQ} \right)$, which is better than the baselines when the dimension $d$ is smaller the number of the examples $N$. The framework of \citet{koloskova2024convergence} includes RR, IG, SO (and so on). The optimization term of the upper bounds of IG and SO is improved from $\gO\left(\frac{LF_0}{Q}\right)$ to $\gO\left(\frac{LF_0}{NQ}\right)$; one drawback is that they cannot recover the prior best known bound of RR. Most importantly, the existing works can not include GraBs.

Convergence analyses of FL with regularized client participation. The convergence analyses of regularized-participation FL have been studied in \citet{wang2022unified,cho2023convergence,malinovsky2023federated}, where \citet{wang2022unified} considered regularized-participation FL with AP (FL-AP), \citet{cho2023convergence} considered regularized-participation FL with OP (FL-OP, or FL with cyclic client participation) and \citet{malinovsky2023federated} considered regularized-participation FL with RR (FL-RR). Thus, this work aims to develop a unified framework that includes these cases. Importantly, this work focuses on client ordering in FL with regularized participation, which is different from the studies of FL with arbitrary participation \citep{gu2021fast,wang2022unified,wang2024lightweight} and client sampling \citep{cho2022towards,horvath2022fedshuffle}.

\section{Notations}
\label{apx-sec:notations}

\begin{table}[ht]
	\caption{Summary of key notations.}
	\label{tab:notations}
	\centering{%\small
				\begin{threeparttable}
					\begin{tabular}{lp{42em}}
						\toprule
						\textbf{Notation}  & \textbf{Description}  \vspace{0ex} \\\midrule
						$Q$ &Number of epochs.\\
						$N$ &Number of local objective functions.\\
						$K$ &Number of local steps in FL.\\
						$S$ &Number of participating clients in each round in FL. \\
						$L_{p,p'}$ &Smoothness constants (see Definition~\ref{def:smoothness}).\\
						$A,B,D$ &Constants in Assumptions~\ref{asm:order error} and~\ref{asm:FL:order error}.\\
						$d$ &Dimension of the model parameter vector\\
						$\varsigma$ &Constant in Assumption~\ref{asm:local deviation}\\
						$\theta$ &Constant in Assumption~\ref{asm:parameter deviation}\\
						$\gamma$ &Step size\\
						$\eta$ &Global step size (in FL)\\
						$\bar\phi$ &Order Error in SGD\\
						$\bar\varphi$ &Order Error in FL\\
						$\pi$ &A permutation of $\{0,1, \ldots, N-1\}$. It serves as the order of examples or clients.\\
						$\pi(n)$ &The $(n+1)$-th element of permutation $\pi$.\\
						$f$ &Global objective function.\\
						$f_n$ &Local objective function. It represents examples in SGD; it represents clients in FL.\\
						$F_0$ &$F_0 = f(\rvx_0) - f_\ast$\\
						$\rvx$ &Model parameter vector.\\
						$\rvx_q^n$ &Parameter vector after $n$ steps in epoch $q$ (in SGD).\\
						$\rvx_{q,k}^n$ &Parameter vector after $k$ local updates in client $n$ in epoch $q$ (in FL).\\
						$\rvp_q^n$ &Pseudo-gradient of client $n$ in epoch $q$ in FL.\\

						\bottomrule
					\end{tabular}
				\end{threeparttable}
		}
	\end{table}

In this paper, ``SGD'' refers to ``permutation-based SGD'' and ``FL'' refers to ``regularized-participation FL (FL with regularized client participation)''.

Key notations are in Table~\ref{tab:notations}.

\textit{Norm.} We use $\norm{\cdot}_p$ to denote the Lebesgue $p$-norm; unless otherwise stated, we use $\norm{\cdot}_{}$ to denote the Lebesgue $2$-norm. 

\textit{Set.} We let $[n] \coloneqq \{1,2,\ldots, n\}$ for $n\in \mathbb{N}^+$ and $\{x_i\}_{i\in \gS} \coloneqq \{x_i \mid i \in \gS\}$ for any set $\gS$. We let $\abs{\gS}$ be the size of any set $\gS$. 

\textit{Big O notations.} We use $\lesssim$ to denote ``less than'' up to some numerical constants and polylogarithmic factors, and $\gtrsim$ and $\asymp$ are defined likewise. We also use the big O notations, $\tilde O$, $\gO$, $\Omega$, where $\gO$, $\Omega$ hide numerical constants, $\tilde\gO$ hides numerical constants and polylogarithmic factors.

\textit{Notations in proofs.} For convenience, we will use ``$\textrm{T}_n$'' to denote the $n$-th term on the right hand side in some equation in the following proofs. We will use $\pm$ to mean ``add ($+$)'' and then ``subtract ($-$)'' the term: $a\pm b$ means $a - b+ b$.

Importantly, $\pi$ is a permutation of $\{0,1,\ldots, N-1\}$, and it serves as the training orders of data examples in SGD or training orders of clients in FL. Next, we need to define an operation on $\pi$ as done in \citet[Appendix~B]{lu2022grab} and \citet[Appendix C.4]{cooper2023coordinating}:
\begin{align*}
	\pi^{-1}(i) \coloneq j \text{ such that } \pi(j) = i,\quad i, j \in \{0,1,\ldots, N-1\}
\end{align*}
It represents that the index of $i$ in the permutation $\pi$ is $j$, where $i,j \in \{0,1,\ldots, N-1\}$. This operation will be very useful in Appendices~\ref{apx-subsec:GraB},~\ref{apx-subsec:PairGraB} and \ref{apx-subsec:balancing}. In addition, according to the definition, we have
\begin{align*}
	\pi^{-1} \left( \pi(j) \right) = j\, .
\end{align*}
\begin{proof}
	Assume that $\pi^{-1} \left( \pi(j) \right) = k \neq j$. Then, according to the definition, we get $\pi(j) = \pi(k)$, which implies that $j = k$. This contradicts our assumption. Thus, we have $\pi^{-1} \left( \pi(j) \right) = k = j$.
\end{proof}

\section{Algorithms}
\label{apx-sec:algs}

In this section, we provide more details about GraBs.

\subsection{Preliminaries of GraBs}
\label{apx-subsec:pre-GraBs}

\begin{minipage}[t]{0.5\linewidth}
\IncMargin{1em}
\begin{algorithm}[H]
	\DontPrintSemicolon
	\SetAlgoNoEnd
	\SetAlgoVlined
	\caption{Balancing \citep{alweiss2021discrepancy}}\label{alg:balancing}
	
	\let\oldnl\nl
	\newcommand{\nonl}{\renewcommand{\nl}{\let\nl\oldnl}}
	\SetKwFunction{Balance}{Balance} % Declare the function
	
	\SetKwProg{Function}{Function}{}{} % Define Function keyword
	
	\Function{\Balance{$\{\rvz_n\}_{n=0}^{N-1}$}}{
		Initialize running sum $\rvs$, hyperparameter $c$\\
		Initialize $\{\epsilon_n\}$ for assigned signs\\
		\For{$n = 0,\dots, N-1$}{
			Compute $\tilde p \gets \frac{1}{2}- \frac{\innerprod{\rvs, \rvz_n}}{2c}$\\
			Assign signs:\\
			\nonl\qquad$\epsilon_n \gets +1$ with probability $\tilde p$;\\
			\nonl\qquad$\epsilon_n \gets -1$ with probability $1-\tilde p$\\
			Update $\rvs\gets \rvs+ \epsilon_n \cdot \rvz_n$
		}
		\Return{\textnormal{the assigned signs} $\{\epsilon_n\}$}
	}
\end{algorithm}
\DecMargin{1em}
\end{minipage}
\hfill
\begin{minipage}[t]{0.5\linewidth}
\IncMargin{1em}
\begin{algorithm}[H]
	\DontPrintSemicolon
	\SetAlgoNoEnd
	\SetAlgoVlined
	\caption{Reordering \citep{harvey2014near}}\label{alg:reordering}
	
	\SetKwFunction{Reorder}{Reorder}
	
	\SetKwProg{Function}{Function}{}{} % Define Function keyword
	
	\Function{\Reorder{$\pi$, $\{\epsilon_n\}_{n=0}^{N-1}$}}{
		Initialize two lists $L_{\text{positive}} \leftarrow [\hspace{0.1em}]$, $L_{\text{negative}} \leftarrow [\hspace{0.1em}]$\;
		\For{$n = 0,\dots, N-1$}{
			\If{$\epsilon_{n} = +1$}{
				Append $\pi(n)$ to $L_{\text{positive}}$
			}
			\Else{
				Append $\pi(n)$ to $L_{\text{negative}}$
			}
		}
		$\pi' = \textnormal{concatenate}(L_{\textnormal{positive}}, \textnormal{reverse}(L_{\textnormal{negative}}))$\\
		\Return{\textnormal{the new order $\pi'$}}
	}
\end{algorithm}
\DecMargin{1em}
\end{minipage}

Recall that our goal is to find a permutation to minimize the order error (Notably, in GraBs, $\bar \phi_q$ is defined by $\norm{\cdot}_{\infty}$)
\begin{align*}%\textstyle
	\bar \phi_q \coloneqq \max_{n\in[N]} \norm{ \sum_{i=0}^{n-1} \left(\nabla f_{\pi_q(i)}(\rvx_q) - \nabla f(\rvx_q)\right) }_{\infty}\,,
\end{align*}
which is aligned with the goal of herding \citep{welling2009herding}. With this insight, \citet{lu2022grab} proposed GraB (to produce good permutations online) based on the theory of herding and balancing \citep{harvey2014near,alweiss2021discrepancy}: Consider $N$ vectors $\{\rvz_n \}_{n=0}^{N-1}$ such that $\sum_{n=0}^{N-1} \rvz_n = 0$ and $\norm{\rvz_n} \leq 1$. First, for any permutation $\pi$, assign the signs $\{\epsilon_{n} \}_{n=0}^{N-1}$ ($\epsilon_n \in \{-1,+1\}$) to the permuted vectors $\{\rvz_{\pi(n)}\}_{n=0}^{N-1}$ using the \textit{balancing} algorithms (such as Algorithm~\ref{alg:balancing}). Second, with the assigned signs and the old permutation $\pi$, produce a new permutation $\pi'$ by the \textit{reordering} algorithm (that is, Algorithm~\ref{alg:reordering}). Then,% the herding errors satisfy the relation
\begin{align}
	\max_{n \in [N]} \norm{\sum_{i=0}^{n-1}\rvz_{\pi'(i)}}_{\infty}
	&\leq \frac{1}{2} \max_{n \in [N]} \norm{\sum_{i=0}^{n-1}\rvz_{\pi(i)}}_{\infty} +  \frac{1}{2} \max_{n \in [N]} \norm{\sum_{i=0}^{n-1}\epsilon_i\cdot  \rvz_{\pi(i)}}_{\infty} \label{eq:ex-grab:key relation^2},
\end{align}
where we call the three terms, the herding error under $\pi'$, the herding error under $\pi$, and the signed herding error under $\pi$, respectively (see Lemma~\ref{lem:basic-balancing-reordering}). Ineq.~\eqref{eq:ex-grab:key relation^2} ensures that the herding error will be reduced (from $\pi$ to $\pi'$) as long as the signed herding error is small. That is, the herding error can be progressively reduced by balancing and reordering the vectors. By iteratively applying this process (balancing and then reordering), the herding error will approach the signed herding error, which is proved to be $\tilde\gO\left( 1\right)$, if the signs are assigned by Algorithm~\ref{alg:balancing} \citep[Theorem~1.1]{alweiss2021discrepancy}. 

Now, we introduce the concrete GraB algorithms. To present the key idea of GraBs, as well as our theory, we propose GraB-proto and PairGraB-proto, where the former is a simplified version of the original GraB algorithm \citep{lu2022grab}, and the latter is a simplified version of PairGraB algorithm \citep{lu2022grab,cooper2023coordinating}.
\begin{itemize}
	\item GraB-proto. Use \texttt{BasicBR} (Algorithm~\ref{alg:BasicBR}) as the \texttt{Permute} function in Algorithm~\ref{alg:SGD}, with the inputs of $\pi_q$, $\{\nabla f_{\pi_q(n)}(\rvx_q)\}_{n=0}^{N-1}$ and $\nabla f(\rvx_q)$, for each epoch $q$.
	\item PairGraB-proto. Use \texttt{PairBR} (Algorithm~\ref{alg:PairBR}) as the \texttt{Permute} function in Algorithm~\ref{alg:SGD}, with the inputs of $\pi_q$, $\{\nabla f_{\pi_q(n)}(\rvx_q)\}_{n=0}^{N-1}$ and $\nabla f(\rvx_q)$, for each epoch $q$.
\end{itemize}
The main distinction is that GraB-proto uses the basic balancing and reordering algorithm (\texttt{BasicBR}) while PairGraB-proto uses the pair balancing and reordering algorithm (\texttt{PairBR}). The advantage of \texttt{PairBR} is that it is free of \textit{centering} the input vectors in the practical implementation. As shown in Algorithm~\ref{alg:PairBR} (Lines \ref{alg-line:PairBR:diff}--\ref{alg-line:PairBR:diff^2}), it balances the difference of two centered vectors, which is equivalent to balancing the difference of the two original vectors as the mean vectors are canceled out:
\begin{align*}%\textstyle
	\rvd_l = \left( \rvz_{2l} - \rvm \right) - \left( \rvz_{2l+1} - \rvm \right) = \rvz_{2l} - \rvz_{2l+1}.
\end{align*}
This advantage makes it seamlessly compatible with online algorithms such as SGD. Notably, compared with the original GraB and PairGraB algorithms, whose implementation details are deferred to Appendix~\ref{apx-subsec:imp-GraBs}, GraB-proto and PairGraB-proto are impractical in computation and storage, however, they are simple, and sufficient to support our theory.

Next, we briefly introduce the original GraB and PairGraB algorithms.
\begin{itemize}
	\item GraB. Use \texttt{BasicBR} (Algorithm~\ref{alg:BasicBR}) as the \texttt{Permute} function in Algorithm~\ref{alg:SGD}, with the inputs of $\pi_q$, $\{ \nabla f_{\pi_q(n)} (\rvx_q^n) \}_{n=0}^{N-1}$ and $\frac{1}{N}\sum_{n=0}^{N-1}\nabla f_{\pi_{q-1}(n)} (\rvx_{q-1}^n)$, for each epoch $q$.
	\item PairGraB. Use \texttt{PairBR} (Algorithm~\ref{alg:PairBR}) as the \texttt{Permute} function in Algorithm~\ref{alg:SGD}, with the inputs of $\pi_q$, $\{ \nabla f_{\pi_q(n)} (\rvx_q^n) \}_{n=0}^{N-1}$ and $\frac{1}{N}\sum_{n=0}^{N-1}\nabla f_{\pi_q(n)} (\rvx_q^n)$, for each epoch $q$.
\end{itemize}
They replace $\nabla f_{\pi_q(n)} (\rvx_q)$ in their prototype versions with the easily accessible $\nabla f_{\pi_q(n)} (\rvx_q^n)$, reducing the unnecessary computational cost. Besides, for GraB, to overcome the challenge of centering the gradients in the \texttt{BasicBR} algorithm, GraB uses the average of the stale gradients as the estimate of the actual average of the fresh gradients, to ``center'' the (fresh) gradients. This trick is not required for PairGraB. See the implementation details in Algorithms~\ref{alg:SGD:GraB} and \ref{alg:SGD:PairGraB}.

In this paper, we categorize the permutation-based SGD algorithms that use \texttt{BasicBR} or \texttt{PairBR} to produce the permutation as the GraB algorithms (or GraBs).

\tikzhighlight{highlight2}{pink!80}{{-1pt,12pt}}{{60pt, -6pt}}
\begin{minipage}[t]{0.5\linewidth}
\IncMargin{1em}
\begin{algorithm}[H]
	\caption{Basic Balancing and Reordering}
	\label{alg:BasicBR}
	\DontPrintSemicolon
	\SetAlgoNoEnd
	\SetAlgoVlined
	
	\SetKwProg{Function}{Function}{}{}
	\SetKwFunction{Balance}{Balance}
	\SetKwFunction{Reorder}{Reorder}
	\SetKwFunction{BasicBR}{BasicBR}
	
	\Function{\BasicBR{$\pi$, $\{\rvz_n\}_{n=0}^{N-1}$, $\rvm$}\footnote{The mean vector $\rvm$ is used to center the input vectors $\{\rvz_n\}_{n=0}^{N-1}$ (Line 2). In most cases, it is the average of the input vectors $\frac{1}{N}\sum_{n=0}^{N-1}\rvz_n$, except in the original GraB algorithm, where it is replaced by an estimate of the actual average.}}{
		Centering: $\{\rvc_n \coloneqq \rvz_n - \rvm \}_{n=0}^{N-1}$\\[2pt]
		\tikzmark{highlight2-start}$\{\epsilon_n\}_{n=0}^{N-1}$ $\gets$ \Balance{$\{\rvc_n\}_{n=0}^{N-1}$}\tikzmark{highlight2-end}\\[2pt]
		$\pi'$ $\gets$ \Reorder{$\pi$, $\{\epsilon_n\}_{n=0}^{N-1}$}\\
		\Return{$\pi'$}
	}
\end{algorithm}
\DecMargin{1em}
\end{minipage}
\tikzhighlight{highlight3}{pink!80}{{-1pt,12pt}}{{11em, -6pt}}
\begin{minipage}[t]{0.5\linewidth}
\IncMargin{1em}
\begin{algorithm}[H]
	\caption{Pair Balancing and Reordering}
	\label{alg:PairBR}
	
	\let\oldnl\nl
	\newcommand{\nonl}{\renewcommand{\nl}{\let\nl\oldnl}}
	\DontPrintSemicolon
	\SetAlgoNoEnd
	\SetAlgoVlined
	
	\SetKwProg{Function}{Function}{}{}
	\SetKwFunction{Balance}{Balance}
	\SetKwFunction{Reorder}{Reorder}
	\SetKwFunction{PairBR}{PairBR}
	
	\Function{\PairBR{$\pi$, $\{\rvz_n\}_{n=0}^{N-1}$, $\rvm$}}{
		\dottedbox{Centering: $\{\rvc_n \coloneqq \rvz_n - \rvm \}_{n=0}^{N-1}$}\footnote{The step of centering is not required in practical implementations}\\[1pt]
		\tikzmark{highlight3-start}%
		Compute $\{\rvd_l\coloneqq \rvc_{2l} - \rvc_{2l+1} \}_{l=0}^{\frac{N}{2}-1}$\label{alg-line:PairBR:diff}\\[1pt]
		$\{\tilde\epsilon_l\}_{l=0}^{\frac{N}{2}-1} \gets$ \Balance{$\{\rvd_l\}_{l=0}^{\frac{N}{2}-1}$}\label{alg-line:PairBR:diff^2}\\[1pt]
		Compute $\{\epsilon_{n}\}_{n=0}^{N-1}$ such that\\
		
		\nonl \Indp $\epsilon_{2l}=\tilde\epsilon_l$ and $\epsilon_{2l+1} = -\tilde\epsilon_l$ for $l=0,\ldots,\frac{N}{2}-1$\tikzmark{highlight3-end}\\[2pt]
		\Indm
		
		$\pi'$ $\gets$ \Reorder{$\pi$, $\{\epsilon_n\}_{n=0}^{N-1}$}\\
		\Return{$\pi'$}
	}
\end{algorithm}
\DecMargin{1em}
\end{minipage}

\subsection{Implementations of GraBs}
\label{apx-subsec:imp-GraBs}

The practical implementations of GraB are provided in Algorithms~\ref{alg:SGD:GraB} and \ref{alg:SGD:GraB^2}. The implementation of PairGraB is provided in Algorithm~\ref{alg:SGD:PairGraB}. The implementation of FL-GraB is provided in Algorithm~\ref{alg:FL-GraB}. As done in \citet{lu2022grab,cooper2023coordinating}, we use Algorithm~\ref{alg:assign sign} for the theories in this paper, while we use Algorithm~\ref{alg:assign sign^2} for the experiments on quadratic functions in Section~\ref{sec:exp} and the experiments on real data sets in Appendix~\ref{apx-sec:exps}.

Notably, Algorithm~\ref{alg:SGD:GraB} (the original algorithm in \citealt[Algorithm 4]{lu2022grab}) is logically equivalent to Algorithm~\ref{alg:SGD:GraB^2}. Compared with Algorithm~\ref{alg:SGD:GraB}, which updates the new order at the end of each step (Lines~\ref{alg-line:SGD:GraB:new order}--\ref{alg-line:SGD:GraB:new order^2}), Algorithm~\ref{alg:SGD:GraB^2} generates the new order at the end of each epoch (Line~\ref{alg-line:SGD:GraB2:new order}). In fact, in Algorithm~\ref{alg:SGD:GraB^2}, we can reorder the examples for multiple times with the same signs (see Line~\ref{alg-line:SGD:GraB2:new order}), which may be useful in practice. Similar variants can also be formulated for Algorithms~\ref{alg:SGD:PairGraB} and \ref{alg:FL-GraB}.

\begin{minipage}[t]{0.5\linewidth}
	\IncMargin{1em}
	\begin{algorithm}[H]
		\DontPrintSemicolon
		\SetAlgoNoEnd
		\SetAlgoVlined
		\caption{Assign signs. \citep{alweiss2021discrepancy}}\label{alg:assign sign}
		
		\let\oldnl\nl
		\newcommand{\nonl}{\renewcommand{\nl}{\let\nl\oldnl}}
		\SetKwFunction{AssignSign}{AssignSign} % Declare the function
		
		\SetKwProg{Function}{Function}{}{} % Define Function keyword
		
		\Function{\AssignSign{$\rvs$, $\rvz$, $c$}\footnote{$c$ is a hyperparameter. See \citet[Theorem 4]{lu2022grab}.}}{
			Compute $\tilde p \gets \frac{1}{2}- \frac{\innerprod{\rvs, \rvz}}{2c}$\\
			Assign signs:\\
			\nonl\qquad$\epsilon \gets +1$ with probability $\tilde p$;\\
			\nonl\qquad$\epsilon \gets -1$ with probability $1-\tilde p$\\
			\Return{$\epsilon$}
		}
	\end{algorithm}
	\DecMargin{1em}
\end{minipage}
\begin{minipage}[t]{0.5\linewidth}
	\IncMargin{1em}
	\begin{algorithm}[H]
		\DontPrintSemicolon
		\SetAlgoNoEnd
		\SetAlgoVlined
		\caption{Assign signs without normalization. \citep[Algorithm 5]{lu2022grab}}\label{alg:assign sign^2}
		
		\let\oldnl\nl
		\newcommand{\nonl}{\renewcommand{\nl}{\let\nl\oldnl}}
		\SetKwFunction{AssignSign}{AssignSign} % Declare the function
		\SetKwProg{Function}{Function}{}{} % Define Function keyword
		
		\Function{\AssignSign{$\rvs$, $\rvz$}}{
			\If{$\norm{\rvs+ \rvz} < \norm{\rvs - \rvz}$}{
				$\epsilon \gets +1$
			}
			\Else{
				$\epsilon \gets -1$
			}
			\Return{$\epsilon$}
		}
	\end{algorithm}
	\DecMargin{1em}
\end{minipage}

\begin{minipage}[t]{0.5\linewidth}
\IncMargin{1em}
\begin{algorithm}[H]
	\SetNoFillComment
	\SetAlgoNoEnd
	\DontPrintSemicolon
	\caption{GraB \citep[Algorithm 4]{lu2022grab}}
	\label{alg:SGD:GraB}
	
	\let\oldnl\nl
	\newcommand{\nonl}{\renewcommand{\nl}{\let\nl\oldnl}}
	
	\SetKwProg{Function}{Function}{}{}
	\SetKwFunction{Balancing}{Balancing}
	\SetKwFunction{Reordering}{Reordering}
	\SetKwFunction{Permuting}{Permuting}
	\SetKwFunction{AssignSign}{AssignSign} % Declare the function
	
	\KwIn{$\pi_0$, $\rvx_0$; \textbf{Output}: $\{\rvx_q\}$}
	Initialize $\rvs\gets \vzero$, $\rvm \gets \vzero$, $\rvm_{\text{stale}} \gets \vzero$\\
	\For{$q = 0, 1,\ldots, Q-1$}{
		$\rvs\gets \vzero$; $\rvm_{\text{stale}} \gets \rvm$; $\rvm\gets \vzero$; $l\gets 0$, $r\gets N-1$\\
		
		\For{$n=0,1,\ldots, N-1$}{
			Compute the gradient $\nabla f_{\pi_q(n)}(\rvx_q^n)$\\
			Update the parameter: $\rvx_{q}^{n+1} \gets \rvx_{q}^{n} - \gamma \nabla f_{\pi_q(n)}(\rvx_q^n)$\\
			Update the mean: $\rvm \gets \rvm + \frac{1}{N}\nabla f_{\pi_q(n)}(\rvx_q^n)$\\
			Center the gradient $\rvc \gets \nabla f_{\pi_q(n)}(\rvx_q^n) - \rvm_{\text{stale}}$\\
			Assign the sign: $\epsilon \gets$ \AssignSign{$\rvs, \rvc$}\\
			Update the sign sum: $\rvs\gets \rvs+ \epsilon \cdot \nabla f_{\pi_q(n)}(\rvx_q^n)$\\
			\If{$\epsilon = +1$}{\label{alg-line:SGD:GraB:new order}
				 $\pi_{q+1}(l) \gets \pi_q(n)$; $l \gets l+1$.
			}
			\Else{
				$\pi_{q+1}(r) \gets \pi_q(n)$; $r \gets r-1$.\label{alg-line:SGD:GraB:new order^2}
			}
		}
		Update the parameter: $\rvx_{q+1} \gets \rvx_q^{N}$
	}
\end{algorithm}
\DecMargin{1em}
\end{minipage}
\begin{minipage}[t]{0.5\linewidth}
\IncMargin{1em}
\begin{algorithm}[H]
	\SetNoFillComment
	\SetAlgoNoEnd
	\DontPrintSemicolon
	\caption{GraB}
	\label{alg:SGD:GraB^2}
	
	\let\oldnl\nl
	\newcommand{\nonl}{\renewcommand{\nl}{\let\nl\oldnl}}
	
	\SetKwProg{Function}{Function}{}{}
	\SetKwFunction{Balancing}{Balancing}
	\SetKwFunction{Reorder}{Reorder}
	\SetKwFunction{Permuting}{Permuting}
	\SetKwFunction{AssignSign}{AssignSign} % Declare the function
	
	\KwIn{$\pi_0$, $\rvx_0$; \textbf{Output}: $\{\rvx_q\}$}
	Initialize $\rvs\gets \vzero$, $\rvm \gets \vzero$, $\rvm_{\text{stale}} \gets \vzero$\\
	\For{$q = 0, 1,\ldots, Q-1$}{
		$\rvs\gets \vzero$; $\rvm_{\text{stale}} \gets \rvm$; $\rvm\gets \vzero$; $l\gets 0$, $r\gets N-1$\\
		
		\For{$n=0,1,\ldots, N-1$}{
			Compute the gradient $\nabla f_{\pi_q(n)}(\rvx_q^n)$\\
			Update the parameter: $\rvx_{q}^{n+1} \gets \rvx_{q}^{n} - \gamma \nabla f_{\pi_q(n)}(\rvx_q^n)$\\
			Update the mean: $\rvm \gets \rvm + \frac{1}{N}\nabla f_{\pi_q(n)}(\rvx_q^n)$\\
			Center the gradient $\rvc \gets \nabla f_{\pi_q(n)}(\rvx_q^n) - \rvm_{\text{stale}}$\\
			Assign the sign: $\epsilon_n \gets$ \AssignSign{$\rvs, \rvc$}\\
			Update the sign sum: $\rvs\gets \rvs + \epsilon_n \cdot \nabla f_{\pi_q(n)}(\rvx_q^n)$\\
		}
		Update the parameter: $\rvx_{q+1} \gets \rvx_q^{N}$\\
		$\pi_{q+1} \gets$ \Reorder{$\pi_q$, $\{\epsilon_n\}_{n=0}^{N-1}$}\label{alg-line:SGD:GraB2:new order}\footnote{We can reorder the examples for multiple times with the same signs in this step.}\\
	}
\end{algorithm}
\DecMargin{1em}
\end{minipage}

\IncMargin{1em}
\begin{algorithm}[H]
	\SetNoFillComment
	\SetAlgoNoEnd
	\DontPrintSemicolon
	\caption{PairGraB}
	\label{alg:SGD:PairGraB}
	
	\let\oldnl\nl
	\newcommand{\nonl}{\renewcommand{\nl}{\let\nl\oldnl}}
	
	\SetKwProg{Function}{Function}{}{}
	\SetKwFunction{Balancing}{Balancing}
	\SetKwFunction{Reordering}{Reordering}
	\SetKwFunction{Permuting}{Permuting}
	\SetKwFunction{AssignSign}{AssignSign} % Declare the function
	
	\KwIn{$\pi_0$, $\rvx_0$; \textbf{Output}: $\{\rvx_q\}$}
	\For{$q = 0, 1,\ldots, Q-1$}{
		$\rvs\gets\vzero$; $\rvd\gets \vzero$, $l\gets 0$, $r\gets N-1$\\
		\For{$n=0,1,\ldots, N-1$}{
			Compute the gradient $\nabla f_{\pi_q(n)}(\rvx_q^n)$\\
			Update the parameter: $\rvx_{q}^{n+1} \gets \rvx_{q}^{n} - \gamma \nabla f_{\pi_q(n)}(\rvx_q^n)$\\
			
			\If{$(n+1) \mod 2 = 0$}{
				Compute the difference: $\rvd \gets \nabla f_{\pi_q(n-1)} - \nabla f_{\pi_q(n)}$\\
				Assign the sign: $\epsilon \gets$ \AssignSign{$\rvs, \rvd$}\\
				Update the sign sum: $\rvs \gets \rvs + \epsilon \cdot \rvd$\\
				\If{$\epsilon = +1$}{
					$\pi_{q+1}(l) \gets \pi_q(n)$; $l \gets l+1$\\
					$\pi_{q+1}(r) \gets \pi_q(n-1)$; $r \gets r-1$
				}
				\Else{
					$\pi_{q+1}(l) \gets \pi_q(n-1)$; $l \gets l+1$\\
					$\pi_{q+1}(r) \gets \pi_q(n)$; $r \gets r-1$
				}
			}
		}
		Update the parameter: $\rvx_{q+1} \gets \rvx_q^{N}$\\
	}
\end{algorithm}
\DecMargin{1em}

\IncMargin{1em}
\begin{algorithm}[H]
	\SetNoFillComment
	\SetAlgoNoEnd
	\DontPrintSemicolon
	\caption{FL-GraB (Server-side)}
	\label{alg:FL-GraB}
	
	\let\oldnl\nl
	\newcommand{\nonl}{\renewcommand{\nl}{\let\nl\oldnl}}
	
	\SetKwProg{Function}{Function}{}{}
	\SetKwFunction{AssignSign}{AssignSign}
	\SetKwFunction{Reorder}{Reorder}
	
	\KwIn{$\pi_0$, $\rvx_0$; \textbf{Output}: $\{\rvx_q\}$}
	\For{$q = 0, 1,\ldots, Q-1$}{
		$\rvs\gets\vzero$; $\rvd\gets \vzero$, $l\gets 0$, $r\gets N-1$\\
		\For{$n=0,1,\ldots, N-1$}{
			Get the pseudo-gradient $\rvp_q^n = \sum_{k=0}^{K-1} \nabla f_{\pi_q(n)}(\rvx_{q,k}^n)$\\
			\tcc{Update the parameter}
			\If{$(n+1) \mod S=0$}{
				$\rvw \gets \rvw - \sum_{s=0}^{S-1} \rvp_q^{n-S}$\\
			}
			
			\If{$(n+1) \mod 2 = 0$}{
				\tcc{Balance}
				Compute the difference: $\rvd \gets \nabla f_{\pi_q(n-1)} - \nabla f_{\pi_q(n)}$\\
				Assign the signs: $\epsilon \gets$ \AssignSign{$\rvs, \rvd$}\\
				Update the sign sum: $\rvs \gets \rvs + \epsilon \cdot \rvd$\\
				\tcc{Update the new order}
				\If{$\epsilon = +1$}{
					$\pi_{q+1}(l) \gets \pi_q(n)$; $l \gets l+1$\\
					$\pi_{q+1}(r) \gets \pi_q(n-1)$; $r \gets r-1$
				}
				\Else{
					$\pi_{q+1}(l) \gets \pi_q(n-1)$; $l \gets l+1$\\
					$\pi_{q+1}(r) \gets \pi_q(n)$; $r \gets r-1$
				}
			}	
		}
		\tcc{Update the parameter}
		$\rvx_{q+1} \gets \rvx_q - \eta \left (\rvx_q -\rvw \right)$\\
	}
\end{algorithm}
\DecMargin{1em}

\section{Helper Lemmas}

\begin{lemma}\label{lem:step-size}
	For any parameters $r_0>0$, $T>0$, $c>0$ and $\gamma \leq \frac{1}{d}$, there exists constant step sizes $\gamma = \min\left\{\frac{1}{d},\left(  \frac{c r_0}{T}\right)^\frac{1}{3}\right\} \leq \frac{1}{d}$ such that
	\begin{align*}
		\Psi_T \coloneq \frac{r_0}{\gamma T} + c \gamma^2 \leq \frac{dr_0}{T} + 2\frac{c^{\frac{1}{3}}r_0^{\frac{2}{3}}}{T^{\frac{2}{3}}} = \gO\left(\frac{dr_0}{T} + \frac{c^{\frac{1}{3}}r_0^{\frac{2}{3}}}{T^{\frac{2}{3}}} \right).
	\end{align*}
\end{lemma}
\begin{proof}
	If $\frac{1}{d}\leq \left(  \frac{ r_0}{cT}\right)^\frac{1}{3}$, choosing $\gamma = \frac{1}{d}$ gives
	\begin{align*}
		\Psi_T = \frac{dr_0}{T} + \frac{c}{d^2} \leq \frac{dr_0}{T} + \frac{c^{\frac{1}{3}}r_0^{\frac{2}{3}}}{T^{\frac{2}{3}}}.
	\end{align*}
	If $\left(  \frac{ r_0}{cT}\right)^\frac{1}{3} \leq \frac{1}{d}$, choosing $\gamma = \left(  \frac{ r_0}{cT}\right)^\frac{1}{3}$ gives
	\begin{align*}
		\Psi_T = \frac{dr_0}{T} + \frac{c}{d^2} \leq \frac{c^{\frac{1}{3}}r_0^{\frac{2}{3}}}{T^{\frac{2}{3}}} + \frac{c^{\frac{1}{3}}r_0^{\frac{2}{3}}}{T^{\frac{2}{3}}} \leq 2\frac{c^{\frac{1}{3}}r_0^{\frac{2}{3}}}{T^{\frac{2}{3}}}.
	\end{align*}
	Thus,
	\begin{align*}
		\Psi_T \leq \frac{dr_0}{T} + 2\frac{c^{\frac{1}{3}}r_0^{\frac{2}{3}}}{T^{\frac{2}{3}}} = \gO\left(\frac{dr_0}{T} + \frac{c^{\frac{1}{3}}r_0^{\frac{2}{3}}}{T^{\frac{2}{3}}} \right).
	\end{align*}
\end{proof}

\begin{lemma}\label{lem:basic-balancing-reordering}
	Consider $N$ vectors $\{\rvz_n\}_{n=0}^{N-1}$ and a permutation $\pi$ of $\{0,1,\ldots, N-1\}$. Assign the signs $\{\epsilon_n\}_{n=0}^{N-1}$ ($\epsilon_n \in \{-1,+1\}$) by the balancing algorithms (such as Algorithm~\ref{alg:balancing}) to the permuted vectors under the permutation $\pi$ (that is, $\{\rvz_{\pi(n)}\}_{n=0}^{N-1}$). Let $\pi'$ be the new permutation produced by Algorithm~\ref{alg:reordering} with the input of the old permutation $\pi$ and the assigned signs $\{\epsilon_n\}_{n=0}^{N-1}$. Then,
	\begin{align*}
		\max_{n \in [N]} \norm{\sum_{i=0}^{n-1}\rvz_{\pi'(i)}}_{\infty} \leq \frac{1}{2} \max_{n \in [N]} \norm{\sum_{i=0}^{n-1}\rvz_{\pi(i)}}_{\infty} +  \frac{1}{2} \max_{n \in [N]} \norm{\sum_{i=0}^{n-1}\epsilon_i\cdot  \rvz_{\pi(i)}}_{\infty} + \norm{\sum_{i=0}^{N-1} \rvz_{i}}_{\infty} \,.%\label{eq:basic-balancing-reordering}
	\end{align*}
	Furthermore, suppose that the signs $\{\epsilon_n\}_{n=0}^{N-1}$ are assigned by Algorithm~\ref{alg:balancing}. If $\norm{\rvz_{n}}_{2} \leq a$ for all $n \in \{0,1,\ldots, N-1\}$ and $\norm{\sum_{i=0}^{N-1} \rvz_{i}}_{\infty} \leq b$, then, with probability at least $1-\delta$,
	\begin{align*}
		\max_{n \in [N]} \norm{\sum_{i=0}^{n-1}\rvz_{\pi'(i)}}_{\infty} \leq \frac{1}{2} \max_{n \in [N]} \norm{\sum_{i=0}^{n-1}\rvz_{\pi(i)}}_{\infty} +  \frac{1}{2}C a + b,
	\end{align*}
	where $C = 30\log (\frac{dN}{\delta}) = \gO\left( \log\left(\frac{dN}{\delta}\right) \right) = \tilde \gO (1)$ is from \citet[Theorem~1.1]{alweiss2021discrepancy}.
\end{lemma}
\begin{proof}	
	This is Lemma~5 in \citet{lu2022grab} and we reproduce it for completeness. Let $M^+ = \{ i\in\{0, 1, \ldots, N-1\}\mid \epsilon_i = +1 \}$ and $M^- = \{ i\in\{0, 1, \ldots, N-1\}\mid \epsilon_i = -1 \}$. Then, for any $n \in \{1, 2\ldots, N\}$,
	\begin{align}
		\sum_{i=0}^{n-1} \rvz_{\pi(i)} + \sum_{i=0}^{n-1} \epsilon_{i} \cdot \rvz_{\pi(i)} = 2\cdot \sum_{ i\in M^+ \cap \{0,1,\ldots, n-1\} } \rvz_{\pi(i)} \label{eq:basic-order-relation:positive}\\
		\sum_{i=0}^{n-1} \rvz_{\pi(i)} - \sum_{i=0}^{n-1} \epsilon_{i} \cdot \rvz_{\pi(i)} = 2\cdot \sum_{ i\in M^- \cap \{0,1,\ldots, n-1\} } \rvz_{\pi(i)} \label{eq:basic-order-relation:negative}
	\end{align}
	Now, we show one simple instance for better understanding of the equalities. Let the vectors be $\rvz_0, \rvz_1, \rvz_2, \rvz_3$ and the old permutation $\pi$ is $0,1,3,2$, which implies the permuted vectors $\rvz_{\pi(0)}, \rvz_{\pi(1)}, \rvz_{\pi(2)}, \rvz_{\pi(3)}$ under the old permutation $\pi$ are $\rvz_0, \rvz_1, \rvz_3, \rvz_2$. Let the assigned signs $\epsilon_0, \epsilon_1, \epsilon_2, \epsilon_3$ for the permuted vectors be $+1,-1,-1,+1$. Then, $M^+ = \{ 0, 3\}$, $M^-=\{1,2\}$. The results are in Table~\ref{tab:relation-instance}.
	\begin{table}
		\centering
		\caption{A simple instance.}
		\label{tab:relation-instance}
		\begin{tabular}{l|ll|l|l}
			\toprule
			$n$ &$\sum_{i=0}^{n-1} \rvz_{\pi(i)}$ &$\sum_{i=0}^{n-1} \epsilon_{i} \cdot \rvz_{\pi(i)}$ &$2\cdot \sum_{ i\in M^+ \cap \{0,1,\ldots, n-1\}  } \rvz_{\pi(i)}$ &$2\cdot \sum_{ i\in M^- \cap \{0,1,\ldots, n-1\}  } \rvz_{\pi(i)}$\\\midrule
			$1$ &$\rvz_0$  &$\rvz_0$ &$2\rvz_0$ &$0$\\
			$2$ &$\rvz_0+\rvz_1$ &$\rvz_0-\rvz_1$ &$2\rvz_0$ &$2\rvz_1$\\
			$3$ &$\rvz_0+\rvz_1+\rvz_3$ &$\rvz_0-\rvz_1-\rvz_3$ &$2\rvz_0$ &$2\rvz_1+2\rvz_3$\\
			$4$ &$\rvz_0+\rvz_1+\rvz_3+\rvz_2$ &$\rvz_0-\rvz_1-\rvz_3+\rvz_2$ &$2\rvz_0+2\rvz_2$ &$2\rvz_1+2\rvz_3$\\
			\bottomrule
		\end{tabular}
	\end{table}
	
	By using triangular inequality, for any $n \in \{1, 2\ldots, N\}$, we have
	\begin{align*}
		\norm{\sum_{ i\in M^+ \cap \{0,1,\ldots, n-1\} } \rvz_{\pi(i)}}_{\infty} \leq \frac{1}{2}\norm{\sum_{i=0}^{n-1} \rvz_{\pi(i)}}_{\infty} + \frac{1}{2}\norm{\sum_{i=0}^{n-1} \epsilon_{i} \cdot \rvz_{\pi(i)}}_{\infty}\\
		\norm{\sum_{ i\in M^- \cap \{0,1,\ldots, n-1\} } \rvz_{\pi(i)}}_{\infty} \leq \frac{1}{2}\norm{\sum_{i=0}^{n-1} \rvz_{\pi(i)}}_{\infty} + \frac{1}{2}\norm{\sum_{i=0}^{n-1} \epsilon_{i} \cdot \rvz_{\pi(i)}}_{\infty}
	\end{align*}
	
	Next, we consider the upper bound of $\norm{ \sum_{i=0}^{n'-1} \rvz_{\pi'(i)} }_{\infty}$ for all $n' \in \{1,2, \ldots, N\}$. Recall that Algorithm~\ref{alg:reordering} puts the vectors with positive assigned signs in the front of the new permutation and the vectors with negative assigned signs in the back of the new permutation.
	
	If $n' \leq \abs{M^+}$ ($\abs{M^+}$ denotes the size of $M^+$), we get
	\begin{align*}
		\norm{ \sum_{i=0}^{n'-1} \rvz_{\pi'(i)} }_{\infty} &\leq \max_{n\in [N]}\norm{\sum_{ i\in M^+ \cap \{0,1,\ldots, n-1\} } \rvz_{\pi(i)}}_{\infty} \\
		&\leq \frac{1}{2}\max_{n\in [N]}\norm{\sum_{i=0}^{n-1} \rvz_{\pi(i)}}_{\infty} + \frac{1}{2}\max_{n\in [N]}\norm{\sum_{i=0}^{n-1} \epsilon_{i} \cdot \rvz_{\pi(i)}}_{\infty}
	\end{align*}

	If $n' > \abs{M^+}$ ($\abs{M^-}$ denotes the size of $M^+$), we get
	\begin{align*}
		\norm{ \sum_{i=0}^{n'-1} \rvz_{\pi'(i)} }_{\infty} 
		&= \norm{\sum_{i=0}^{N-1} \rvz_{\pi'(i)} - \sum_{i=n'}^{N-1} \rvz_{\pi'(i)}  }_{\infty}\\
		&\leq \norm{\sum_{i=0}^{N-1} \rvz_{\pi'(i)}}_{\infty} + \norm{\sum_{i=n'}^{N-1} \rvz_{\pi'(i)}}_{\infty}\\
		&\leq \norm{\sum_{i=0}^{N-1} \rvz_{\pi'(i)}}_{\infty} + \max_{n\in [N]}\norm{\sum_{ i\in M^- \cap \{0,1,\ldots, n-1\} } \rvz_{\pi(i)}}_{\infty} \\
		&\leq \norm{\sum_{i=0}^{N-1} \rvz_{i}}_{\infty} + \frac{1}{2}\max_{n\in [N]}\norm{\sum_{i=0}^{n-1} \rvz_{\pi(i)}}_{\infty} + \frac{1}{2}\max_{n\in [N]}\norm{\sum_{i=0}^{n-1} \epsilon_{i} \cdot \rvz_{\pi(i)}}_{\infty}
	\end{align*}
	Thus we combine the two cases and get the relation
	\begin{align*}
		\max_{n \in [N]} \norm{\sum_{i=0}^{n-1}\rvz_{\pi'(i)}}_{\infty} \leq \frac{1}{2} \max_{n \in [N]} \norm{\sum_{i=0}^{n-1}\rvz_{\pi(i)}}_{\infty} +  \frac{1}{2} \max_{n \in [N]} \norm{\sum_{i=0}^{n-1}\epsilon_i\cdot  \rvz_{\pi(i)}}_{\infty} + \norm{\sum_{i=0}^{N-1} \rvz_{i}}_{\infty}
	\end{align*}
	
	Using \citet{alweiss2021discrepancy}'s Theorem~1.1, for all $n \in [N]$, we have
	\begin{align*}
		\norm{\sum_{i=0}^{n-1}\epsilon_i\cdot  \rvz_{\pi(i)}}_{\infty} = \norm{\sum_{i=0}^{n-1}\epsilon_i\cdot  \frac{\rvz_{\pi(i)}}{\max_{j\in\{0,1,\ldots,N-1\}}\norm{\rvz_{\pi(j)}}_2}}_{\infty} \cdot \max_{j\in\{0,1,\ldots,N-1\}}\norm{ \rvz_{\pi(j)} }_2 \leq Ca
	\end{align*}
	Then, using $\norm{\sum_{i=0}^{N-1} \rvz_{i}}_{\infty} \leq b$, we get the claimed bound.
\end{proof}

\begin{lemma}\label{lem:pair-balancing-reordering}
	Let $\pi$, $\{\rvz_{\pi(n)}\}_{n=0}^{N-1}$ and $\frac{1}{N}\sum_{n=0}^{N-1} \rvz_{\pi(n)}$ be the inputs of Algorithm~\ref{alg:PairBR}, and $\pi'$ be the corresponding output. Suppose that $N \mod 2 = 0$. If $\norm{\rvz_{n}}_{2} \leq a$ for all $n \in \{0,1,\ldots, N-1\}$ and $\norm{\sum_{i=0}^{N-1} \rvz_{i}}_{\infty} \leq b$, then, with probability at least $1-\delta$,
	 \begin{align*}
	 	\max_{n \in [N]} \norm{\sum_{i=0}^{n-1}\rvz_{\pi'(i)}}_{\infty} \leq \frac{1}{2} \max_{n \in [N]} \norm{\sum_{i=0}^{n-1}\rvz_{\pi(i)}}_{\infty} +  C a + b,
	 \end{align*}
	where $C  = 30\log (\frac{dN}{2\delta}) = \gO\left( \log\left(\frac{dN}{\delta}\right) \right) = \tilde \gO (1)$ is from \citet[Theorem~1.1]{alweiss2021discrepancy}.
\end{lemma}
\begin{proof}
	We use $\tilde\epsilon_j$ to denote the assigned sign of $\rvd_j=\rvz_{\pi(2j)} - \rvz_{\pi(2j+1)}$ for all $j \in \{0,1, \ldots \frac{N}{2}-1\}$; we use $\epsilon_i$ to denote the assigned sign of $\rvz_{\pi(i)}$ for all $i \in \{0,1,\ldots, N-1\}$. Since $\{\rvd_j\}_{j=0}^{\frac{N}{2}-1}$ is the input of Algorithm~\ref{alg:balancing}, according to \citet{alweiss2021discrepancy}'s Theorem~1.1, for all $l \in \{1, 2,\ldots,\frac{N}{2}\}$,
	\begin{align*}
		\norm{ \sum_{j=0}^{l-1} \tilde \epsilon_j \rvd_j}_{\infty} = \norm{ \sum_{j=0}^{l-1} \tilde\epsilon_j\frac{\rvd_j}{\max_{j \in \{0,1,\ldots,l-1\}}\norm{\rvd_j}_2} }_{\infty} \cdot \max_{j \in \{0,1,\ldots,l-1\}}\norm{\rvd_j}_2\leq C \max_{j \in \{0,1,\ldots,l-1\}}\norm{\rvd_j}_2 \leq 2C a,
	\end{align*}
	where the last inequality is because for any $j \in \{0,1, \ldots \frac{N}{2}-1\}$,
	\begin{align*}
		\norm{\rvd_j}_2 = \norm{\rvz_{\pi(2j)} - \rvz_{\pi(2j+1)}}_{2} \leq \norm{\rvz_{\pi(2j)}}_{2} + \norm{\rvz_{\pi(2j+1)}}_{2} \leq 2a\,.
	\end{align*}
	
	We define $x_l$ and $y_l$ for $l \in \{1, 2,\ldots, \frac{N}{2}\}$,
	\begin{align*}
		x_l &= \sum_{j=0}^{l-1} \left( \rvz_{\pi(2j)} +  \rvz_{\pi(2j+1)} \right)=\sum_{i=0}^{2l-1} \rvz_{\pi(i)}\\
		y_l &= \sum_{j=0}^{l-1} \left( \epsilon_{2j}\rvz_{\pi(2j)} + \epsilon_{2j+1} \rvz_{\pi(2j+1)} \right) = \sum_{j=0}^{l-1} \left( \tilde\epsilon_{j}\rvz_{\pi(2j)} - \tilde\epsilon_{j} \rvz_{\pi(2j+1)} \right) = \sum_{j=0}^{l-1} \tilde\epsilon_j \rvd_j
	\end{align*}
	Let $M^+ = \{ i\in\{0, 1, \ldots, N-1\}\mid \epsilon_i = +1 \}$ and $M^- = \{ i\in\{0, 1, \ldots, N-1\}\mid \epsilon_i = -1 \}$.
	Then, for all $l \in \{1, 2,\ldots, \frac{N}{2}\}$, it follows that
	\begin{align*}
		\sum_{i \in M^+ \cap \{0, 1,\ldots, 2l-1\}} \rvz_{\pi(i)} &= \frac{1}{2}\sum_{j=0}^{l-1} \left( (1+\tilde\epsilon_j)\rvz_{\pi(2j)} +  (1-\tilde\epsilon_j)\rvz_{\pi(2j+1)} \right)= \frac{1}{2} x_l + \frac{1}{2} y_l\\
		\sum_{i \in M^- \cap \{0, 1,\ldots, 2l-1\}} \rvz_{\pi(i)} &= \frac{1}{2}\sum_{j=0}^{l-1} \left( (1-\tilde\epsilon_j)\rvz_{\pi(2j)} +  (1+\tilde\epsilon_j)\rvz_{\pi(2j+1)} \right)= \frac{1}{2} x_l - \frac{1}{2} y_l
	\end{align*}
	In fact, this can be seen as one special (restricted) case of that discussed in Lemma~\ref{lem:basic-balancing-reordering}. Now, we show one simple instance for better understanding of the equalities. Let the vectors be $\rvz_0, \rvz_1, \rvz_2, \rvz_3$ and the old permutation $\pi$ is $0,1,3,2$, which implies the permuted vectors $\rvz_{\pi(0)}, \rvz_{\pi(1)}, \rvz_{\pi(2)}, \rvz_{\pi(3)}$ under the old permutation $\pi$ are $\rvz_0, \rvz_1, \rvz_3, \rvz_2$. Let the assigned signs $\epsilon_0, \epsilon_1, \epsilon_2, \epsilon_3$ for the permuted vectors be $+1,-1,-1,+1$ (equivalently, $\tilde\epsilon_0=+1, \tilde\epsilon_1=-1$). Then, $M^+ = \{ 0, 3\}$, $M^-=\{1,2\}$. The results are in Table~\ref{tab:relation-instance-pair}.
	\begin{table}
		\centering
		\caption{A simple instance.}
		\label{tab:relation-instance-pair}
		\begin{tabular}{l|ll|l|l}
			\toprule
			$l$ &$u_l$ &$v_l$ &$2\cdot \sum_{ i\in M^+ \cap \{0,1,\ldots, 2l-1\}  } \rvz_{\pi(i)}$ &$2\cdot \sum_{ i\in M^- \cap \{0,1,\ldots, 2l-1\}  } \rvz_{\pi(i)}$\\\midrule
			$1$ &$\rvz_0+\rvz_1$  &$\rvz_0-\rvz_1$ &$2\rvz_0$ &$2\rvz_1$\\
			$2$ &$\rvz_0+\rvz_1+\rvz_3+\rvz_2$ &$\rvz_0-\rvz_1-\rvz_3+\rvz_2$ &$2\rvz_0+2\rvz_2$ &$2\rvz_1+2\rvz_3$\\
			\bottomrule
		\end{tabular}
	\end{table}
	
	By using the triangle inequality, for all $l \in \{1, 2,\ldots, \frac{N}{2}\}$, we can get
	\begin{align*}
		\norm{\sum_{i \in M^+ \cap \{0, 1,\ldots, 2l-1\}} \rvz_{\pi(i)}}_{\infty} &\leq \frac{1}{2} \norm{x_l}_{\infty} + \frac{1}{2} \norm{y_l}_{\infty}\\
		&= \frac{1}{2}\norm{\sum_{i=0}^{2l-1} \rvz_{\pi(i)}}_{\infty} + \frac{1}{2}\norm{\sum_{j=0}^{l-1} \tilde\epsilon_j \cdot \rvd_{j}}_{\infty}\\
		&\leq \frac{1}{2}\norm{\sum_{i=0}^{2l-1} \rvz_{\pi(i)}}_{\infty} + Ca\,,\\
		\norm{\sum_{i \in M^- \cap \{0, 1,\ldots, 2l-1\}} \rvz_{\pi(i)}}_{\infty}
		&\leq \frac{1}{2}\norm{\sum_{i=0}^{2l-1} \rvz_{\pi(i)}}_{\infty} + C a\,.
	\end{align*}
	
	Next, we consider the upper bound of $\norm{ \sum_{i=0}^{l'-1} \rvz_{\pi'(i)} }_{\infty}$ for all $l' \in \{1,2, \ldots, N\}$. Recall that Algorithm~\ref{alg:reordering} puts the vectors with positive assigned signs in the front of the new permutation and the vectors with negative assigned signs in the back of the new permutation.
	
	If $l' \in \{ 1, 2, \ldots, \frac{N}{2} \}$, we get
	\begin{align*}
		\norm{\sum_{i=0}^{l'-1} \rvz_{\pi'(i)}}_{\infty} &= \norm{\sum_{i \in M^+ \cap \{0, 1,\ldots, 2l'-1\}} \rvz_{\pi(i)}}_{\infty} \leq \frac{1}{2}\norm{\sum_{i=0}^{2l'-1} \rvz_{\pi(i)}}_{\infty} + C a\,.
	\end{align*}
	Note that if $l' \in  \{ 1, 2, \ldots, \frac{N}{2} \}$, then $2l' \in \{ 2, 4, \ldots, N \}$. Thus, we can get
	\begin{align*}
		\norm{\sum_{i=0}^{l'-1} \rvz_{\pi'(i)}}_{\infty}
		&\leq \frac{1}{2}\max_{n\in [N]}\norm{\sum_{i=0}^{n-1} \rvz_{\pi(i)}}_{\infty} + Ca\,.
	\end{align*}
	If $l' \in \{ \frac{N}{2}+1,\frac{N}{2}+2,\ldots, N \}$, we get
	\begin{align*}
		\norm{\sum_{i=0}^{l'-1} \rvz_{\pi'(i)}}_{\infty} &= \norm{\sum_{i=0}^{N-1} \rvz_{\pi'(i)} - \sum_{i=l'}^{N-1} \rvz_{\pi'(i)}}_{\infty} \\
		&\leq \norm{\sum_{i=0}^{N-1} \rvz_{\pi'(i)}}_{\infty}+ \norm{\sum_{i=l'}^{N-1} \rvz_{\pi'(i)}}_{\infty}\\
		&= \norm{\sum_{i=0}^{N-1} \rvz_{\pi'(i)}}_{\infty}+ \norm{\sum_{i \in M^- \cap \{0, 1,\ldots, 2(N-l')-1\}} \rvz_{\pi(i)}}_{\infty}\\
		&\leq \norm{\sum_{i=0}^{N-1} \rvz_{i}}_{\infty}+ \frac{1}{2}\norm{\sum_{i=0}^{2(N-l')-1} \rvz_{\pi(i)}}_{\infty}+ \frac{1}{2}\norm{\sum_{j=0}^{(N-l')-1} \tilde\epsilon_j \cdot \rvd_{j} }_{\infty}.
	\end{align*}
	Note that if $l' \in \{ \frac{N}{2}+1,\frac{N}{2}+2,\ldots, N \}$, then $(N-l') \in \{ 0, 1, \ldots, \frac{N}{2}-1 \}$ and $2(N-l') \in \{ 0, 2, \ldots, N-2 \}$. Thus,
	\begin{align*}
		\norm{\sum_{i=0}^{l'-1} \rvz_{\pi'(i)}}_{\infty} &\leq \norm{\sum_{i=0}^{N-1} \rvz_{i}}_{\infty} + \frac{1}{2} \max_{n\in [N]}\norm{\sum_{i=0}^{n-1} \rvz_{\pi(i)}}_{\infty} + C a\,.
	\end{align*}
	Thus, combining the two cases and using $\norm{\sum_{i=0}^{N-1} \rvz_{i}}_{\infty} \leq b$, we get
	\begin{align*}
		\max_{n \in [N]} \norm{\sum_{i=0}^{n-1}\rvz_{\pi'(i)}}_{\infty} \leq \frac{1}{2} \max_{n \in [N]} \norm{\sum_{i=0}^{n-1}\rvz_{\pi(i)}}_{\infty} +  Ca + b\,,
	\end{align*}
	which is the claimed bound.
\end{proof}

\begin{lemma}\label{lem:FL:pair-balancing-reordering}
	Let $\pi$, $\{\rvz_{\pi(n)}\}_{n=0}^{N-1}$ and $\frac{1}{N}\sum_{n=0}^{N-1} \rvz_{\pi(n)}$ be the inputs of Algorithm~\ref{alg:PairBR}, and $\pi'$ be the corresponding output. Suppose that $N \mod S = 0$ and $S \mod 2 = 0$. If $\norm{\rvz_{n}}_{2} \leq a$ for all $n \in \{0,1,\ldots, N-1\}$ and $\norm{\sum_{i=0}^{N-1} \rvz_{i}}_{\infty} \leq b$, then, with probability at least $1-\delta$,
	\begin{align*}
		\max_{m \in \{S, 2S, \ldots, N\}} \norm{\sum_{i=0}^{m-1}\rvz_{\pi'(i)}}_{\infty} \leq \frac{1}{2} \max_{m \in \{S, 2S, \ldots, N\}} \norm{\sum_{i=0}^{m-1}\rvz_{\pi(i)}}_{\infty} +  Ca + b\,,
	\end{align*}
	where $C  = 30\log (\frac{dN}{2\delta}) = \gO\left( \log\left(\frac{dN}{\delta}\right) \right) = \tilde \gO (1)$ is from \citet[Theorem~1.1]{alweiss2021discrepancy}.
\end{lemma}
\begin{proof}
	We use $\tilde\epsilon_j$ to denote the assigned sign of $\rvd_j=\rvz_{\pi(2j)} - \rvz_{\pi(2j+1)}$ for all $j \in \{0,1, \ldots \frac{N}{2}-1\}$; we use $\epsilon_i$ to denote the assigned sign of $\rvz_{\pi(i)}$ for all $i \in \{0,1,\ldots, N-1\}$. Since $\{\rvd_j\}_{j=0}^{\frac{N}{2}-1}$ is the input of Algorithm~\ref{alg:balancing}, according to \citet{alweiss2021discrepancy}'s Theorem~1.1, for all $l \in \{1, 2,\ldots,\frac{N}{2}\}$,
	\begin{align*}
		\norm{ \sum_{j=0}^{l-1} \tilde \epsilon_j \rvd_j}_{\infty} = \norm{ \sum_{j=0}^{l-1} \tilde\epsilon_j\frac{\rvd_j}{\max_{j \in \{0,1,\ldots,l-1\}}\norm{\rvd_j}_2} }_{\infty} \cdot \max_{j \in \{0,1,\ldots,l-1\}}\norm{\rvd_j}_2\leq C \max_{j \in \{0,1,\ldots,l-1\}}\norm{\rvd_j}_2 \leq 2Ca\,.
	\end{align*}
	where the last inequality is because for any $j \in \{0,1, \ldots \frac{N}{2}-1\}$,
	\begin{align*}
		\norm{\rvd_j}_2 = \norm{\rvz_{\pi(2j)} - \rvz_{\pi(2j+1)}}_{2} \leq \norm{\rvz_{\pi(2j)}}_{2} + \norm{\rvz_{\pi(2j+1)}}_{2} \leq 2a\,.
	\end{align*}
	
	We define $x_l$ and $y_l$ for $l \in \{1, 2,\ldots, \frac{N}{2}\}$,
	\begin{align*}
		x_l &= \sum_{j=0}^{l-1} \left( \rvz_{\pi(2j)} +  \rvz_{\pi(2j+1)} \right)=\sum_{i=0}^{2l-1} \rvz_{\pi(i)}\,,\\
		y_l &= \sum_{j=0}^{l-1} \left( \epsilon_{2j}\rvz_{\pi(2j)} + \epsilon_{2j+1} \rvz_{\pi(2j+1)} \right) = \sum_{j=0}^{l-1} \left( \tilde\epsilon_{j}\rvz_{\pi(2j)} - \tilde\epsilon_{j} \rvz_{\pi(2j+1)} \right) = \sum_{j=0}^{l-1} \tilde\epsilon_j \rvd_j\,.
	\end{align*}
	Let $M^+ = \{ i\in\{0, 1, \ldots, N-1\}\mid \epsilon_i = +1 \}$ and $M^- = \{ i\in\{0, 1, \ldots, N-1\}\mid \epsilon_i = -1 \}$.
	Then, for all $l \in \{1, 2,\ldots, \frac{N}{2}\}$, it follows that
	\begin{align*}
		\sum_{i \in M^+ \cap \{0, 1,\ldots, 2l-1\}} \rvz_{\pi(i)} &= \frac{1}{2}\sum_{j=0}^{l-1} \left( (1+\tilde\epsilon_j)\rvz_{\pi(2j)} +  (1-\tilde\epsilon_j)\rvz_{\pi(2j+1)} \right)= \frac{1}{2} x_l + \frac{1}{2} y_l\,,\\
		\sum_{i \in M^- \cap \{0, 1,\ldots, 2l-1\}} \rvz_{\pi(i)} &= \frac{1}{2}\sum_{j=0}^{l-1} \left( (1-\tilde\epsilon_j)\rvz_{\pi(2j)} +  (1+\tilde\epsilon_j)\rvz_{\pi(2j+1)} \right)= \frac{1}{2} x_l - \frac{1}{2} y_l\,.
	\end{align*}
	
	By using the triangle inequality, for all $l \in \{1, 2,\ldots, \frac{N}{2}\}$, we can get
	\begin{align*}
		\norm{\sum_{i \in M^+ \cap \{0, 1,\ldots, 2l-1\}} \rvz_{\pi(i)}}_{\infty} &\leq \frac{1}{2} \norm{x_l}_{\infty} + \frac{1}{2} \norm{y_l}_{\infty}\\
		&= \frac{1}{2}\norm{\sum_{i=0}^{2l-1} \rvz_{\pi(i)}}_{\infty} + \frac{1}{2}\norm{\sum_{j=0}^{l-1} \tilde\epsilon_j \cdot \rvd_{j}}_{\infty}\\
		&\leq \frac{1}{2}\norm{\sum_{i=0}^{2l-1} \rvz_{\pi(i)}}_{\infty} + Ca\,,\\
		\norm{\sum_{i \in M^- \cap \{0, 1,\ldots, 2l-1\}} \rvz_{\pi(i)}}_{\infty}
		&\leq \frac{1}{2}\norm{\sum_{i=0}^{2l-1} \rvz_{\pi(i)}}_{\infty} + Ca\,.
	\end{align*}
	
	Next, we consider the upper bound of $\norm{\sum_{i=0}^{l'-1} \rvz_{\pi'(i)}}_{\infty}$ for all $l' \in \{ \frac{1}{2}S, S, \frac{3}{2}S,\ldots, \frac{N}{S}\cdot S \}$.
	
	If $l' \leq \frac{N}{S}\cdot \frac{1}{2}S$, or equivalently, $l' \in \{ \frac{1}{2}S, S, \frac{3}{2}S,\ldots, \frac{N}{S}\cdot \frac{1}{2}S \} \subseteq \{1, 2,\ldots, \frac{N}{2}\}$, then we can get
	\begin{align*}
		\norm{\sum_{i=0}^{l'-1} \rvz_{\pi'(i)}}_{\infty} &= \norm{\sum_{i \in M^+ \cap \{0, 1,\ldots, 2l'-1\}} \rvz_{\pi(i)}}_{\infty}\leq \frac{1}{2}\norm{\sum_{i=0}^{2l'-1} \rvz_{\pi(i)}}_{\infty} + Ca\,.
	\end{align*}
	Then, note that if $l' \in \{ \frac{1}{2}S, S, \frac{3}{2}S,\ldots, \frac{N}{S}\cdot \frac{1}{2}S \}$, which implies that $2l' \in \{ S, 2S, 3S,\ldots, N \}$, then
	\begin{align*}
		\norm{\sum_{i=0}^{l'-1} \rvz_{\pi'(i)}}_{\infty} \leq \frac{1}{2} \max_{m\in \{ S, 2S, 3S,\ldots, N \}}\norm{\sum_{i=0}^{m-1} \rvz_{\pi(i)}}_{\infty} + C a\,.
	\end{align*}
	If $l' > \frac{N}{S}\cdot \frac{1}{2}S$, or equivalently, $l' \in \{ \left(\frac{N}{S}+1\right)\frac{S}{2}, \left(\frac{N}{S}+2\right)\frac{S}{2},\ldots, N \}$, then we can get
	\begin{align*}
		\norm{\sum_{i=0}^{l'-1} \rvz_{\pi'(i)}}_{\infty} &= \norm{\sum_{i=0}^{N-1} \rvz_{\pi'(i)} - \sum_{i=l'}^{N-1} \rvz_{\pi'(i)}}_{\infty} \\
		&\leq \norm{\sum_{i=0}^{N-1} \rvz_{\pi'(i)}}_{\infty}+ \norm{\sum_{i=l'}^{N-1} \rvz_{\pi'(i)}}_{\infty}\\
		&= \norm{\sum_{i=0}^{N-1} \rvz_{\pi'(i)}}_{\infty}+ \norm{\sum_{i \in M^- \cap \{0, 1,\ldots, 2(N-l')-1\}} \rvz_{\pi(i)}}_{\infty}\\
		&\leq \norm{\sum_{i=0}^{N-1} \rvz_{\pi'(i)}}_{\infty}+ \frac{1}{2}\norm{\sum_{i=0}^{2(N-l')-1} \rvz_{\pi(i)}}_{\infty}+ \frac{1}{2}\norm{\sum_{j=0}^{(N-l')-1} \tilde\epsilon_j \rvd_j}_{\infty}.
	\end{align*}
	Note that if $l' \in \{ \left(\frac{N}{S}+1\right)\frac{S}{2}, \left(\frac{N}{S}+2\right)\frac{S}{2},\ldots, N \}$, then $(N-l') \in \{ 0, \frac{S}{2}, S, \ldots, \left(\frac{N}{S}-1\right)\frac{S}{2} \}$ and $2(N-l') \in \{ 0, S,2S, \ldots, N-S \}$. Thus, we can get
	\begin{align*}
		\norm{\sum_{i=0}^{l'-1} \rvz_{\pi'(i)}}_{\infty} &\leq \norm{\sum_{i=0}^{N-1} \rvz_{\pi(i)}}_{\infty} + \frac{1}{2} \max_{m\in \{ S, 2S, 3S,\ldots, N-S \}}\norm{\sum_{i=0}^{m-1} \rvz_{\pi(i)}}_{\infty} + C a\\
		&\leq \norm{\sum_{i=0}^{N-1} \rvz_{\pi(i)}}_{\infty} + \frac{1}{2} \max_{m\in \{ S, 2S, 3S,\ldots, N\}}\norm{\sum_{i=0}^{m-1} \rvz_{\pi(i)}}_{\infty} + C a\,.
	\end{align*}
	
	The bounds for these two cases hold for all $l' \in \{ \frac{1}{2}S, S, \frac{3}{2}S,\ldots, \frac{N}{S}\cdot S \}$, which means that
	\begin{align*}
		\max_{l' \in \{ \frac{1}{2}S, S, \frac{3}{2}S,\ldots, \frac{N}{S}\cdot S \}} \norm{\sum_{i=0}^{l'-1} \rvz_{\pi'(i)}}_{\infty} \leq \frac{1}{2} \max_{m\in \{ S, 2S, 3S,\ldots, N \}}\norm{\sum_{i=0}^{m-1} \rvz_{\pi(i)}}_{\infty} + \norm{\sum_{i=0}^{N-1} \rvz_{\pi(i)}}_{\infty}+ Ca\,.
	\end{align*}
	Since $\{ S, 2S, 3S,\ldots, \frac{N}{S}\cdot S \} \subseteq \{ \frac{1}{2}S, S, \frac{3}{2}S,\ldots, \frac{N}{S}\cdot S \}$, then
	\begin{align*}
		\max_{m \in \{ S, 2S, 3S,\ldots,N\}} \norm{\sum_{i=0}^{m-1} \rvz_{\pi'(i)}}_{\infty} \leq \frac{1}{2} \max_{m\in \{ S, 2S, 3S,\ldots, N \}}\norm{\sum_{i=0}^{m-1} \rvz_{\pi(i)}}_{\infty} + \norm{\sum_{i=0}^{N-1} \rvz_{\pi(i)}}_{\infty}+ Ca\,.
	\end{align*}
	Using $\norm{\sum_{i=0}^{N-1} \rvz_{i}}_{\infty} \leq b$, we get the claimed bound.
\end{proof}

If using \texttt{BasicBR} (Algorithm~\ref{alg:BasicBR}), we are unable to get the similar relation (to Lemma~\ref{lem:FL:pair-balancing-reordering}) between
\begin{align*}
	\max_{m \in \{S, 2S, \ldots, N\}} \norm{\sum_{i=0}^{m-1}\rvz_{\pi'(i)}}_{\infty} \quad \text{and} \quad \max_{m \in \{S, 2S, \ldots, N\}} \norm{\sum_{i=0}^{m-1}\rvz_{\pi(i)}}_{\infty}
\end{align*}
with the existing theoretical techniques \citep{harvey2014near,lu2022grab}. This causes that we cannot get the upper bound for regularized-participation FL with (original) GraB, which depends on Algorithm~\ref{alg:BasicBR} (\texttt{BasicBR}). Now, we provide the intuitive analysis. As shown in Figure~\ref{fig:FL:balancing-reordering-instance}, we consider a simple instance with 24 vectors $\{\rvz_{n}\}_{n=0}^{23}$. Assume that the old older $\pi = 0,1,2,\ldots,23$. The permuted vectors are assigned with $\{+1,-1\}$ signs by some balancing algorithm, where the blue rectangles denote the vectors with the positive assigned signs and the yellow rectangles denote the vectors with the negative assigned signs. Let us focus on the blue rectangles. In the basic case (\texttt{BasicBR}), according to the analysis in Lemma~\ref{lem:basic-balancing-reordering} (specifically, Eqs.~\ref{eq:basic-order-relation:positive} and~\ref{eq:basic-order-relation:negative}), we can get the partial sum of the vectors with the positive assigned signs over consecutive chunks, each with a size of $S=8$. That is, $\sum_{n\in \{0,\ldots,6\}} \rvz_n$, $\sum_{n\in \{0,\ldots,6\}\cup\{8,\ldots,12\} } \rvz_n$ and $\sum_{n\in \{0,\ldots,6\}\cup\{8,\ldots,12\}\cup \{16,\ldots,20\} } \rvz_n$. Yet, in the new order $\pi'$, it is required to compute $\sum_{n\in \{0,\ldots,6\} \cup \{8\}} \rvz_n$. This is unachievable by the known information. However, in the pair case (\texttt{PairBR}), the number of vectors with the positive assigned signs are equal to the number of vectors with the negative assigned signs in each chunk. This characteristic makes it feasible for us to get the relation (as shown in Lemma~\ref{lem:FL:pair-balancing-reordering}), with the existing theoretical techniques. Yet, It is still open whether similar results can be derived for the basic case with other more advanced techniques.
\begin{figure}[h]
	\centering
	\includegraphics[width=0.6\linewidth]{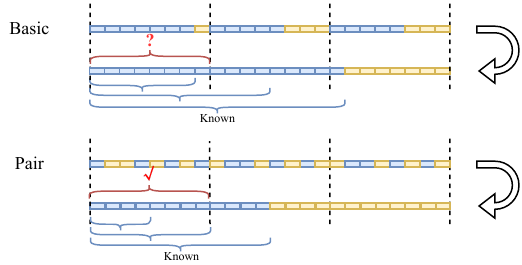}
	\caption{One instance of the basic and pair balancing and reordering algorithm (Algorithms~\ref{alg:BasicBR} and~\ref{alg:PairBR}) for FL. The top subfigure shows the instance of the basic case and the bottom subfigure shows the instance of the pair case. The blue rectangles denote the vectors with the positive assigned signs and the yellow rectangles denote the vectors with the negative assigned signs.}
	\label{fig:FL:balancing-reordering-instance}
\end{figure}

\section{Theorem~\ref{thm:SGD}}

\subsection{Order Error in SGD}
\label{apx-subsec:order error}

\citet{smith2021origin} says that, for small finite step sizes (It means $\gamma$ is large enough that terms of $\gO\left( \gamma^2 N^2 \right)$ may be significant, but small enough that terms of $\gO\left( \gamma^3 N^3 \right)$ are negligible), the cumulative updates of permutation-based SGD in one epoch are
\begin{align}
	\rvx^N - \rvx^0 &= - \gamma N \nabla f(\rvx^0) + \gamma^2 \sum_{n=0}^{N-1}\sum_{i<n} \nabla^2 f_{\pi(n)}(\rvx^0) \nabla f_{\pi(i)}(\rvx^0) + \gO( \gamma^3 N^3).\label{eq:order error analysis}
\end{align}

\begin{proof}[Proof of Eq.~\eqref{eq:order error analysis}]
	The Taylor expansion of $h$ at $x=x_0$ is $\sum_{n=0}^{\infty} \frac{1}{n!} h^{(n)} (x_0) (x-x_0)^n$. Here we only need
	\begin{align*}
		h(x) = h(x_0) + h'(x_0) (x-x_0) + \gO \left((x-x_0)^2\right).
	\end{align*}
	For permutation-based SGD, for any epoch $q\geq 0$,
	\begin{align*}
		\rvx_{q}^{N} - \rvx_{q}^{0} = -\gamma \sum_{n=0}^{N-1} \nabla f_{\pi(n)} \left(\rvx_q^n\right).
	\end{align*}
	Next, we drop the subscripts $q$ for convenience. For any $n\in\{0,1,\ldots,N-1\}$,
	\begin{align*}
		\rvx^{n} - \rvx^{0} = -\gamma \sum_{i=0}^{n-1} \nabla f_{\pi(i)} \left(\rvx^i\right).
	\end{align*}
	Then, using Taylor expansion of $\nabla f_{\pi(n)} \left(\rvx^n\right)$, we get
	\begin{align}
		\nabla f_{\pi(n)} \left(\rvx^n\right) &= \nabla f_{\pi(n)}\left( \rvx^0 \right) + \nabla \nabla f_{\pi(n)}(\rvx^0) \left( -\gamma \sum_{i=0}^{n-1} \nabla f_{\pi(i)} \left(\rvx^i\right) \right) + \gO\left( \gamma^2 N^2 \right)\nonumber\\
		&= \nabla f_{\pi(n)}\left( \rvx^0 \right) - \gamma \nabla \nabla f_{\pi(n)}(\rvx^0)  \sum_{i=0}^{n-1} \nabla f_{\pi(i)} \left(\rvx^i\right) + \gO\left( \gamma^2N^2\right).\label{eq:pf-def:order eror}
	\end{align}
	Then, note that the expansion of Eq.~\eqref{eq:pf-def:order eror} is also applied to $\nabla f_{\pi(i)} \left(\rvx^i\right)$ in Eq.~\eqref{eq:pf-def:order eror}. Using Eq.~\eqref{eq:pf-def:order eror} recursively, we get
	\begin{align*}
		\Term{2}{Eq.~\eqref{eq:pf-def:order eror}} &= - \gamma \nabla \nabla f_{\pi(n)}(\rvx^0)  \sum_{i=0}^{n-1} \nabla f_{\pi(i)} \left(\rvx^i\right)\\
		&= - \gamma \nabla \nabla f_{\pi(n)}(\rvx^0)  \sum_{i=0}^{n-1} \left( \nabla f_{\pi(i)}\left( \rvx^0 \right) - \gamma \nabla \nabla f_{\pi(i)}(\rvx^0)  \sum_{a=0}^{i-1} \nabla f_{\pi(a)} \left(\rvx^a\right) + \gO\left( \gamma^2N^2\right) \right)\\
		&= - \gamma \nabla \nabla f_{\pi(n)}(\rvx^0)  \sum_{i=0}^{n-1}\nabla f_{\pi(i)}\left( \rvx^0 \right) + \gO\left( \gamma^2N^2 \right) + \gO\left( \gamma^3N^3 \right).
	\end{align*}
	Substituting it gives
	\begin{align*}
		\nabla f_{\pi(n)} \left(\rvx^n\right) &= \nabla f_{\pi(n)}\left( \rvx^0 \right) - \gamma \nabla \nabla f_{\pi(n)}(\rvx^0)  \sum_{i=0}^{n-1}\nabla f_{\pi(i)}\left( \rvx^0 \right) + \gO\left( \gamma^2N^2 \right) + \gO\left( \gamma^3N^3 \right) + \gO\left( \gamma^2N^2\right)\\
		&= \nabla f_{\pi(n)}\left( \rvx^0 \right) - \gamma \nabla \nabla f_{\pi(n)}(\rvx^0)  \sum_{i=0}^{n-1}\nabla f_{\pi(i)}\left( \rvx^0 \right) + \gO\left( \gamma^2N^2 \right).
	\end{align*}
	At last, we get
	\begin{align*}
		\rvx^N - \rvx^0 &= -\gamma \sum_{n=0}^{N-1 } \nabla f_{\pi(n)} (\rvx^n) \\
		&= -\gamma \sum_{n=0}^{N-1}\nabla f_{\pi(n)}\left( \rvx^0 \right) + \gamma^2 \sum_{n=0}^{N-1} \nabla \nabla f_{\pi(n)}(\rvx^0)  \sum_{i=0}^{n-1}\nabla f_{\pi(i)}\left( \rvx^0 \right) + \gO\left( \gamma^3N^3 \right).
	\end{align*}
	After recovering the subscripts $q$ and noting $\rvx_{q}^{0} = \rvx_q$, we get Eq.~\eqref{eq:order error analysis}.
\end{proof}

\subsection{Parameter Deviation in SGD}

We define the maximum parameter deviation (drift) in any epoch $q$, $\Delta_q$ as
\begin{align*}
	\Delta_q = \max_{n \in \left[N\right]}\norm{\rvx_q^n- \rvx_q^0}_{p}\,.
\end{align*}

\begin{lemma}
	\label{lem:parameter drift}
	If $\gamma L_p N\leq \frac{1}{32}$, the maximum parameter drift is bounded:
	\begin{align*}
		\Delta_q &\leq \frac{32}{31}\gamma \bar \phi_q + \frac{32}{31} \gamma N \norm{\nabla f(\rvx_q)}_p,\\
		\left( \Delta_q\right)^2&\leq 3\gamma^2 \left(\bar \phi_q\right)^2 + 3\gamma^2 N^2 \norm{\nabla f(\rvx_q)}_p^2.
	\end{align*}
\end{lemma}
\begin{proof}
	In this lemma, we mainly focus on one epoch $q$, thus we drop the subscripts $q$ for convenience. For any $n\in \left[ N \right]$, it follows that
	\begin{align*}
		\norm{\rvx^n- \rvx^0}_{p}
		&= \gamma \norm{\sum_{i=0}^{n-1} \nabla f_{\pi(i)}(\rvx^i)}_{p} \nonumber\\
		&= \gamma\norm{\sum_{i=0}^{n-1} \left( \nabla f_{\pi(i)}(\rvx^i) - \nabla f_{\pi(i)}(\rvx^0) + \nabla f_{\pi(i)}(\rvx^0) - \nabla f(\rvx^0) + \nabla f(\rvx^0) \right) }_{p}\nonumber\\
		&\leq \gamma\norm{\sum_{i=0}^{n-1} \left(\nabla f_{\pi(i)}(\rvx^i)- \nabla f_{\pi(i)}(\rvx^0)\right) }_{p} + \gamma\norm{ \sum_{i=0}^{n-1} \left( \nabla f_{\pi(i)}(\rvx^0) - \nabla f(\rvx^0)\right) }_{p} + \gamma \norm{\sum_{i=0}^{n-1} \nabla f(\rvx^0)}_{p}\nonumber\\
		&\leq \gamma L_{p} \sum_{i=0}^{n-1}\norm{\rvx^i - \rvx^0}_{p} + \gamma \phi^n + \gamma n \norm{\nabla f(\rvx^0)}_{p}\nonumber\\
		&\leq \gamma L_{p} N \Delta + \gamma \bar \phi + \gamma N \norm{\nabla f(\rvx^0)}_{p}\nonumber.
	\end{align*}
	Note that this bound holds for any $n\in \left[ N \right]$. This means
	\begin{align*}
		\Delta \leq \gamma L_{p} N \Delta + \gamma \bar \phi + \gamma N \norm{\nabla f(\rvx^0)}_{p}.
	\end{align*}
	Then, using $\gamma L_{p} N \leq \frac{1}{32}$, we have
	\begin{align*}
		\Delta &\leq \frac{32}{31}\gamma \bar \phi + \frac{32}{31} \gamma N \norm{\nabla f(\rvx^0)}_p,\\
		\left( \Delta_q\right)^2&\leq 3\gamma^2 \left(\bar \phi_q\right)^2 + 3\gamma^2 N^2 \norm{\nabla f(\rvx_q)}_p^2 .
	\end{align*}
	Recovering the subscripts $q$ yields the final result.
\end{proof}

\subsection{Proof of Theorem~\ref{thm:SGD}}

\begin{proof}[Proof of Theorem~\ref{thm:SGD}]
	For permutation-based SGD, the cumulative updates over any epoch $q$ are
	\begin{align*}
		\rvx^{N}_q - \rvx^{0}_q =- \gamma \sum_{n=0}^{N-1}\nabla f_{\pi_q(n)}(\rvx^{n}_q)\,.
	\end{align*}
	Next, we focus on one single epoch, and therefore drop the subscripts $q$ for clarity. Since the global objective function $f$ is $L$-smooth, it follows that
	\begin{align*}
		f(\rvx^{N}) &\leq f(\rvx^0) + \langle \nabla f(\rvx^0), \rvx^N-\rvx^0\rangle + \frac{1}{2}L\norm{\rvx^N-\rvx^0}^2\,.
	\end{align*}
	After substituting $\rvx^{N} - \rvx^{0}$, we have
	\begin{align*}
		&\langle \nabla f(\rvx^0), \rvx_N-\rvx^0\rangle\\
		&= -\gamma N\left\langle \nabla f(\rvx^0), \frac{1}{N}\sum_{n=0}^{N-1}\nabla f_{\pi(n)}(\rvx^n) \right\rangle \\
		&= -\frac{1}{2}\gamma N\normsq{\nabla f(\rvx^0)} -\frac{1}{2}\gamma N \normsq{\frac{1}{N}\sum_{n=0}^{N-1}\nabla f_{\pi(n)}(\rvx^n)} + \frac{1}{2}\gamma N \normsq{\frac{1}{N}\sum_{n=0}^{N-1}\nabla f_{\pi(n)}(\rvx^n) - \nabla f(\rvx^0)}\,,
	\end{align*}
	where the second equality is due to $2 \langle \rvx,\rvy\rangle =  \norm{\rvx} + \norm{\rvy} - \norm{\rvx-\rvy}$.
	\begin{align*}
		\frac{1}{2}L\norm{\rvx^N - \rvx^0}^2 =\frac{1}{2}\gamma^2LN^2\normsq{\frac{1}{N}\sum_{n=0}^{N-1}\nabla f_{\pi(n)}(\rvx^{n})}.
	\end{align*}
	Next, plugging back, we get
	\begin{align*}
		f(\rvx^{N}) &\leq f(\rvx^0) + \langle \nabla f(\rvx^0), \rvx^N - \rvx^0\rangle + \frac{1}{2}L\norm{\rvx^N - \rvx^0}^2\\
		&\leq f(\rvx^0) -\frac{1}{2}\gamma N \normsq{\nabla f(\rvx^0)} -\frac{1}{2}\gamma N(1-\gamma LN) \E\normsq{\frac{1}{N}\sum_{n=0}^{N-1}\nabla f_{\pi(n)}(\rvx_n)} \\
		&\quad + \frac{1}{2}\gamma N \normsq{\frac{1}{N}\sum_{n=0}^{N-1}\nabla f_{\pi(n)}(\rvx^n) - \nabla f(\rvx^0)}.
	\end{align*}
	Since $\textcolor{blue}{\gamma L N \leq 1}$, we get
	\begin{align}
		f(\rvx^{N})
		\leq f(\rvx^0) -\frac{1}{2}\gamma N \normsq{\nabla f(\rvx^0)} + \frac{1}{2}\gamma N \normsq{\frac{1}{N}\sum_{n=0}^{N-1}\nabla f_{\pi(n)}(\rvx^n) - \nabla f(\rvx^0)}. \label{eq:pf-lem:gradient errors}
	\end{align}
	
	Since each local objective function $f_n$ is $L_{2,p}$-smooth, we have
	\begin{align*}
		\Term{3}{\eqref{eq:pf-lem:gradient errors}} &= \frac{1}{2}\gamma N \normsq{\frac{1}{N}\sum_{n=0}^{N-1}\nabla f_{\pi(n)}(\rvx^n) - \nabla f(\rvx^0)} \\
		&= \frac{1}{2}\gamma N \normsq{\frac{1}{N}\sum_{n=0}^{N-1}\left(\nabla f_{\pi(n)}(\rvx_n) -\nabla f_{\pi(n)}(\rvx^0)\right)} \\
		&\leq \frac{1}{2}\gamma L_{2,p}^2\sum_{n=0}^{N-1}\normsq{\rvx^n - \rvx^0}_{p}.
	\end{align*}
	Plugging back, we get
	\begin{align*}
		f(\rvx^{N})
		&\leq f(\rvx^0)-\frac{1}{2}\gamma N \normsq{\nabla f(\rvx^0)} + \frac{1}{2}\gamma L_{2,p}^2\sum_{n=0}^{N-1}\normsq{\rvx^n - \rvx^0}_{p}\\
		&\leq f(\rvx^0) -\frac{1}{2}\gamma N \normsq{\nabla f(\rvx^0)} + \frac{1}{2}\gamma L_{2,p}^2 N \left( \Delta \right)^2.
	\end{align*}
	Recovering the subscripts $q$ yields
	\begin{align*}
		f(\rvx_{q+1}) - f(\rvx_q) \leq -\frac{1}{2}\gamma N \normsq{\nabla f(\rvx_q)} + \frac{1}{2}\gamma L_{2,p}^2 N \left( \Delta_q \right)^2.
	\end{align*}
	According to Lemma~\ref{lem:parameter drift}, we can get
	\begin{align*}
		\left( \Delta_q\right)^2&\leq 3\gamma^2 \left(\bar \phi_q\right)^2 + 3 \gamma^2 N^2 \norm{\nabla f(\rvx_q)}_p^2 .
	\end{align*}
	Plugging the upper bound of $\left( \Delta_q\right)^2$, we can get
	\begin{align*}
		f(\rvx_{q+1}) - f(\rvx_q)
		&\leq -\frac{1}{2}\gamma N \normsq{\nabla f(\rvx_q)} + \frac{1}{2}\gamma L_{2,p}^2 N \left( \Delta_q\right)^2\\
		&\leq -\frac{1}{2}\gamma N \normsq{\nabla f(\rvx_q)} + \frac{1}{2}\gamma L_{2,p}^2 N \left( 3\gamma^2 \left(\bar \phi_q\right)^2 + 3\gamma^2N^2\norm{\nabla f(\rvx_q)}_{p}^2 \right)\\
		&\leq f(\rvx_q) -\frac{1}{2}\gamma N  \left( 1-3\gamma^2L_{2,p}^2N^2\right)\normsq{\nabla f(\rvx_q)} + \frac{3}{2} \gamma^3L_{2,p}^2 N \left(\bar \phi_q\right)^2 \\
		&\leq f(\rvx_q) -\frac{255}{512}\gamma N  \normsq{\nabla f(\rvx_q)} + 2 \gamma^3L_{2,p}^2 N \left(\bar \phi_q\right)^2,
	\end{align*}
	where the last inequality is due to $\norm{\rvx}_p \leq \norm{\rvx}$ for $p\geq 2$ and $\textcolor{blue}{\gamma L_{2,p} N \leq \frac{1}{32}}$. Then,
	\begin{align*}
		&f(\rvx_{q+1}) - f(\rvx_q) \leq - \frac{255}{512}\gamma N\normsq{\nabla f(\rvx_q)} + 2\gamma^3 L_{2,p}^2 N \left(\bar \phi_q\right)^2\\
		\implies &\frac{1}{Q}\sum_{q=0}^{Q-1} \left( f(\rvx_{q+1}) - f(\rvx_q) \right) \leq -\frac{255}{512}\gamma N \frac{1}{Q}\sum_{q=0}^{Q-1}\normsq{\nabla f(\rvx_q)} + 2\gamma^3 L_{2,p}^2 N \frac{1}{Q}\sum_{q=0}^{Q-1} \left(\bar \phi_{q}\right)^2\\
		\implies &\frac{1}{\gamma NQ} \left( f(\rvx_{Q}) - f(\rvx_0) \right) \leq -\frac{255}{512} \frac{1}{Q}\sum_{q=0}^{Q-1}\normsq{\nabla f(\rvx_q)} + 2\gamma^2 L_{2,p}^2 \frac{1}{Q}\sum_{q=0}^{Q-1} \left(\bar \phi_{q}\right)^2.
	\end{align*}
	
	Then, we use Assumption~\ref{asm:order error}. Recall that $A_i= 0$ and $B_i=0$ for $i> \nu$ in this theorem. We can write it as
	\begin{align*}
		\left( \bar \phi_{q}\right)^2
		&\leq A_{1} \left(\bar \phi_{q-1}\right)^2 + A_{2} \left(\bar \phi_{q-2}\right)^2 + \cdots + A_{\nu} \left(\bar \phi_{q-\nu}\right)^2\\
		&\quad+ B_{0} \normsq{\nabla f(\rvx_q)} + B_{1} \normsq{\nabla f(\rvx_{q-1})}_{p} + \cdots + B_{\nu} \normsq{\nabla f(\rvx_{q-\nu})}_{p} + D.
	\end{align*}
	Then,
	\begin{align*}
		&\left( \bar \phi_{q}\right)^2
		\leq A_{1} \left(\bar \phi_{q-1}\right)^2 + A_{2} \left(\bar \phi_{q-2}\right)^2 + \cdots + A_{\nu} \left(\bar \phi_{q-\nu}\right)^2\\
		&\qquad\qquad+ B_{0} \normsq{\nabla f(\rvx_q)} + B_{1} \normsq{\nabla f(\rvx_{q-1})} + \cdots + B_{\nu} \normsq{\nabla f(\rvx_{q-\nu})} + D\\
		\implies
		&\sum_{q=\nu}^{Q-1}\left( \bar \phi_{q}\right)^2
		\leq A_{1} \sum_{q=\nu}^{Q-1}\left(\bar \phi_{q-1}\right)^2 + A_{2} \sum_{q=\nu}^{Q-1}\left(\bar \phi_{q-2}\right)^2 + \cdots + A_{\nu} \sum_{q=\nu}^{Q-1}\left(\bar \phi_{q-\nu}\right)^2\\
		&\qquad\qquad+ B_{0} \sum_{q=\nu}^{Q-1}\normsq{\nabla f(\rvx_q)} + B_{1} \sum_{q=\nu}^{Q-1}\normsq{\nabla f(\rvx_{q-1})} + \cdots + B_{\nu} \sum_{q=\nu}^{Q-1}\normsq{\nabla f(\rvx_{q-\nu})} + \sum_{q=\nu}^{Q-1} D\\
		\implies
		&\sum_{q=0}^{Q-1}\left( \bar \phi_{q}\right)^2
		\leq \sum_{i=0}^{\nu-1}\left( \bar \phi_{i}\right)^2 + A_{1} \sum_{q=0}^{Q-1}\left(\bar \phi_q\right)^2 + A_{2} \sum_{q=0}^{Q-1}\left(\bar \phi_{q}\right)^2 + \cdots + A_{\nu} \sum_{q=0}^{Q-1}\left(\bar \phi_{q}\right)^2\\
		&\qquad\qquad+ B_{0} \sum_{q=0}^{Q-1}\normsq{\nabla f(\rvx_q)} + B_{1} \sum_{q=0}^{Q-1}\normsq{\nabla f(\rvx_{q})} + \cdots + B_{\nu} \sum_{q=0}^{Q-1}\normsq{\nabla f(\rvx_{q})} + \sum_{q=0}^{Q-1} D\\
		\implies
		&\left( 1- \sum_{i=1}^{\nu} A_i \right)\frac{1}{Q}\sum_{q=0}^{Q-1}\left( \bar \phi_{q}\right)^2 \leq \frac{1}{Q}\sum_{i=0}^{\nu-1}\left( \bar \phi_{i}\right)^2 + \left( \sum_{i=0}^{\nu} B_i \right)\frac{1}{Q}\sum_{q=0}^{Q-1}\normsq{\nabla f(\rvx_{q})} + D.
	\end{align*}
	Then, we get
	\begin{align*}
		\frac{ f(\rvx_{Q}) - f(\rvx_0) }{\gamma NQ}  &\leq -\frac{255}{512} \frac{1}{Q}\sum_{q=0}^{Q-1}\normsq{\nabla f(\rvx_q)} + 2\gamma^2 L_{2,p}^2 \frac{1}{Q}\sum_{q=0}^{Q-1} \left(\bar \phi_{q}\right)^2\\
		&\leq -\frac{255}{512} \frac{1}{Q}\sum_{q=0}^{Q-1}\normsq{\nabla f(\rvx_q)} + \frac{2\gamma^2 L_{2,p}^2}{\left( 1- \sum_{i=1}^{\nu} A_i \right)} \left( \frac{1}{Q}\sum_{i=0}^{\nu-1}\left( \bar \phi_{i}\right)^2 + \left( \sum_{i=0}^{\nu} B_i \right)\frac{1}{Q}\sum_{q=0}^{Q-1}\normsq{\nabla f(\rvx_{q})} + D \right).
	\end{align*}
	To ensue that $\frac{255}{512} - \frac{2\gamma^2 L_{2,p}^2\sum_{i=0}^{\nu} B_i }{1-\sum_{i=1}^{\nu} A_i} >0$, considering that $\textcolor{blue}{\gamma L_{2,p}N\leq \frac{1}{32}}$, we can use a stricter condition $ \frac{255}{512} - \frac{\sum_{i=0}^{\nu} B_i }{512N^2\left(1-\sum_{i=1}^{\nu} A_i\right)} >0$. Thus, if $\frac{255}{512} - \frac{\sum_{i=0}^{\nu} B_i }{512N^2\left(1-\sum_{i=1}^{\nu} A_i\right)} >0$,
	\begin{align*}
		\frac{1}{Q}\sum_{q=0}^{Q-1}\normsq{\nabla f(\rvx_q)} \leq c_1 \cdot \frac{ f(\rvx_{0}) - f(\rvx_Q) }{\gamma NQ} + c_2 \cdot \gamma^2L_{2,p}^2  \frac{1}{Q}\sum_{i=0}^{\nu-1}\left( \bar \phi_{i}\right)^2 + c_2\cdot\gamma^2L_{2,p}^2 D,
	\end{align*}
	where $c_1$ and $c_2$ are numerical constants such that $c_1 \geq \nicefrac{1}{\left( \frac{255}{512} - \frac{\sum_{i=0}^{\nu} B_i }{512N^2\left(1-\sum_{i=1}^{\nu} A_i\right)}\right)}$ and $c_2 \geq \left(\frac{2}{1-\sum_{i=1}^{\nu} A_i}\right)\cdot c_1$. Let $F_0 = f(\rvx_0) -f_\ast$.
	\begin{align*}
		\min_{q\in \{0,1,\ldots,Q-1\}} \normsq{\nabla f(\rvx_q)} &\leq \frac{1}{Q}\sum_{q=0}^{Q-1}\normsq{\nabla f(\rvx_q)} \\
		&\leq c_1 \cdot \frac{ F_0 }{\gamma NQ} + c_2 \cdot \gamma^2L_{2,p}^2  \frac{1}{Q}\sum_{i=0}^{\nu-1}\left( \bar \phi_{i}\right)^2 + c_2\cdot\gamma^2L_{2,p}^2 D,
	\end{align*}
	where the last inequality is due to $f(\rvx_{0}) - f(\rvx_Q) \leq f(\rvx_0) -f_\ast = F_0 $.
	
	At last, we summarize the constraints on the step sizes $\gamma$ and $\eta$ (they are marked in blue),
	\begin{align*}
		\gamma L N \leq 1,\\
		\gamma L_{2,p} N \leq \frac{1}{32},\\
		\gamma L_p N\leq \frac{1}{32},
	\end{align*}
	where the last one is from Lemma~\ref{lem:parameter drift}. For simplicity, we use a tighter constraint $\gamma \leq \min \left\{ \frac{1}{ L N}, \frac{1}{32L_{2,p} N}, \frac{1}{32L_p N} \right\}$.
\end{proof}

\section{Special Cases in SGD}
\label{apx-sec:SGD:special cases}

In this section, we provide proofs of the examples of SGD in Section~\ref{sec:cases}.

\subsection{Arbitrary Permutation (AP)}

\begin{proof}[Proof of Example~\ref{ex:arbitrary permutation}]
	For any $q$, it follows that
	\begin{align*}
		\left(\bar\phi_{q}\right)^2 &= \max_{n \in [N]} \norm{\sum_{i=0}^{n-1} \left( \nabla f_{\pi_{q}(i)}(\rvx_{q}) - \nabla f(\rvx_{q})  \right) }^2\\
		&\leq \max_{n \in [N]} \left\{ n\sum_{i=0}^{n-1} \normsq{\nabla f_{\pi_{q}(i)}(\rvx_{q}) - \nabla f(\rvx_{q})} \right\} \\
		&\leq \max_{n \in [N]} \left\{ n^2 \varsigma^2 \right\} = N^2\varsigma^2.
	\end{align*}
	
	In this example, for Assumption~\ref{asm:order error}, $p=2$, $A_1 =A_2=\cdots =A_q =0$, $B_0 = B_1=\cdots =B_q = 0$ and $D = N^2\varsigma^2$. Then, we verify that
	\begin{align*}
		\frac{255}{512} - \frac{1}{512N^2}\cdot \frac{\sum_{i=0}^{\nu} B_i }{1-\sum_{i=1}^{\nu} A_i} = \frac{255}{512} >0\,.
	\end{align*}
	Thus, we can set $c_1 = 3$ and $c_2 = 6$ for Theorem~\ref{thm:SGD}. In addition, for Theorem~\ref{thm:SGD}, $\nu=0$. These lead to the upper bound,
	\begin{align*}
		\min_{q\in \{0,1\ldots, Q-1\}}\normsq{ \nabla f(\rvx_q) } 
		&= \gO\left(  \frac{ F_0 }{\gamma NQ} + \gamma^2 L^2N^2\varsigma^2 \right),
	\end{align*}
	where $F_0 = f(\rvx_0) - f_\ast$ and $L = L_{2,p} = L_{p}$ when $p=2$.
	
	Next, we summarize the constraints on the step size:
	\begin{align*}
	\gamma \leq \min \left\{ \frac{1}{L N}, \frac{1}{32L_{2,p} N}, \frac{1}{32L_p N} \right\} = \frac{1}{32 L N}\,.
	\end{align*}
	It is from Theorem~\ref{thm:SGD}. After we use the effective step size $\tilde \gamma \coloneqq \gamma N$, the constraint becomes
	\begin{align*}
		\tilde \gamma \leq \min \left\{ \frac{1}{L }, \frac{1}{32L_{2,p} }, \frac{1}{32L_p } \right\} = \frac{1}{32 L }\,,
	\end{align*}
	and the upper bound becomes
	\begin{align*}
		\min_{q\in \{0,1\ldots, Q-1\}}\normsq{ \nabla f(\rvx_q) } 
		&= \gO\left(  \frac{ F_0 }{\tilde\gamma Q} +\tilde\gamma^2 L^2\varsigma^2\right).
	\end{align*}
	Applying Lemma~\ref{lem:step-size}, we get
	\begin{align*}
		\min_{q\in \{0,1\ldots, Q-1\}}\normsq{ \nabla f(\rvx_q) } =
		\gO\left( \frac{L F_0}{Q} + \left(\frac{LF_0 N \varsigma }{ NQ}\right)^{\frac{2}{3}}  \right).
	\end{align*}
\end{proof}

\subsection{Random Reshuffling (RR)}

\begin{proof}[Proof of Example~\ref{ex:random reshuffling}]
	Since the permutations $\{\pi_q\}$ are independent across different epochs, for any $q$, when conditional on $\rvx_q$, we get that, with probability at least $1-\delta$,
	\begin{align*}
		\left(\bar\phi_{q}\right)^2 &= \max_{n \in [N]} \norm{\sum_{i=0}^{n-1} \left( \nabla f_{\pi_{q}(i)}(\rvx_{q}) - \nabla f(\rvx_{q})  \right) }^2 \leq 4N \varsigma^2 \log^2 \left( \frac{8}{\delta}\right),
	\end{align*}
	where the last inequality is due to \citet{yu2023high}'s Proposition~2.3.
	
	In this case, for Assumption~\ref{asm:order error}, $p=2$, $A_1=A_2 = \cdots = A_q =0$, $B_0 = B_1 = \cdots = B_q =0$ and $D = 4N \varsigma^2 \log^2 \left( \frac{8}{\delta}\right)$. Then, we verify that
	\begin{align*}
		\frac{255}{512} - \frac{1}{512N^2}\cdot \frac{\sum_{i=0}^{\nu} B_i }{1-\sum_{i=1}^{\nu} A_i} = \frac{255}{512} >0\,.
	\end{align*}
	Thus, we can set $c_1 = 3$ and $c_2 = 6$ for Theorem~\ref{thm:SGD}. In addition, for Theorem~\ref{thm:SGD}, $\nu=0$. These lead to the upper bound,
	\begin{align*}
		\min_{q\in \{0,1\ldots, Q-1\}}\normsq{ \nabla f(\rvx_q) } 
		&= \tilde \gO\left(  \frac{ F_0 }{\gamma NQ} + \gamma^2 L^2N\varsigma^2 \right),
	\end{align*}
	where $F_0 = f(\rvx_0) - f_\ast$ and $L = L_{2,p} = L_{p}$ when $p=2$.
	
	Next, we summarize the constraints on the step size:
	\begin{align*}
		\gamma \leq \min \left\{ \frac{1}{L N}, \frac{1}{32L_{2,p} N}, \frac{1}{32L_p N} \right\} = \frac{1}{32 L N}.
	\end{align*}
	It is from Theorem~\ref{thm:SGD}. After we use the effective step size $\tilde \gamma \coloneqq \gamma N$, the constrain becomes
	\begin{align*}
		\tilde \gamma \leq \min \left\{ \frac{1}{L }, \frac{1}{32L_{2,p} }, \frac{1}{32L_p } \right\} = \frac{1}{32 L },
	\end{align*}
	and the upper bound becomes
	\begin{align*}
		\min_{q\in \{0,1\ldots, Q-1\}}\normsq{ \nabla f(\rvx_q) } 
		&= \tilde\gO\left(  \frac{ F_0 }{\tilde\gamma Q} +\tilde\gamma^2 L^2\frac{1}{N}\varsigma^2\right).
	\end{align*}
	Applying Lemma~\ref{lem:step-size}, we get
	\begin{align*}
		\min_{q\in \{0,1\ldots, Q-1\}}\normsq{ \nabla f(\rvx_q) } =
		\tilde\gO\left( \frac{L F_0}{Q} + \left(\frac{LF_0 \sqrt{N} \varsigma }{ NQ}\right)^{\frac{2}{3}}  \right).
	\end{align*}	
\end{proof}

\subsection{One Permutation (OP)}

\begin{proof}[Proof of Example~\ref{ex:one permutation}]
	For any $q\geq 1$ and $n \in [N]$, we have
	\begin{align*}
		\phi_{q}^n &= \norm{ \sum_{i=0}^{n-1} \left( \nabla f_{\pi_{q}(i)}(\rvx_{q}) - \nabla f(\rvx_{q})  \right) }  \\
		&= \norm{ \sum_{i=0}^{n-1} \left( \nabla f_{\pi_{q}(i)}(\rvx_{q}) - \nabla f(\rvx_{q})  \right) - \left( \nabla f_{\pi_{q}(i)}(\rvx_{0}) - \nabla f(\rvx_{0})  \right) + \left( \nabla f_{\pi_{q}(i)}(\rvx_{0}) - \nabla f(\rvx_{0})  \right) }\\
		&\leq \norm{ \sum_{i=0}^{n-1} \left( \nabla f_{\pi_{q}(i)}(\rvx_{q}) - \nabla f_{\pi_{q}(i)}(\rvx_{0})  \right)  } + \norm{ \sum_{i=0}^{n-1} \left( \nabla f(\rvx_{q}) - \nabla f(\rvx_{0}) \right) } + \norm{\sum_{i=0}^{n-1} \left( \nabla f_{\pi_{q}(i)}(\rvx_{0}) - \nabla f(\rvx_{0})  \right) }\\
		&\leq \sum_{i=0}^{n-1}\norm{ \nabla f_{\pi_{q}(i)}(\rvx_{q}) - \nabla f_{\pi_{q}(i)}(\rvx_{0}) } + \sum_{i=0}^{n-1}\norm{ \nabla f(\rvx_{q}) - \nabla f(\rvx_{0}) } + \norm{\sum_{i=0}^{n-1} \left( \nabla f_{\pi_{q}(i)}(\rvx_{0}) - \nabla f(\rvx_{0})  \right) }\\
		&\leq 2Ln\norm{\rvx_{q} - \rvx_{0}} + \norm{\sum_{i=0}^{n-1} \left( \nabla f_{\pi_{q}(i)}(\rvx_{0}) - \nabla f(\rvx_{0})  \right) }\\
		&\leq 2LN\theta + \bar \phi_0\,,
	\end{align*}
	where we use the fact that the permutations are exactly the same, $\pi_q = \pi_0$ for $q\geq 1$ in OP. Since the preceding inequality holds for all $n\in [N]$, we have
	\begin{align*}
		\bar\phi_{q} \leq 2LN \theta + \bar\phi_0
		\implies & \left( \bar\phi_{q}\right)^2 \leq 2 \cdot \left(2LN \theta\right)^2 + 2\cdot \left( \bar\phi_0 \right)^2 = 8 L^2 N^2 \theta^2 + 2 \left( \bar\phi_0 \right)^2
	\end{align*}
	
	In this case, for Assumption~\ref{asm:order error}, $p=2$, $A_1=A_2 = \cdots = A_q =0$, $B_0 = B_1 = \cdots = B_q =0$ and $D = 8 L^2 N^2 \theta^2 + 2 \left( \bar\phi_0 \right)^2$. Then, we verify that
	\begin{align*}
		\frac{255}{512} - \frac{1}{512N^2}\cdot \frac{\sum_{i=0}^{\nu} B_i }{1-\sum_{i=1}^{\nu} A_i} = \frac{255}{512} >0\,.
	\end{align*}
	Thus, we can set $c_1 = 3$ and $c_2 = 6$ for Theorem~\ref{thm:SGD}. In addition, for Theorem~\ref{thm:SGD}, $\nu=0$. These lead to the upper bound,
	\begin{align*}
		\min_{q\in \{0,1\ldots, Q-1\}}\normsq{ \nabla f(\rvx_q) } = \gO \left(  \frac{F_0}{\gamma NQ} + \gamma^2 L^2 \left( \bar\phi_0 \right)^2 + \gamma^2 L^4 N^2\theta^2 \right),
	\end{align*}
	where $F_0 = f(\rvx_0) - f_\ast$ and $L = L_{2,p} = L_{p}$ when $p=2$.
	
	Next, we summarize the constraints on the step size:
	\begin{align*}
		\gamma \leq \min \left\{ \frac{1}{L N}, \frac{1}{32L_{2,p} N}, \frac{1}{32L_p N} \right\} = \frac{1}{32 L N}\,.
	\end{align*}
	It is from Theorem~\ref{thm:SGD}. After we use the effective step size $\tilde \gamma \coloneqq \gamma N$, the constraint becomes
	\begin{align*}
		\tilde \gamma \leq \min \left\{ \frac{1}{L }, \frac{1}{32L_{2,p} }, \frac{1}{32L_p } \right\} = \frac{1}{32 L },
	\end{align*}
	and the upper bound becomes
	\begin{align*}
		\min_{q\in \{0,1\ldots, Q-1\}}\normsq{ \nabla f(\rvx_q) } 
		&= \tilde\gO\left(  \frac{ F_0 }{\tilde\gamma Q} +\tilde\gamma^2 L^2\frac{1}{N^2}\left( \bar\phi_0\right)^2 + \tilde\gamma^2 L^4 \theta^2\right).
	\end{align*}
	Applying Lemma~\ref{lem:step-size}, we get
	\begin{align*}
		\min_{q\in \{0,1\ldots, Q-1\}}\normsq{ \nabla f(\rvx_q) } =
		\gO\left( \frac{L F_0}{Q} + \left(\frac{LF_0 \bar\phi_0+L^2F_0 N\theta }{ NQ}\right)^{\frac{2}{3}}  \right).
	\end{align*}
	Furthermore, if $\theta \lesssim \frac{\bar\phi_0}{LN}$, then
	\begin{align*}
		\min_{q\in \{0,1\ldots, Q-1\}}\normsq{ \nabla f(\rvx_q) } =
		\gO\left( \frac{L F_0}{Q} + \left(\frac{LF_0 \bar\phi_0 }{ NQ}\right)^{\frac{2}{3}}  \right).
	\end{align*}
	
	Next, let us deal with $\bar\phi_0$, depending on the initial permutation.
	\begin{itemize}
		\item If the initial permutation $\pi_0$ is generated arbitrarily, we get
		\begin{align*}
			\left(\bar\phi_{0}\right)^2 = \max_{n \in [N]} \norm{\sum_{i=0}^{n-1} \left( \nabla f_{\pi_{0}(i)}(\rvx_{0}) - \nabla f(\rvx_{0})  \right) }^2 \leq \max_{n \in [N]} \left( n^2 \varsigma^2 \right) = N^2\varsigma^2.
		\end{align*}
		Then,
		\begin{align*}
			\min_{q\in \{0,1\ldots, Q-1\}}\normsq{ \nabla f(\rvx_q) } =
			\gO\left( \frac{L F_0}{Q} + \left(\frac{LF_0 N\varsigma }{ NQ}\right)^{\frac{2}{3}}  \right).
		\end{align*}
		\item Shuffle Once (SO). If the initial permutation $\pi_0$ is generated randomly, we get that, with probability at least $1-\delta$,
		\begin{align*}
			\left(\bar\phi_{0}\right)^2 = \max_{n \in [N]} \norm{\sum_{i=0}^{n-1} \left( \nabla f_{\pi_{0}(i)}(\rvx_{0}) - \nabla f(\rvx_{0})  \right) }^2 \leq 4N \varsigma^2 \log^2 \left( \frac{8}{\delta}\right).
		\end{align*}
		Then,
		\begin{align*}
			\min_{q\in \{0,1\ldots, Q-1\}}\normsq{ \nabla f(\rvx_q) } =
			\tilde\gO\left( \frac{L F_0}{Q} + \left(\frac{LF_0 \sqrt{N}\varsigma }{ NQ}\right)^{\frac{2}{3}}  \right).
		\end{align*}
		It holds with probability at least $1-\delta$, because \citet{yu2023high}'s Proposition~2.3 is only used for the initial epoch.
		\item Nice Permutation (NP). If the initial permutation $\pi_0$ is a nice permutation such that $\bar\phi_0 = \tilde \gO\left( \varsigma \right)$,
		\begin{align*}
			\min_{q\in \{0,1\ldots, Q-1\}}\normsq{ \nabla f(\rvx_q) } =
			\tilde\gO\left( \frac{L F_0}{Q} + \left(\frac{LF_0 \varsigma }{ NQ}\right)^{\frac{2}{3}}  \right).
		\end{align*}
		In fact, we can generate such a nice permutation by GraBs \citep[Section 6. Ablation Study: are good permutations fixed?]{lu2022grab}.
	\end{itemize}
\end{proof}

\subsection{GraB-proto}

GraB-proto: Use \texttt{BasicBR} (Algorithm~\ref{alg:BasicBR}) as the \texttt{Permute} function in Algorithm~\ref{alg:SGD}, with the inputs of $\pi_q$, $\{\nabla f_{\pi_q(n)}(\rvx_q)\}_{n=0}^{N-1}$ and $\nabla f(\rvx_q)$, for each epoch $q$.

Thus, the key idea of our proof is as follows:
\begin{align*}
	\bar \phi_{q+1} \rightarrow \max_{n \in [N]} \norm{\sum_{i=0}^{n-1} \left( \nabla f_{\pi_{q+1}(i)}(\rvx_{q}) - \nabla f(\rvx_{q})  \right) }_{\infty}  \overset{\text{Lemma~\ref{lem:basic-balancing-reordering}}}{\rightarrow}  \bar \phi_{q}\,.
\end{align*}

\begin{proof}[Proof of Example~\ref{ex:GraB-proto}]
	We need to find the relation between $\bar \phi_{q}$ and $\bar \phi_{q-1}$ for $q\geq 1$. For any $n \in [N]$,
	\begin{align*}
		\phi_{q+1}^n &= \norm{ \sum_{i=0}^{n-1} \left( \nabla f_{\pi_{q+1}(i)}(\rvx_{q+1}) - \nabla f(\rvx_{q+1})  \right) }_{\infty}  \\
		&= \norm{ \sum_{i=0}^{n-1} \left( \nabla f_{\pi_{q+1}(i)}(\rvx_{q+1}) - \nabla f(\rvx_{q+1})  \right) - \left( \nabla f_{\pi_{q+1}(i)}(\rvx_{q}) - \nabla f(\rvx_{q})  \right) + \left( \nabla f_{\pi_{q+1}(i)}(\rvx_{q}) - \nabla f(\rvx_{q})  \right) }_{\infty}\\
		&\leq \norm{ \sum_{i=0}^{n-1} \left( \nabla f_{\pi_{q+1}(i)}(\rvx_{q+1}) - \nabla f_{\pi_{q+1}(i)}(\rvx_{q})  \right)  }_{\infty} + \norm{ \sum_{i=0}^{n-1} \left( \nabla f(\rvx_{q+1}) - \nabla f(\rvx_{q}) \right) }_{\infty} + \norm{\sum_{i=0}^{n-1} \left( \nabla f_{\pi_{q+1}(i)}(\rvx_{q}) - \nabla f(\rvx_{q})  \right) }_{\infty}\\
		&\leq \sum_{i=0}^{n-1}\norm{ \nabla f_{\pi_{q+1}(i)}(\rvx_{q+1}) - \nabla f_{\pi_{q+1}(i)}(\rvx_{q}) }_{\infty} + \sum_{i=0}^{n-1}\norm{ \nabla f(\rvx_{q+1}) - \nabla f(\rvx_{q}) }_{\infty} + \norm{\sum_{i=0}^{n-1} \left( \nabla f_{\pi_{q+1}(i)}(\rvx_{q}) - \nabla f(\rvx_{q})  \right) }_{\infty}\\
		&\leq 2L_{\infty}n\norm{\rvx_{q+1} - \rvx_{q}}_{\infty} + \norm{\sum_{i=0}^{n-1} \left( \nabla f_{\pi_{q+1}(i)}(\rvx_{q}) - \nabla f(\rvx_{q})  \right) }_{\infty}.
	\end{align*}
	Since the above inequality holds for all $n \in [N]$, we have
	\begin{align*}
		\bar \phi_{q+1} \leq 2L_{\infty}N\norm{\rvx_{q+1} - \rvx_{q}}_{\infty} + \max_{n\in[N]} \norm{\sum_{i=0}^{n-1} \left( \nabla f_{\pi_{q+1}(i)}(\rvx_{q}) - \nabla f(\rvx_{q})  \right) }_{\infty}
	\end{align*}
	
	Note that $\nabla f_{\pi_{q}(i)}(\rvx_q) - \nabla f(\rvx_q)$ and $\nabla f_{\pi_{q+1}(i)}(\rvx_q) - \nabla f(\rvx_q)$ correspond to $\rvz_{\pi(i)}$ and $\rvz_{\pi' (i)}$ in Lemma~\ref{lem:basic-balancing-reordering}, respectively. In GraB-proto, since
	\begin{align*}
		&\norm{\nabla f_i (\rvx_q) - f(\rvx_q)} \leq \varsigma, \quad\forall i\in \{0,1,\ldots, N-1\},\\
		&\norm{ \sum_{i=0}^{N-1} \left( \nabla f_i (\rvx_q) - f(\rvx_q)\right) }_{\infty} = 0,
	\end{align*}
	we apply Lemma~\ref{lem:basic-balancing-reordering} with $a = \varsigma$ and $b = 0$, and get
	\begin{align*}
		\max_{n \in [N]} \norm{\sum_{i=0}^{N-1} \left( \nabla f_{\pi_{q+1}(i)}(\rvx_{q}) - \nabla f(\rvx_{q})  \right) }_{\infty} &\leq \frac{1}{2} \max_{n \in [N]} \norm{\sum_{i=0}^{n-1} \left( \nabla f_{\pi_{q}(i)}(\rvx_{q}) - \nabla f(\rvx_{q})  \right) }_{\infty} + \frac{1}{2} C \varsigma\\
		&= \frac{1}{2} \bar \phi_q + \frac{1}{2}C\varsigma\,.
	\end{align*}
	
	Using Lemma~\ref{lem:parameter drift} that $\Delta_q\leq \frac{32}{31}\gamma \bar \phi_q + \frac{32}{31}\gamma N \norm{\nabla f(\rvx_q)}_{\infty}$, we get
	\begin{align*}
		\bar \phi_{q+1} &\leq 2L_{\infty}N\norm{\rvx_{q+1} - \rvx_{q}}_{\infty} + \max_{n\in[N]} \norm{\sum_{i=0}^{n-1} \left( \nabla f_{\pi_{q+1}(i)}(\rvx_{q}) - \nabla f(\rvx_{q})  \right) }_{\infty}\\
		&\leq 2L_{\infty}N \left(\frac{32}{31}\gamma \bar \phi_q + \frac{32}{31}\gamma N \norm{\nabla f(\rvx_q)}_{\infty} \right) + \left( \frac{1}{2} \bar \phi_q + \frac{1}{2}C\varsigma\right)\\
		&\leq \frac{35}{62} \bar\phi_q + \frac{2}{31} N \norm{\nabla f(\rvx_q)}_{\infty} + \frac{1}{2}C\varsigma\,.
	\end{align*}
	where the last inequality is due to $\textcolor{blue}{\gamma L_{\infty} N \leq \frac{1}{32}}$. Next, using $\norm{\rvx}_{p} \leq \norm{\rvx}_{2}$ for $p\geq 2$, we get
	\begin{align*}
		\left(\bar \phi_{q+1}\right)^2
		&\leq \left( \frac{35}{62}\bar \phi_q + \frac{2}{31}N \norm{\nabla f(\rvx_q)}_{} + \frac{1}{2}C\varsigma \right)^2 \\
		&\leq 2\cdot \left( \frac{35}{62} \bar \phi_q \right)^2 + 4 \cdot \left( \frac{2}{31}N \norm{\nabla f(\rvx_q)}_{} \right)^2+ 4 \cdot \left( \frac{1}{2}C\varsigma \right)^2\\
		&\leq \frac{3}{4}\left( \bar \phi_q \right)^2 + \frac{1}{50} N^2 \normsq{\nabla f(\rvx_q)}_{} + C^2 \varsigma^2\,.
	\end{align*}
	So the relation between $\bar \phi_q$ and $\bar \phi_{q-1}$ is 
	\begin{align*}
		\left(\bar \phi_{q}\right)^2 \leq \frac{3}{4}\left( \bar \phi_{q-1} \right)^2 + \frac{1}{50} N^2 \normsq{\nabla f(\rvx_{q-1})}_{} + C^2 \varsigma^2\,.
	\end{align*} for $q\geq 1$. Besides, we need to get the bound of $\left(\bar\phi_{0}\right)^2$:
	\begin{align*}
		\left(\bar\phi_{0}\right)^2 = \max_{n \in [N]} \norm{\sum_{i=0}^{n-1} \left( \nabla f_{\pi_{0}(i)}(\rvx_{0}) - \nabla f(\rvx_{0})  \right) }^2 \leq N^2 \varsigma^2\,.
	\end{align*}
	
	In this case, for Assumption~\ref{asm:order error}, $p=\infty$, $A_1=\frac{3}{4}$, $A_2 =\cdots= A_q = 0$, $B_0 = 0$,$B_1 =\frac{1}{50}N^2$, $B_2=\cdots=B_q=0$ and $D = C^2 \varsigma^2$. Then, we verify that
	\begin{align*}
		\frac{255}{512} - \frac{1}{512N^2}\cdot \frac{\sum_{i=0}^{\nu} B_i }{1-\sum_{i=1}^{\nu} A_i} \geq \frac{254}{512} >0\,.
	\end{align*}
	Thus, we can set $c_1 = 3$ and $c_2 = 24$ for Theorem~\ref{thm:SGD}. In addition, for Theorem~\ref{thm:SGD}, $\nu=1$ and $\left(\bar\phi_{0}\right)^2 \leq N^2 \varsigma^2$. These lead to the upper bound,
	\begin{align*}
		\min_{q\in \{0,1\ldots, Q-1\}}\normsq{ \nabla f(\rvx_q) } 
		&= \gO\left(  \frac{ F_0 }{\gamma NQ} + \gamma^2 L_{2,\infty}^2 N^2\frac{1}{Q} \varsigma^2 + \gamma^2 L_{2,\infty}^2 C^2\varsigma^2 \right).
	\end{align*}
	where $F_0 = f(\rvx_0) - f_\ast$. Since Lemma~\ref{lem:basic-balancing-reordering} is used for each epoch (that is, for $Q$ times), so by the union bound, the preceding bound holds with probability at least $1-Q\delta$.
	
	Next, we summarize the constraints on the step size:
	\begin{align*}
		\gamma &\leq \min \left\{ \frac{1}{LN}, \frac{1}{32L_{2,p} N}, \frac{1}{32L_{p} N} \right\}, \\
		\gamma  &\leq \frac{1}{32 L_{\infty} N}.
	\end{align*}
	The first one is from Theorem~\ref{thm:FL} and the other is from the derivation of the relation. For simplicity, we can use a tighter constraint
	\begin{align*}
		\gamma &\leq \min \left\{ \frac{1}{L N}, \frac{1}{32L_{2,\infty} N}, \frac{1}{32L_\infty N} \right\}.
	\end{align*}
	After we use the effective step size $\tilde \gamma \coloneqq \gamma N $, the constraint will be
	\begin{align*}
		\tilde \gamma \leq \min \left\{ \frac{1}{L}, \frac{1}{32L_{2,\infty}}, \frac{1}{32L_\infty} \right\},
	\end{align*}
	and the upper bound will be
	\begin{align*}
		\min_{q\in \{0,1\ldots, Q-1\}}\normsq{ \nabla f(\rvx_q) } 
		&= \gO\left(  \frac{ F_0 }{\tilde\gamma Q} + \tilde\gamma^2 L_{2,\infty}^2 \frac{1}{Q} \varsigma^2 + \tilde\gamma^2 L_{2,\infty}^2 \frac{1}{N^2}C^2\varsigma^2 \right).
	\end{align*}
	Applying Lemma~\ref{lem:step-size}, we get
	\begin{align*}
		\min_{q\in \{0,1\ldots, Q-1\}}\normsq{ \nabla f(\rvx_q) } &=
		\gO\left( \frac{\left( L +L_{2,\infty}+ L_{\infty}\right)  F_0}{Q}  + \frac{ (L_{2,\infty}F_0 \varsigma)^{\frac{2}{3}} }{Q}  + \left(\frac{L_{2,\infty}F_0 C \varsigma }{ NQ}\right)^{\frac{2}{3}}  \right).
	\end{align*}
\end{proof}

\subsection{PairGraB-proto}
\label{apx-subsec:PairGraB-proto}

PairGraB-proto. Use \texttt{PairBR} (Algorithm~\ref{alg:PairBR}) as the \texttt{Permute} function in Algorithm~\ref{alg:SGD}, with the inputs of $\pi_q$, $\{\nabla f_{\pi_q(n)}(\rvx_q)\}_{n=0}^{N-1}$ and $\nabla f(\rvx_q)$, for each epoch $q$.

Thus, the key idea of our proof is as follows:
\begin{align*}
	\bar \phi_{q+1} \rightarrow \max_{n \in [N]} \norm{\sum_{i=0}^{n-1} \left( \nabla f_{\pi_{q+1}(i)}(\rvx_{q}) - \nabla f(\rvx_{q})  \right) }_{\infty}  \overset{\text{Lemma~\ref{lem:pair-balancing-reordering}}}{\rightarrow}  \bar \phi_{q}\,.
\end{align*}

\begin{example}[PairGraB-proto]
	\label{ex:PairGraB-proto}
	Let $\{f_n\}$ be $L_{\infty}$-smooth and Assumption~\ref{asm:local deviation} hold. Assume that $N\mod 2=0$. Then, if $\gamma \leq \frac{1}{32L_{\infty}N}$, Assumption~\ref{asm:order error} holds with probability at least $1-\delta$:
	\begin{align*}
		\left(\bar \phi_{q}\right)^2 \leq \frac{3}{4}\left( \bar \phi_{q-1} \right)^2 + \frac{1}{50} N^2 \normsq{\nabla f(\rvx_{q-1})}_{} + \textcolor{blue}{4C^2 \varsigma^2}\,,
	\end{align*}
	where $C = \gO \left(\log\left(\frac{d N}{\delta}\right) \right) = \tilde\gO\left( 1\right)$. Applying Theorem~\ref{thm:SGD}, we get that, with probability at least $1-Q\delta$,
	\begin{align*}
		\min_{q\in \{0,1\ldots, Q-1\}} \normsq{ \nabla f(\rvx_q) } 
		&= \gO\left(  \frac{F_0}{\gamma NQ} + \gamma^2 \frac{1}{Q} L_{2,\infty}^2 N^2 \varsigma^2 + \gamma^2 L_{2,\infty}^2C^2\varsigma^2 \right).
	\end{align*}
	After the step size is tuned, the upper bound becomes $\gO\left( \frac{\tilde LF_0 + \left(L_{2,\infty}F_0\varsigma\right)^{\frac{2}{3}} }{Q} + \left(\frac{L_{2,\infty}F_0 C \varsigma}{NQ}\right)^{\frac{2}{3}} \right)$, where $\tilde L = L+L_{2,\infty}+L_{\infty}$.
\end{example}

\begin{proof}[Proof of Example~\ref{ex:PairGraB-proto}]
	The proof of Example~\ref{ex:PairGraB-proto} is almost identical to that of Example~\ref{ex:GraB-proto}, except that Lemma~\ref{lem:basic-balancing-reordering} is replaced by Lemma~\ref{lem:pair-balancing-reordering}. This difference only causes that some numerical constants are changed accordingly.
\end{proof}

\subsection{GraB}
\label{apx-subsec:GraB}

GraB. Use \texttt{BasicBR} (Algorithm~\ref{alg:BasicBR}) as the \texttt{Permute} function in Algorithm~\ref{alg:SGD}, with the inputs of $\pi_q$, $\{ \nabla f_{\pi_q(n)} (\rvx_q^n) \}_{n=0}^{N-1}$ and $\frac{1}{N}\sum_{n=0}^{N-1}\nabla f_{\pi_{q-1}(n)} (\rvx_{q-1}^n)$, for each epoch $q$.

Thus, the key idea of our proof is as follows:
\begin{align*}
	\bar \phi_{q+1} \rightarrow \max_{n \in [N]} \norm{\sum_{i=0}^{n-1} \left( \nabla f_{\pi_{q+1}(i)}\left(\rvx_{q}^{ \pi_q^{-1} (\pi_{q+1}(i)) }\right) - \frac{1}{N}\sum_{l=0}^{N-1}\nabla f_{\pi_{q-1}(l)}(\rvx_{q-1}^l)  \right) }_{\infty}   \\
	\overset{\text{Lemma~\ref{lem:basic-balancing-reordering}}}{\rightarrow} \max_{n\in[N]} \norm{\sum_{i=0}^{n-1} \left( \nabla f_{\pi_{q}(i)}\left(\rvx_{q}^{i}\right) - \frac{1}{N}\sum_{l=0}^{N-1}\nabla f_{\pi_{q-1}(l)}(\rvx_{q-1}^l)  \right) }_{\infty} \rightarrow
	 \bar \phi_{q}\,.
\end{align*}

\begin{proof}[Proof of Example~\ref{ex:GraB}]
	We need to find the relation between $\bar \phi_{q+1}$ and $\bar \phi_q$.
	\begin{align}
		\phi_{q+1}^n &= \norm{ \sum_{i=0}^{n-1} \left( \nabla f_{\pi_{q+1}(i)}(\rvx_{q+1}) - \nabla f(\rvx_{q+1})  \right) }_{\infty} \nonumber \\
		&= \norm{ \sum_{i=0}^{n-1} \left( \left( \nabla f_{\pi_{q+1}(i)}(\rvx_{q+1}) - \nabla f(\rvx_{q+1})  \right) \pm \left( \nabla f_{\pi_{q+1}(i)}\left(\rvx_{q}^{ \pi_q^{-1} (\pi_{q+1}(i)) }\right) - \frac{1}{N}\sum_{l=0}^{N-1}\nabla f_{\pi_{q-1}(l)}(\rvx_{q-1}^l)  \right)\right) }_{\infty}\nonumber\\
		&\leq \norm{ \sum_{i=0}^{n-1} \left( \nabla f_{\pi_{q+1}(i)}(\rvx_{q+1}) - \nabla f_{\pi_{q+1}(i)}\left(\rvx_{q}^{ \pi_q^{-1} (\pi_{q+1}(i)) }\right)  \right)  }_{\infty} \nonumber\\
		&\quad+ \norm{ \sum_{i=0}^{n-1} \left( \frac{1}{N}\sum_{l=0}^{N-1}\nabla f_{\pi_{q+1}(l)}(\rvx_{q+1}) - \frac{1}{N}\sum_{l=0}^{N-1}\nabla f_{\pi_{q-1}(l)}(\rvx_{q-1}^l) \right) }_{\infty} \nonumber\\
		&\quad+ \norm{\sum_{i=0}^{n-1} \left( \nabla f_{\pi_{q+1}(i)}\left(\rvx_{q}^{ \pi_q^{-1} (\pi_{q+1}(i)) }\right) - \frac{1}{N}\sum_{l=0}^{N-1}\nabla f_{\pi_{q-1}(l)}(\rvx_{q-1}^l)  \right)   }_{\infty}.\label{eq:pf-lem-GraB:relation}
	\end{align}
	Then,
	\begin{align*}
		\Term{1}{\eqref{eq:pf-lem-GraB:relation}}&= \norm{ \sum_{i=0}^{n-1} \left( \nabla f_{\pi_{q+1}(i)}(\rvx_{q+1}) - \nabla f_{\pi_{q+1}(i)}\left(\rvx_{q}^{ \pi_q^{-1} (\pi_{q+1}(i)) }\right)  \right)  }_{\infty}\\
		&\leq \sum_{i=0}^{n-1}\norm{  \nabla f_{\pi_{q+1}(i)}(\rvx_{q+1}) - \nabla f_{\pi_{q+1}(i)}\left(\rvx_{q}^{ \pi_q^{-1} (\pi_{q+1}(i)) }\right) }_{\infty}\\
		&\leq L_{\infty}\sum_{i=0}^{n-1}\norm{  \rvx_{q+1} - \rvx_{q}^{ \pi_q^{-1} (\pi_{q+1}(i)) } }_{\infty}\\
		&\leq L_{\infty}\sum_{i=0}^{n-1}\left( \norm{\rvx_{q+1}- \rvx_q}_{\infty} + \norm{\rvx_q - \rvx_{q}^{ \pi_q^{-1} (\pi_{q+1}(i)) }}_{\infty}\right)\\
		&\leq 2L_{\infty} N \Delta_q\,,
	\end{align*}
	\begin{align*}
		\Term{2}{\eqref{eq:pf-lem-GraB:relation}}&= \norm{ \sum_{i=0}^{n-1} \left( \frac{1}{N}\sum_{l=0}^{N-1}\nabla f_{\pi_{q+1}(l)}(\rvx_{q+1}) - \frac{1}{N}\sum_{l=0}^{N-1}\nabla f_{\pi_{q-1}(l)}\left(\rvx_{q-1}^l\right) \right) }_{\infty}\\
		&= \norm{ \sum_{i=0}^{n-1} \left( \frac{1}{N}\sum_{l=0}^{N-1}\nabla f_{l}\left(\rvx_{q+1}\right) - \frac{1}{N}\sum_{l=0}^{N-1}\nabla f_{l}\left(\rvx_{q-1}^{\pi_{q-1}^{-1}(l)}\right) \right) }_{\infty}\\
		&\leq \sum_{i=0}^{n-1}  \frac{1}{N}\sum_{l=0}^{N-1}\norm{ \nabla f_{l}(\rvx_{q+1}) -\nabla f_{l}\left(\rvx_{q-1}^{\pi_{q-1}^{-1}(l)}\right) }_{\infty}\\
		&\leq L_{\infty}\sum_{i=0}^{n-1}  \frac{1}{N}\sum_{l=0}^{N-1}\norm{\rvx_{q+1} -\rvx_{q-1}^{\pi_{q-1}^{-1}(l)} }_{\infty}\\
		&\leq L_{\infty}\sum_{i=0}^{n-1}  \frac{1}{N}\sum_{l=0}^{N-1}\left(\norm{\rvx_{q+1} -\rvx_{q} }_{\infty} + \norm{\rvx_{q} -\rvx_{q-1} }_{\infty} + \norm{\rvx_{q-1} -\rvx_{q-1}^{\pi_{q-1}^{-1}(l)} }_{\infty} \right)\\
		&\leq L_{\infty} N\Delta_q + 2L_{\infty} N\Delta_{q-1}\,.
	\end{align*}
	Since the preceding inequalities hold for all $n\in [N]$, we have
	\begin{align}
		\bar\phi_{q+1} \leq 3L_{\infty} N \Delta_q + 2L_{\infty} N \Delta_{q-1} + \max_{n \in [N]}\norm{\sum_{i=0}^{n-1} \left( \nabla f_{\pi_{q+1}(i)}\left(\rvx_{q}^{ \pi_q^{-1} (\pi_{q+1}(i)) }\right) - \frac{1}{N}\sum_{l=0}^{N-1}\nabla f_{\pi_{q-1}(l)}(\rvx_{q-1}^l)  \right)   }_{\infty}\label{eq:pf-lem-GraB:relation^2}
	\end{align}
	
	Note that $\nabla f_{\pi_{q}(i)}\left(\rvx_{q}^{i}\right) - \frac{1}{N}\sum_{l=0}^{N-1}\nabla f_{\pi_{q-1}(l)}(\rvx_{q-1}^l)$ and $\nabla f_{\pi_{q+1}(i)}\left(\rvx_{q}^{ \pi_q^{-1} (\pi_{q+1}(i)) }\right) - \frac{1}{N}\sum_{l=0}^{N-1}\nabla f_{\pi_{q-1}(l)}(\rvx_{q-1}^l) $ correspond to $\rvz_{\pi(i)}$ and $\rvz_{\pi' (i)}$ in Lemma~\ref{lem:basic-balancing-reordering}, respectively. We next get the upper bounds of 
	\begin{align*}
		\norm{\rvz_{\pi(i)}}_2, \norm{\sum_{i=0}^{N-1} \rvz_{\pi(i)}}_{\infty} \text{ and } \max_{n\in [N]} \norm{\sum_{i=0}^{n-1} \rvz_{\pi(i)}}_{\infty},
	\end{align*}
	and then apply Lemma~\ref{lem:basic-balancing-reordering} to the last term on the right hand side in Ineq.~\eqref{eq:pf-lem-GraB:relation^2}.
	\begin{align*}
		\norm{\rvz_{\pi(i)}}_2 &= \norm{\nabla f_{\pi_q(i)}(\rvx_q^i) - \frac{1}{N}\sum_{l=0}^{N-1} \nabla f_{\pi_{q-1}(l)}\left(\rvx_{q-1}^l\right)}_2\\
		&= \norm{\left( \nabla f_{\pi_q(i)}(\rvx_q^i) - \frac{1}{N}\sum_{l=0}^{N-1} \nabla f_{\pi_{q-1}(l)}\left(\rvx_{q-1}^l\right)\right) \pm \left( \nabla f_{\pi_q(i)}(\rvx_q) - \frac{1}{N}\sum_{l=0}^{N-1} \nabla f_{\pi_q(l)}(\rvx_q)\right)}_2\\
		&\leq \norm{ \nabla f_{\pi_q(i)}(\rvx_q^i) - \nabla f_{\pi_q(i)}(\rvx_q) }_2 + \norm{ \frac{1}{N}\sum_{l=0}^{N-1} \nabla f_{\pi_{q-1}(l)}\left(\rvx_{q-1}^l\right) - \frac{1}{N}\sum_{l=0}^{N-1} \nabla f_{\pi_q(l)}(\rvx_q) }_2 + \varsigma\\
		&\leq \norm{ \nabla f_{\pi_q(i)}(\rvx_q^i) - \nabla f_{\pi_q(i)}(\rvx_q) }_2 + \norm{ \frac{1}{N}\sum_{l=0}^{N-1} \nabla f_{l}\left(\rvx_{q-1}^{\pi_{q-1}^{-1}(l)}\right) - \frac{1}{N}\sum_{l=0}^{N-1} \nabla f_{l}(\rvx_q) }_2 + \varsigma\\
		&\leq L_{2,\infty}\norm{\rvx_q^i - \rvx_q}_{\infty} + \frac{1}{N}\sum_{l=0}^{N-1}L_{2,\infty}\norm{\rvx_{q-1}^{\pi_{q-1}^{-1}(l)} - \rvx_q}_{\infty} +\varsigma\\
		&\leq L_{2,\infty}\norm{\rvx_q^i - \rvx_q}_{\infty} + \frac{1}{N}\sum_{l=0}^{N-1}L_{2,\infty}\left( \norm{\rvx_{q-1}^{\pi_{q-1}^{-1}(l)} - \rvx_{q-1}}_{\infty} + \norm{\rvx_{q-1} - \rvx_q}_{\infty}\right) + \varsigma\\
		&\leq L_{2,\infty} \Delta_q + 2L_{2,\infty} \Delta_{q-1}+ \varsigma\,.
	\end{align*}
	The preceding inequality holds for any $i\in \{0,1,\ldots, N-1\}$.
	\begin{align*}
		\norm{\sum_{i=0}^{N-1} \rvz_{\pi(i)}}_{\infty} &= \norm{ \sum_{i=0}^{N-1} \left( \nabla f_{\pi_q(i)}(\rvx_q^i) - \frac{1}{N}\sum_{l=0}^{N-1} \nabla f_{\pi_{q-1}(l)}(\rvx_{q-1}^l)  \right) }_\infty \\
		&= \norm{ \sum_{i=0}^{N-1}  \nabla f_{\pi_q(i)}(\rvx_q^i) - \sum_{i=0}^{N-1} \nabla f_{\pi_{q-1}(i)}(\rvx_{q-1}^i)   }_\infty  \\
		&= \norm{ \sum_{i=0}^{N-1}  \nabla f_{i}\left(\rvx_q^{\pi_q^{-1}(i)}\right) - \sum_{i=0}^{N-1} \nabla f_{i}\left(\rvx_{q-1}^{\pi_{q-1}^{-1}(i)}\right)   }_\infty\\
		&\leq \sum_{i=0}^{N-1}\norm{ \nabla f_{i}\left(\rvx_q^{\pi_q^{-1}(i)}\right) - \nabla f_{i}\left(\rvx_{q-1}^{\pi_{q-1}^{-1}(i)}\right)   }_\infty\\
		&\leq L_{\infty}\sum_{i=0}^{N-1}\norm{ \rvx_q^{\pi_q^{-1}(i)} - \rvx_{q-1}^{\pi_{q-1}^{-1}(i)}   }_\infty\\
		&\leq L_{\infty}\sum_{i=0}^{N-1} \left( \norm{ \rvx_q^{\pi_q^{-1}(i)} - \rvx_{q}   }_\infty + \norm{ \rvx_q- \rvx_{q-1} }_\infty + \norm{ \rvx_{q-1}- \rvx_{q-1}^{\pi_{q-1}^{-1}(i)}   }_\infty  \right)\\
		&\leq L_{\infty} N \Delta_q + 2L_{\infty}N \Delta_{q-1}\,.
	\end{align*}
	For any $n \in [N]$, we have
	\begin{align*}
		&\norm{\sum_{i=0}^{n-1} \rvz_{\pi(i)}}_{\infty} \\
		&= \norm{\sum_{i=0}^{n-1} \left( \nabla f_{\pi_{q}(i)}\left(\rvx_{q}^{i}\right) - \frac{1}{N}\sum_{l=0}^{N-1}\nabla f_{\pi_{q-1}(l)}(\rvx_{q-1}^l)  \right) }_{\infty}\\
		&= \norm{\sum_{i=0}^{n-1} \left(\left( \nabla f_{\pi_{q}(i)}\left(\rvx_{q}^{i}\right) - \frac{1}{N}\sum_{l=0}^{N-1}\nabla f_{\pi_{q-1}(l)}(\rvx_{q-1}^l)  \right) \pm \left( \nabla f_{\pi_{q}(i)}\left(\rvx_{q}\right) - \frac{1}{N}\sum_{l=0}^{N-1}\nabla f_{\pi_q(l)}(\rvx_{q})  \right) \right) }_{\infty}\\
		&\leq \norm{\sum_{i=0}^{n-1} \left(\nabla f_{\pi_q(i)}(\rvx_q^i) - \nabla f_{\pi_q(i)} (\rvx_q)\right) }_{\infty} + \norm{ \sum_{i=0}^{n-1}\left(\frac{1}{N}\sum_{l=0}^{N-1}\nabla f_{\pi_{q-1}(l)}(\rvx_{q-1}^l)- \frac{1}{N}\sum_{l=0}^{N-1}\nabla f_{\pi_q(l)}(\rvx_{q}) \right) }_{\infty} + \bar\phi_q\\
		&\leq \norm{\sum_{i=0}^{n-1} \left(\nabla f_{\pi_q(i)}(\rvx_q^i) - \nabla f_{\pi_q(i)} (\rvx_q)\right) }_{\infty} + \norm{ \sum_{i=0}^{n-1}\left(\frac{1}{N}\sum_{l=0}^{N-1}\nabla f_{l}\left(\rvx_{q-1}^{\pi_{q-1}^{-1}(l)}\right)- \frac{1}{N}\sum_{l=0}^{N-1}\nabla f_{l}(\rvx_{q}) \right) }_{\infty} + \bar\phi_q\\
		&\leq \sum_{i=0}^{n-1}\norm{\nabla f_{\pi_q(i)}(\rvx_q^i) - \nabla f_{\pi_q(i)} (\rvx_q)}_{\infty} + \sum_{i=0}^{n-1}\frac{1}{N}\sum_{l=0}^{N-1}\norm{ \nabla f_{l}\left(\rvx_{q-1}^{\pi_{q-1}^{-1}(l)}\right)- \nabla f_{l}(\rvx_{q})  }_{\infty}+ \bar\phi_q\\
		&\leq L_{\infty}\sum_{i=0}^{n-1}\norm{\rvx_q^i - \rvx_q}_{\infty} + L_{\infty}\sum_{i=0}^{n-1}\frac{1}{N}\sum_{l=0}^{N-1}\norm{\rvx_{q-1}^{\pi_{q-1}^{-1}(l)} - \rvx_q}_{\infty} + \bar\phi_q\\
		&\leq L_{\infty}\sum_{i=0}^{n-1}\norm{\rvx_q^i - \rvx_q}_{\infty} + L_{\infty}\sum_{i=0}^{n-1}\frac{1}{N}\sum_{l=0}^{N-1}\left( \norm{\rvx_{q-1}^{\pi_{q-1}^{-1}(l)} - \rvx_{q-1}}_{\infty} + \norm{\rvx_{q-1} - \rvx_q}_{\infty}\right) + \bar\phi_q\\
		&\leq L_{\infty} N \Delta_q + 2L_{\infty} N \Delta_{q-1} + \bar\phi_q\,.
	\end{align*}
	Since it holds for all $n \in [N]$, we have
	\begin{align*}
		\max_{n\in [N]} \norm{\sum_{i=0}^{n-1} \rvz_{\pi(i)}}_{\infty} \leq L_{\infty} N \Delta_q + 2L_{\infty} N \Delta_{q-1} + \bar\phi_q\,.
	\end{align*}
	Now, applying Lemma~\ref{lem:basic-balancing-reordering} to the last term on the right hand side in Ineq.~\eqref{eq:pf-lem-GraB:relation^2}, we can get
	\begin{align*}
		\bar\phi_{q+1} &\leq \left( 3L_{\infty} N \Delta_q + 2L_{\infty} N \Delta_{q-1}\right) + \frac{1}{2} \left(L_{\infty} N \Delta_q + 2L_{\infty} N \Delta_{q-1} + \bar\phi_q \right) \\
		&\quad+ \left(L_{\infty} N \Delta_q + 2L_{\infty}N \Delta_{q-1}\right)+ \frac{1}{2}C \left( L_{2,\infty} \Delta_q + 2L_{2,\infty} \Delta_{q-1}+ \varsigma\right)\\
		&\leq \left( \frac{9}{2}L_{\infty} N + \frac{1}{2}CL_{2,\infty}\right)\Delta_q + \left( 5L_{\infty} N + CL_{2,\infty}\right)\Delta_{q-1} + \frac{1}{2} \bar\phi_q + \frac{1}{2}C\varsigma\\
		&\leq \left( \frac{9}{2}L_{\infty} N + \frac{1}{2}CL_{2,\infty}\right)\left( \frac{32}{31}\gamma \bar \phi_q + \frac{32}{31}\gamma N \norm{\nabla f(\rvx_q)}_{\infty} \right) \\
		&\quad+ \left( 5L_{\infty} N + CL_{2,\infty}\right)\left(  \frac{32}{31}\gamma \bar \phi_{q-1} + \frac{32}{31}\gamma N \norm{\nabla f(\rvx_{q-1})}_{\infty} \right) + \frac{1}{2} \bar\phi_q + \frac{1}{2}C\varsigma\,,
	\end{align*}
	where the last inequality is due to Lemma~\ref{lem:parameter drift} that $\Delta_q \leq \frac{32}{31}\gamma \bar \phi_q + \frac{32}{31}\gamma N \norm{\nabla f(\rvx_q)}_{\infty}$.

	If $\textcolor{blue}{\gamma L_{\infty} N \leq \frac{1}{128}}$ and $\textcolor{blue}{\gamma L_{2,\infty} C \leq \frac{1}{128}}$, then $\left(\frac{9}{2}\gamma L_{\infty} N + \frac{1}{2}\gamma CL_{2,\infty}\right)\cdot \frac{32}{31} \leq \frac{5}{124}$ and $\left(5\gamma L_{\infty} N + C\gamma L_{2,\infty}\right)\cdot \frac{32}{31} \leq \frac{6}{124}$; we get
	\begin{align*}
		\bar\phi_{q+1} &\leq \frac{67}{124} \bar\phi_q + \frac{6}{124} \bar\phi_{q-1} + \frac{5}{124} N \norm{\nabla f(\rvx_q)}_{\infty} + \frac{6}{124} N \norm{\nabla f(\rvx_{q-1})}_{\infty} + \frac{1}{2}C\varsigma\,.
	\end{align*}
	Then, using $\norm{\rvx}_p \leq \norm{\rvx}$ for $p\geq 2$, we get
	\begin{align*}
		\left(\bar \phi_{q+1}\right)^2
		&\leq \left( \frac{67}{124} \bar\phi_q + \frac{6}{124} \bar\phi_{q-1} + \frac{5}{124} N \norm{\nabla f(\rvx_q)}_{} + \frac{6}{124} N \norm{\nabla f(\rvx_{q-1})}_{} + \frac{1}{2}C\varsigma \right)^2 \\
		&\leq 2\cdot \left( \frac{67}{124} \bar \phi_q \right)^2 + 8\cdot \left( \frac{6}{124} \bar \phi_{q-1} \right)^2 + 8 \cdot \left( \frac{5}{124}N \norm{\nabla f(\rvx_q)}_{} \right)^2 + 8 \cdot \left( \frac{6}{124}N \norm{\nabla f(\rvx_{q-1})}_{} \right)^2+ 8 \cdot \left( \frac{1}{2}C\varsigma \right)^2\\
		&\leq \frac{3}{5}\left( \bar \phi_q \right)^2 + \frac{1}{50}\left( \bar \phi_{q-1} \right)^2 + \frac{1}{50}N^2 \normsq{\nabla f(\rvx_q)}_{} +\frac{1}{50}N^2 \normsq{\nabla f(\rvx_{q-1})}_{} + 2C^2 \varsigma^2\,.
	\end{align*}
	So the relation between $\bar \phi_q$ and $\bar \phi_{q-1}$ is 
	\begin{align*}
	\left(\bar \phi_{q}\right)^2 \leq \frac{3}{5}\left( \bar \phi_{q-1} \right)^2 + \frac{1}{50}\left( \bar \phi_{q-2} \right)^2 + \frac{1}{50}N^2 \normsq{\nabla f(\rvx_{q-1})}_{\infty} +\frac{1}{50}N^2 \normsq{\nabla f(\rvx_{q-2})}_{\infty} + 2C^2 \varsigma^2\,,
	\end{align*}
	for $q\geq 2$. Besides, we have $\left(\bar\phi_{0}\right)^2 \leq N^2 \varsigma^2$ and $\left(\bar\phi_{1}\right)^2 \leq N^2 \varsigma^2$.
	
	In this case, for Assumption~\ref{asm:order error}, $p=\infty$, $A_1= \frac{3}{5}$, $A_2 = \frac{1}{50}$, $A_3=\cdots=A_q =0$, $B_0 = 0, B_1=\frac{1}{50}N^2$, $B_2=\frac{1}{50}N^2$, $B_3=\cdots=B_q =0$ and $D = 2C^2 \varsigma^2$. Then, we verify that
	\begin{align*}
		\frac{255}{512} - \frac{1}{512N^2}\cdot \frac{\sum_{i=0}^{\nu} B_i }{1-\sum_{i=1}^{\nu} A_i} \geq \frac{254}{512} >0\,.
	\end{align*}
	Thus, we can set $c_1 = 3$ and $c_2 = 16$ for Theorem~\ref{thm:SGD}. In addition, for Theorem~\ref{thm:SGD}, $\nu=2$, $\left(\bar\phi_{0}\right)^2 \leq N^2 \varsigma^2$ and $\left(\bar\phi_{1}\right)^2 \leq N^2 \varsigma^2$. These lead to the upper bound,
	\begin{align*}
		\min_{q\in \{0,1\ldots, Q-1\}}\normsq{ \nabla f(\rvx_q) } 
		&= \gO\left(  \frac{ F_0 }{\gamma NQ} + \gamma^2 L_{2,\infty}^2 N^2\frac{1}{Q} \varsigma^2 + \gamma^2 L_{2,\infty}^2 C^2\varsigma^2 \right).
	\end{align*}
	where $F_0 = f(\rvx_0) - f_\ast$. Since Lemma~\ref{lem:basic-balancing-reordering} is used for each epoch (that is, for $Q$ times), so by the union bound, the preceding bound holds with probability at least $1-Q\delta$.
	
	Next, we summarize the constraints on the step size:
	\begin{align*}
		\gamma &\leq \min \left\{ \frac{1}{LN}, \frac{1}{32L_{2,p} N}, \frac{1}{32L_{p} N} \right\} ,\\
		\gamma  &\leq \frac{1}{128 L_{\infty} N},\\
		\gamma  &\leq \frac{1}{128 L_{2,\infty} C}.
	\end{align*}
	The first one is from Theorem~\ref{thm:SGD} and the other is from the derivation of the relation. For simplicity, we can use a tighter constraint
	\begin{align*}
		\gamma &\leq \min \left\{ \frac{1}{L N}, \frac{1}{128L_{2,\infty} (N+C)}, \frac{1}{128L_\infty N} \right\}.
	\end{align*}
	After we use the effective step size $\tilde \gamma \coloneqq \gamma N $, the constraint will be
	\begin{align*}
		\tilde \gamma \leq \min \left\{ \frac{1}{L}, \frac{1}{128 L_{2,\infty}\left (1+\frac{C}{N} \right)}, \frac{1}{128L_\infty} \right\},
	\end{align*}
	and the upper bound will be
	\begin{align*}
		\min_{q\in \{0,1\ldots, Q-1\}}\normsq{ \nabla f(\rvx_q) } 
		&= \gO\left(  \frac{ F_0 }{\tilde\gamma Q} + \tilde\gamma^2 L_{2,\infty}^2 \frac{1}{Q} \varsigma^2 + \tilde\gamma^2 L_{2,\infty}^2 \frac{1}{N^2}C^2\varsigma^2 \right).
	\end{align*}
	Applying Lemma~\ref{lem:step-size}, we get
	\begin{align*}
		\min_{q\in \{0,1\ldots, Q-1\}}\normsq{ \nabla f(\rvx_q) } &=
		\gO\left( \frac{\left( L +L_{2,\infty}\left (1+\frac{C}{N} \right)+ L_{\infty}\right)  F_0}{Q}  + \frac{ (L_{2,\infty}F_0 \varsigma)^{\frac{2}{3}} }{Q}  + \left(\frac{L_{2,\infty}F_0 C \varsigma }{ NQ}\right)^{\frac{2}{3}}  \right).
	\end{align*}	
\end{proof}

\subsection{PairGraB}
\label{apx-subsec:PairGraB}

PairGraB. Use \texttt{PairBR} (Algorithm~\ref{alg:PairBR}) as the \texttt{Permute} function in Algorithm~\ref{alg:SGD}, with the inputs of $\pi_q$, $\{ \nabla f_{\pi_q(n)} (\rvx_q^n) \}_{n=0}^{N-1}$ and $\frac{1}{N}\sum_{n=0}^{N-1}\nabla f_{\pi_q(n)} (\rvx_q^n)$, for each epoch $q$.

Thus, the key idea of our proof is as follows:
\begin{align*}
	\bar \phi_{q+1} \rightarrow \max_{n \in [N]} \norm{\sum_{i=0}^{n-1} \left( \nabla f_{\pi_{q+1}(i)}\left(\rvx_{q}^{ \pi_q^{-1} (\pi_{q+1}(i)) }\right) - \frac{1}{N}\sum_{l=0}^{N-1}\nabla f_{\pi_q(l)}(\rvx_{q}^l)  \right) }_{\infty}   \\
	\overset{\text{Lemma~\ref{lem:pair-balancing-reordering}}}{\rightarrow} \norm{\sum_{i=0}^{n-1} \left( \nabla f_{\pi_{q}(i)}\left(\rvx_{q}^{i}\right) - \frac{1}{N}\sum_{l=0}^{N-1}\nabla f_{\pi_q(l)}(\rvx_{q}^l)  \right) }_{\infty}  \rightarrow
	\bar \phi_{q}\,.
\end{align*}

\begin{proof}[Proof of Example~\ref{ex:PairGraB}]
	We need to find the relation between $\bar \phi_{q+1}$ and $\bar \phi_q$.
	\begin{align}
		\phi_{q+1}^n &= \norm{ \sum_{i=0}^{n-1} \left( \nabla f_{\pi_{q+1}(i)}(\rvx_{q+1}) - \nabla f(\rvx_{q+1})  \right) }_{\infty} \nonumber \\
		&= \norm{ \sum_{i=0}^{n-1} \left( \left( \nabla f_{\pi_{q+1}(i)}(\rvx_{q+1}) - \nabla f(\rvx_{q+1})  \right) \pm \left( \nabla f_{\pi_{q+1}(i)}\left(\rvx_{q}^{ \pi_q^{-1} (\pi_{q+1}(i)) }\right) - \frac{1}{N}\sum_{l=0}^{N-1}\nabla f_{\pi_q(l)}(\rvx_{q}^l)  \right)\right) }_{\infty}\nonumber\\
		&\leq \norm{ \sum_{i=0}^{n-1} \left( \nabla f_{\pi_{q+1}(i)}(\rvx_{q+1}) - \nabla f_{\pi_{q+1}(i)}\left(\rvx_{q}^{ \pi_q^{-1} (\pi_{q+1}(i)) }\right)  \right)  }_{\infty} \nonumber\\
		&\quad+ \norm{ \sum_{i=0}^{n-1} \left( \frac{1}{N}\sum_{l=0}^{N-1}\nabla f_{\pi_q(l)}(\rvx_{q+1}) - \frac{1}{N}\sum_{l=0}^{N-1}\nabla f_{\pi_q(l)}(\rvx_{q}^l) \right) }_{\infty} \nonumber\\
		&\quad+ \norm{\sum_{i=0}^{n-1} \left( \nabla f_{\pi_{q+1}(i)}\left(\rvx_{q}^{ \pi_q^{-1} (\pi_{q+1}(i)) }\right) - \frac{1}{N}\sum_{l=0}^{N-1}\nabla f_{\pi_q(l)}(\rvx_{q}^l)  \right)   }_{\infty}.\label{eq:pf-lem-PairGraB:relation}
	\end{align}
	Then,
	\begin{align*}
		\Term{1}{\eqref{eq:pf-lem-PairGraB:relation}}&= \norm{ \sum_{i=0}^{n-1} \left( \nabla f_{\pi_{q+1}(i)}(\rvx_{q+1}) - \nabla f_{\pi_{q+1}(i)}\left(\rvx_{q}^{ \pi_q^{-1} (\pi_{q+1}(i)) }\right)  \right)  }_{\infty}\\
		&\leq \sum_{i=0}^{n-1}\norm{  \nabla f_{\pi_{q+1}(i)}(\rvx_{q+1}) - \nabla f_{\pi_{q+1}(i)}\left(\rvx_{q}^{ \pi_q^{-1} (\pi_{q+1}(i)) }\right) }_{\infty}\\
		&\leq L_{\infty}\sum_{i=0}^{n-1}\norm{  \rvx_{q+1} - \rvx_{q}^{ \pi_q^{-1} (\pi_{q+1}(i)) } }_{\infty}\\
		&\leq L_{\infty}\sum_{i=0}^{n-1}\left( \norm{\rvx_{q+1}- \rvx_q}_{\infty} + \norm{\rvx_q - \rvx_{q}^{ \pi_q^{-1} (\pi_{q+1}(i)) }}_{\infty}\right)\\
		&\leq 2L_{\infty} N \Delta_q\,,
	\end{align*}
	\begin{align*}
		\Term{2}{\eqref{eq:pf-lem-PairGraB:relation}}&= \norm{ \sum_{i=0}^{n-1} \left( \frac{1}{N}\sum_{l=0}^{N-1}\nabla f_{\pi_q(l)}(\rvx_{q+1}) - \frac{1}{N}\sum_{l=0}^{N-1}\nabla f_{\pi_q(l)}(\rvx_{q}^l) \right) }_{\infty}\\
		&\leq \sum_{i=0}^{n-1}  \frac{1}{N}\sum_{l=0}^{N-1}\norm{ \nabla f_{\pi_q(l)}(\rvx_{q+1}) -\nabla f_{\pi_q(l)}(\rvx_{q}^l) }_{\infty}\\
		&\leq L_{\infty}\sum_{i=0}^{n-1}  \frac{1}{N}\sum_{l=0}^{N-1}\norm{\rvx_{q+1} -\rvx_{q}^l }_{\infty}\\
		&\leq L_{\infty}\sum_{i=0}^{n-1}  \frac{1}{N}\sum_{l=0}^{N-1}\left(\norm{\rvx_{q+1} -\rvx_{q} }_{\infty} + \norm{\rvx_{q} -\rvx_{q}^l }_{\infty} \right)\\
		&\leq 2L_{\infty} N \Delta_q\,.
	\end{align*}
	Since the preceding inequalities hold for all $n\in [N]$, we have
	\begin{align}
		\bar\phi_{q+1} \leq 4L_{\infty} N \Delta_q + \max_{n \in [N]}\norm{\sum_{i=0}^{n-1} \left( \nabla f_{\pi_{q+1}(i)}\left(\rvx_{q}^{ \pi_q^{-1} (\pi_{q+1}(i)) }\right) - \frac{1}{N}\sum_{l=0}^{N-1}\nabla f_{\pi_q(l)}(\rvx_{q}^l)  \right) }_{\infty}.\label{eq:pf-lem-PairGraB:relation^2}
	\end{align}
	
	Note that $\nabla f_{\pi_{q}(i)}\left(\rvx_{q}^{i}\right) - \frac{1}{N}\sum_{l=0}^{N-1}\nabla f_{\pi_{q}(l)}(\rvx_{q}^l)$ and $\nabla f_{\pi_{q+1}(i)}\left(\rvx_{q}^{ \pi_q^{-1} (\pi_{q+1}(i)) }\right) - \frac{1}{N}\sum_{l=0}^{N-1}\nabla f_{\pi_{q}(l)}(\rvx_{q}^l) $ correspond to $\rvz_{\pi(i)}$ and $\rvz_{\pi' (i)}$ in Lemma~\ref{lem:pair-balancing-reordering}, respectively. We next get the upper bounds of 
	\begin{align*}
		\norm{\rvz_{\pi(i)}}_2, \norm{\sum_{i=0}^{N-1} \rvz_{\pi(i)}}_{\infty} \text{ and } \max_{n\in [N]} \norm{\sum_{i=0}^{n-1} \rvz_{\pi(i)}}_{\infty},
	\end{align*}
	and then apply Lemma~\ref{lem:pair-balancing-reordering} to the last term on the right hand side in Ineq.~\eqref{eq:pf-lem-PairGraB:relation^2}.
	\begin{align*}
		\norm{\rvz_{\pi_q(i)}}_2 &= \norm{\nabla f_{\pi_q(i)}(\rvx_q^i) - \frac{1}{N}\sum_{l=0}^{N-1} \nabla f_{\pi_q(l)}(\rvx_q^l)}_2\\
		&= \norm{\left( \nabla f_{\pi_q(i)}(\rvx_q^i) - \frac{1}{N}\sum_{l=0}^{N-1} \nabla f_{\pi_q(l)}(\rvx_q^l)\right) \pm \left( \nabla f_{\pi_q(i)}(\rvx_q) - \frac{1}{N}\sum_{l=0}^{N-1} \nabla f_{\pi_q(l)}(\rvx_q)\right)}_2\\
		&\leq \norm{ \nabla f_{\pi_q(i)}(\rvx_q^i) - \nabla f_{\pi_q(i)}(\rvx_q) }_2 + \norm{ \frac{1}{N}\sum_{l=0}^{N-1} \nabla f_{\pi_q(l)}(\rvx_q^l) - \frac{1}{N}\sum_{l=0}^{N-1} \nabla f_{\pi_q(l)}(\rvx_q) }_2 \\
		&\quad + \norm{ \nabla f_{\pi_q(i)}(\rvx_q) - \frac{1}{N}\sum_{l=0}^{N-1} \nabla f_{\pi_q(l)}(\rvx_q)}_2\\
		&\leq L_{2,\infty}\norm{\rvx_q^i - \rvx_q}_{\infty} + \frac{1}{N}\sum_{l=0}^{N-1}L_{2,\infty}\norm{\rvx_q^l - \rvx_q}_{\infty}+ \varsigma\\
		&\leq 2L_{2,\infty} \Delta_q + \varsigma\,,
	\end{align*}
	\begin{align*}
		\norm{\sum_{i=0}^{N-1} \rvz_{\pi_q(i)}}_{\infty} = \norm{ \sum_{i=0}^{N-1} \left( \nabla f_{\pi_q(i)}(\rvx_q^i) - \frac{1}{N}\sum_{l=0}^{N-1} \nabla f_{\pi_q(l)}(\rvx_q^l)  \right) }_\infty = 0\,.
	\end{align*}
	For any $n \in [N]$, we have
	\begin{align*}
		&\norm{\sum_{i=0}^{n-1} \rvz_{\pi(i)}}_{\infty} \\
		&= \norm{\sum_{i=0}^{n-1} \left( \nabla f_{\pi_{q}(i)}\left(\rvx_{q}^{i}\right) - \frac{1}{N}\sum_{l=0}^{N-1}\nabla f_{\pi_q(l)}(\rvx_{q}^l)  \right) }_{\infty}\\
		&= \norm{\sum_{i=0}^{n-1} \left(\left( \nabla f_{\pi_{q}(i)}\left(\rvx_{q}^{i}\right) - \frac{1}{N}\sum_{l=0}^{N-1}\nabla f_{\pi_q(l)}(\rvx_{q}^l)  \right) \pm \left( \nabla f_{\pi_{q}(i)}\left(\rvx_{q}\right) - \frac{1}{N}\sum_{l=0}^{N-1}\nabla f_{\pi_q(l)}(\rvx_{q})  \right) \right) }_{\infty}\\
		&\leq \norm{\sum_{i=0}^{n-1} \left(\nabla f_{\pi_q(i)}(\rvx_q^i) - \nabla f_{\pi_q(i)} (\rvx_q)\right) }_{\infty} + \norm{ \sum_{i=0}^{n-1}\left(\frac{1}{N}\sum_{l=0}^{N-1}\nabla f_{\pi_q(l)}(\rvx_{q}^l)- \frac{1}{N}\sum_{l=0}^{N-1}\nabla f_{\pi_q(l)}(\rvx_{q}) \right) }_{\infty} + \bar\phi_q\\
		&\leq \sum_{i=0}^{n-1}\norm{\nabla f_{\pi_q(i)}(\rvx_q^i) - \nabla f_{\pi_q(i)} (\rvx_q)}_{\infty} + \sum_{i=0}^{n-1}\frac{1}{N}\sum_{l=0}^{N-1}\norm{ \nabla f_{\pi_q(l)}(\rvx_{q}^l)- \nabla f_{\pi_q(l)}(\rvx_{q})  }_{\infty}+ \bar\phi_q\\
		&\leq L_{\infty}\sum_{i=0}^{n-1}\norm{\rvx_q^i - \rvx_q}_{\infty} + L_{\infty}\sum_{i=0}^{n-1}\frac{1}{N}\sum_{l=0}^{N-1}\norm{\rvx_q^l - \rvx_q}_{\infty} + \bar\phi_q\\
		&\leq 2L_{\infty} N \Delta_q + \bar\phi_q\,.
	\end{align*}
	Since it holds for all $n \in [N]$, we have
	\begin{align*}
		\max_{n\in [N]} \norm{\sum_{i=0}^{n-1} \rvz_{\pi(i)}}_{\infty} \leq 2L_{\infty} N \Delta_q + \bar\phi_q\,.
	\end{align*}
	Now, applying Lemma~\ref{lem:pair-balancing-reordering} to the last term on the right hand side in Ineq.~\eqref{eq:pf-lem-PairGraB:relation^2}, we can get	
	\begin{align*}
		\bar\phi_{q+1} &\leq 4L_{\infty} N \Delta_q + \frac{1}{2} \left(2L_{\infty} N \Delta_q + \bar\phi_q \right) + C \left( 2L_{2,\infty} \Delta_q + \varsigma\right)\\
		&\leq \left( 5L_{\infty} N + 2L_{2,\infty}C\right)\Delta_q + \frac{1}{2} \bar\phi_q + C\varsigma\\
		&\leq \left( 5L_{\infty} N + 2L_{2,\infty}C\right) \left( \frac{32}{31}\gamma \bar \phi_q + \frac{32}{31}\gamma N \norm{\nabla f(\rvx_q)}_{\infty} \right)+ \frac{1}{2} \bar\phi_q + C\varsigma\,,
	\end{align*}
	where the last inequality is due to Lemma~\ref{lem:parameter drift} that $\Delta_q \leq \frac{32}{31}\gamma \bar \phi_q + \frac{32}{31}\gamma N \norm{\nabla f(\rvx_q)}_{\infty}$. If $\textcolor{blue}{\gamma L_{\infty} N \leq \frac{1}{64}}$ and $\textcolor{blue}{\gamma L_{2,\infty} C \leq \frac{1}{64}}$, then $\frac{1}{2} + \frac{7}{64}\cdot \frac{32}{31} = \frac{38}{62}$ and $\frac{7}{64}\cdot \frac{32}{31} = \frac{7}{62}$; we get
	\begin{align*}
		\bar\phi_{q+1} \leq \frac{38}{62} \bar\phi_q + \frac{7}{62}N \norm{\nabla f(\rvx_q)}_{\infty} + C \varsigma\,.
	\end{align*}
	Then, using $\norm{\rvx}_p \leq \norm{\rvx}$ for $p\geq 2$, we get
	\begin{align*}
		\left(\bar \phi_{q+1}\right)^2
		&\leq \left( \frac{38}{62}\bar \phi_q + \frac{7}{62}N \norm{\nabla f(\rvx_q)}_{} + C\varsigma \right)^2 \\
		&\leq 2\cdot \left( \frac{38}{62} \bar \phi_q \right)^2 + 4 \cdot \left( \frac{7}{62}N \norm{\nabla f(\rvx_q)}_{} \right)^2+ 4 \cdot \left( C\varsigma \right)^2\\
		&\leq \frac{4}{5}\left( \bar \phi_q \right)^2 + \frac{3}{50}N^2 \normsq{\nabla f(\rvx_q)}_{} + 4C^2 \varsigma^2\,.
	\end{align*}
	
	So the relation between $\bar \phi_q$ and $\bar \phi_{q-1}$ is 
	\begin{align*}
		\left(\bar \phi_{q}\right)^2 \leq \frac{4}{5}\left( \bar \phi_q \right)^2 + \frac{3}{50}N^2 \normsq{\nabla f(\rvx_q)}_{} + 4C^2 \varsigma^2\,.
	\end{align*}
	for $q\geq 1$. Besides, we have $\left(\bar\phi_{0}\right)^2 \leq N^2 \varsigma^2$.
	
	In this case, for Assumption~\ref{asm:order error}, $p=\infty$, $A_1=\frac{4}{5}$, $A_2=\cdots=A_q = 0$, $B_0=0$, $B_1 = \frac{3}{50}N^2$, $B_2=\cdots=B_q =0$ and $D = 4C^2 \varsigma^2$. Then, we verify that
	\begin{align*}
		\frac{255}{512} - \frac{1}{512N^2}\cdot \frac{\sum_{i=0}^{\nu} B_i }{1-\sum_{i=1}^{\nu} A_i} \geq \frac{254}{512} >0\,.
	\end{align*}
	Thus, we can set $c_1 = 3$ and $c_2 = 30$ for Theorem~\ref{thm:SGD}. In addition, for Theorem~\ref{thm:SGD}, $\nu=1$ and $\left(\bar\phi_{0}\right)^2 \leq N^2 \varsigma^2$. These lead to the upper bound,
	\begin{align*}
		\min_{q\in \{0,1\ldots, Q-1\}}\normsq{ \nabla f(\rvx_q) } 
		&= \gO\left(  \frac{ F_0 }{\gamma NQ} + \gamma^2 L_{2,\infty}^2 N^2\frac{1}{Q} \varsigma^2 + \gamma^2 L_{2,\infty}^2 C^2\varsigma^2 \right).
	\end{align*}
	where $F_0 = f(\rvx_0) - f_\ast$. Since Lemma~\ref{lem:pair-balancing-reordering} is used for each epoch (that is, for $Q$ times), so by the union bound, the preceding bound holds with probability at least $1-Q\delta$.
	
	Next, we summarize the constraints on the step size:
	\begin{align*}
		\gamma &\leq \min \left\{ \frac{1}{LN}, \frac{1}{32L_{2,p} N}, \frac{1}{32L_{p} N} \right\}, \\
		\gamma  &\leq \frac{1}{64 L_{\infty} N}\,,\\
		\gamma  &\leq \frac{1}{64 L_{2,\infty}C}\,.
	\end{align*}
	The first one is from Theorem~\ref{thm:SGD} and the other is from the derivation of the relation. For simplicity, we can use a tighter constraint
	\begin{align*}
		\gamma &\leq \min \left\{ \frac{1}{L N}, \frac{1}{64L_{2,\infty} (N+C)}, \frac{1}{64L_\infty N} \right\}.
	\end{align*}
	After we use the effective step size $\tilde \gamma \coloneqq \gamma N $, the constraint will be
	\begin{align*}
		\tilde \gamma \leq \min \left\{ \frac{1}{L}, \frac{1}{64 L_{2,\infty}\left (1+\frac{C}{N} \right)}, \frac{1}{64L_\infty} \right\},
	\end{align*}
	and the upper bound will be
	\begin{align*}
		\min_{q\in \{0,1\ldots, Q-1\}}\normsq{ \nabla f(\rvx_q) } 
		&= \gO\left(  \frac{ F_0 }{\tilde\gamma Q} + \tilde\gamma^2 L_{2,\infty}^2 \frac{1}{Q} \varsigma^2 + \tilde\gamma^2 L_{2,\infty}^2 \frac{1}{N^2}C^2\varsigma^2 \right).
	\end{align*}
	Applying Lemma~\ref{lem:step-size}, we get
	\begin{align*}
		\min_{q\in \{0,1\ldots, Q-1\}}\normsq{ \nabla f(\rvx_q) } &=
		\gO\left( \frac{\left( L +L_{2,\infty}\left (1+\frac{C}{N} \right)+ L_{\infty}\right)  F_0}{Q}  + \frac{ (L_{2,\infty}F_0 \varsigma)^{\frac{2}{3}} }{Q}  + \left(\frac{L_{2,\infty}F_0 C \varsigma }{ NQ}\right)^{\frac{2}{3}}  \right).
	\end{align*}
\end{proof}

\section{Theorem~\ref{thm:FL}}

\subsection{Order Error in FL}

\textit{Theoretical understanding of Definition~\ref{def:FL:order error}.} Following \citet{smith2021origin}, we can prove that, for small finite step sizes, the cumulative updates in one epoch are
\begin{align}
	\rvx_{q+1} - \rvx_{q} &= -\gamma \frac{1}{S}\sum_{n=0}^{N-1} \sum_{k=0}^{K-1}  \nabla f_{\pi_q(n)}\left(\rvx_{q,k}^{n}\right)\nonumber\\
	&= -\gamma \frac{1}{S}\sum_{n=0}^{N-1} \sum_{k=0}^{K-1} \nabla f_{\pi(n)}\left( \rvx_{q} \right) \nonumber\\
	&\quad+ \gamma^2 \frac{1}{S}\sum_{n=0}^{N-1} \sum_{k=0}^{K-1} \nabla \nabla f_{\pi(n)}(\rvx_{q}) \sum_{j=0}^{k-1} \nabla f_{\pi(n)}\left( \rvx_{q} \right)\nonumber\\
	&\quad+ \gamma^2 \frac{1}{S}\sum_{n=0}^{N-1} \sum_{k=0}^{K-1}\nabla \nabla f_{\pi(n)}(\rvx_{q}) \frac{1}{S} \sum_{i=0}^{v(n)-1} \sum_{j=0}^{K-1} \nabla f_{\pi(i)}\left( \rvx_{q} \right) +  \gO\left( \gamma^3K^3 N^3\frac{1}{S^3} \right).\label{eq:order error-FL-analysis}
\end{align}
Similar to the analysis in the main body, it can be seen that the error vectors are caused by the second and third terms on the right hand side in Eq.~\eqref{eq:order error-FL-analysis}. Note that when we consider $\nabla \nabla f_{\pi(n)}(\rvx_0^0) \approx L$, the second term can be also seen as a optimization vector (with the same direction as $\nabla f(\rvx_{q,0}^0)$). This is mainly because the local solver is the classic SGD in our setup, and it can be different when the local solver is the permutation-based SGD. As a result, we next focus on the third term. With a similar decomposition in the main body, our goal turns to suppress the error vector as follows
\begin{align*}
	\text{Error vector} = \gamma^2 \frac{1}{S}\sum_{n=0}^{N-1} \sum_{k=0}^{K-1}\nabla \nabla f_{\pi(n)}(\rvx_{q}) \frac{1}{S} \sum_{i=0}^{v(n)-1} \sum_{j=0}^{K-1} \left( \nabla f_{\pi(i)}\left( \rvx_{q,0}^0 \right) - \nabla f_{\pi(i)} \left( \rvx_q \right) \right).
\end{align*}
One straightforward way is to minimize the norm of error vector
\begin{align*}
	\norm{\text{Error vector}} &\leq \gamma^2 L \norm{\frac{1}{S}\sum_{n=0}^{N-1} \sum_{k=0}^{K-1} \frac{1}{S} \sum_{i=0}^{v(n)-1} \sum_{j=0}^{K-1} \left( \nabla f_{\pi(i)}\left( \rvx_{q} \right) - \nabla f_{\pi(i)} \left( \rvx_q \right) \right)}\\
	&\leq \gamma^2 L K^2 \frac{1}{S^2} \sum_{n=0}^{N-1}\norm{   \sum_{i=0}^{v(n)-1}  \left( \nabla f_{\pi(i)}\left( \rvx_{q} \right) - \nabla f_{\pi(i)} \left( \rvx_q \right) \right)}\\
	&\leq \gamma^2 L K^2 N\frac{1}{S^2} \bar\varphi_q\,.
\end{align*}

\begin{proof}[Proof of Eq.~\eqref{eq:order error-FL-analysis}]

The Taylor expansion of $h$ at $x=x_0$ is $\sum_{n=0}^{\infty} \frac{1}{n!} h^{(n)} (x_0) (x-x_0)^n$. Here we only need
\begin{align*}
	h(x) = h(x_0) + h'(x_0) (x-x_0) + \gO \left((x-x_0)^2\right).
\end{align*}
For FL with FL-AP, for any epoch $q\geq 0$,
\begin{align*}
	\rvx_{q+1} - \rvx_{q} = -\gamma \frac{1}{S}\sum_{n=0}^{N-1} \sum_{k=0}^{K-1}  \nabla f_{\pi_q(n)}\left(\rvx_{q,k}^{n}\right),
\end{align*}
where we omit the server learning rate $\eta$ here. Besides, we adopt the GD as the local solver at each client. Next, we drop the subscripts $q$ for convenience. For any $n\in\{0,1,\ldots,N-1\}$ and $k\in\{0,1,K-1\}$,
\begin{align*}
	\rvx_{k}^{n} - \rvx_{0}^{0} = -\gamma \sum_{j=0}^{k-1} \nabla f_{\pi(n)} \left(\rvx_j^n\right) - \gamma \frac{1}{S} \sum_{i=0}^{v(n)-1} \sum_{j=0}^{K-1} \nabla f_{\pi(i)} \left( \rvx_k^i \right).
\end{align*}
Then, by using Taylor expansion of $\nabla f_{\pi(n)}\left(\rvx_{k}^{n}\right)$, it follows that
\begin{align}
	\nabla f_{\pi(n)}\left( \rvx_k^n\right) &= \nabla f_{\pi(n)}\left( \rvx_0^0 \right) + \nabla \nabla f_{\pi(n)}(\rvx_0^0) \left( -\gamma \sum_{j=0}^{k-1} \nabla f_{\pi(n)} \left(\rvx_j^n\right) - \gamma \frac{1}{S} \sum_{i=0}^{v(n)-1} \sum_{j=0}^{K-1} \nabla f_{\pi(i)} \left( \rvx_k^i \right) \right) + \gO\left( \gamma^2 K^2N^2 \frac{1}{S^2} \right)\nonumber\\
	&= \nabla f_{\pi(n)}\left( \rvx_0^0 \right)  - \gamma \nabla \nabla f_{\pi(n)}(\rvx_0^0)  \sum_{j=0}^{k-1} \textcolor{blue}{\nabla f_{\pi(n)} \left(\rvx_j^n\right)} \nonumber\\
	&\quad - \gamma \nabla \nabla f_{\pi(n)}(\rvx_0^0) \frac{1}{S} \sum_{i=0}^{v(n)-1} \sum_{j=0}^{K-1} \textcolor{blue}{\nabla f_{\pi(i)} \left( \rvx_k^i \right)} + \gO\left( \gamma^2 K^2N^2 \frac{1}{S^2} \right).\label{eq:pf-def:order eror-FL}
\end{align}
where $v(n) \coloneqq \floor{\frac{n}{S}}S = n-n\mod S$. The remaining error is $\gO\left( \gamma^2 K^2 N^2 \frac{1}{S^2} \right)$ since there are $\gO\left( K^2 N^2 \frac{1}{S^2}\right)$ terms in the Taylor expansion proportional to $\gamma^2$ \citep{smith2021origin}. Then, noting that the expansion of Eq.~\eqref{eq:pf-def:order eror-FL} is also applied to $\nabla f_{\pi(n)} \left(\rvx_j^n\right)$ and $\nabla f_{\pi(i)} \left( \rvx_k^i \right)$ in Eq.~\eqref{eq:pf-def:order eror-FL}. Thus, by using Eq.~\eqref{eq:pf-def:order eror-FL} recursively, we get
\begin{align*}
	&\Term{2}{\eqref{eq:pf-def:order eror-FL}} \\
	&= - \gamma \nabla \nabla f_{\pi(n)}(\rvx_0^0)  \sum_{j=0}^{k-1} \textcolor{blue}{\nabla f_{\pi(n)} \left(\rvx_j^n\right)}\\
	&= - \gamma \nabla \nabla f_{\pi(n)}(\rvx_0^0) \sum_{j=0}^{k-1}  \left( {\color{blue}\nabla f_{\pi(n)}\left( \rvx_0^0 \right)  - \gamma \nabla \nabla f_{\pi(n)}(\rvx_0^0)  \sum_{b=0}^{j-1} \nabla f_{\pi(n)} \left(\rvx_b^n\right)} \right.\\
	&\hspace{12em}\left. {\color{blue}- \gamma \nabla \nabla f_{\pi(n)}(\rvx_0^0) \frac{1}{S} \sum_{a=0}^{v(n)-1} \sum_{b=0}^{K-1} \nabla f_{\pi(a)} \left( \rvx_b^a \right) + \gO\left( \gamma^2 K^2N^2 \frac{1}{S^2} \right)}  \right)\\
	&= - \gamma \nabla \nabla f_{\pi(n)}(\rvx_0^0) \sum_{j=0}^{k-1} \nabla f_{\pi(n)}\left( \rvx_0^0 \right) + \gO\left( \gamma^2 K^2  \right) + \gO\left( \gamma^2 K^2 N \frac{1}{S} \right) + \gO\left( \gamma^3 K^3 N^2 \frac{1}{S^2} \right)\\
	&= - \gamma \nabla \nabla f_{\pi(n)}(\rvx_0^0) \sum_{j=0}^{k-1} \nabla f_{\pi(n)}\left( \rvx_0^0 \right) + \gO\left( \gamma^2 K^2 N^2\frac{1}{S^2}\right),
\end{align*}	
\begin{align*}
	&\Term{3}{\eqref{eq:pf-def:order eror-FL}} \\
	&= - \gamma \nabla \nabla f_{\pi(n)}(\rvx_0^0) \frac{1}{S} \sum_{i=0}^{v(n)-1} \sum_{j=0}^{K-1} \textcolor{blue}{\nabla f_{\pi(i)} \left( \rvx_k^i \right)}\\ 
	&=- \gamma \nabla \nabla f_{\pi(n)}(\rvx_0^0) \frac{1}{S} \sum_{i=0}^{v(n)-1} \sum_{j=0}^{K-1} \left( {\color{blue}\nabla f_{\pi(i)}\left( \rvx_0^0 \right)  - \gamma \nabla \nabla f_{\pi(i)}(\rvx_0^0)  \sum_{b=0}^{j-1} \nabla f_{\pi(i)} \left(\rvx_b^i\right)}  \right.\\
	&\hspace{15em}\left.{\color{blue}- \gamma \nabla \nabla f_{\pi(i)}(\rvx_0^0) \frac{1}{S} \sum_{a=0}^{v(i)-1} \sum_{b=0}^{K-1} \nabla f_{\pi(a)} \left( \rvx_b^a \right) + \gO\left( \gamma^2 K^2N^2 \frac{1}{S^2} \right) } \right)\\
	&= - \gamma \nabla \nabla f_{\pi(n)}(\rvx_0^0) \frac{1}{S} \sum_{i=0}^{v(n)-1} \sum_{j=0}^{K-1} \nabla f_{\pi(i)}\left( \rvx_0^0 \right) + \gO\left( \gamma^2 K^2 N \frac{1}{S} \right) + \gO\left( \gamma^2K^2 N^2\frac{1}{S^2} \right) + \gO\left( \gamma^3K^3 N^3\frac{1}{S^3} \right)\\
	&= - \gamma \nabla \nabla f_{\pi(n)}(\rvx_0^0) \frac{1}{S} \sum_{i=0}^{v(n)-1} \sum_{j=0}^{K-1} \nabla f_{\pi(i)}\left( \rvx_0^0 \right) +  \gO\left( \gamma^2K^2 N^2\frac{1}{S^2} \right).
\end{align*}
Substituting the two terms on the right hand side in Eq.~\eqref{eq:pf-def:order eror-FL} gives
\begin{align*}
	\nabla f_{\pi(n)}\left( \rvx_k^n\right) &= \nabla f_{\pi(n)}\left( \rvx_0^0 \right) - \gamma \nabla \nabla f_{\pi(n)}(\rvx_0^0) \sum_{j=0}^{k-1} \nabla f_{\pi(n)}\left( \rvx_0^0 \right)\\
	&\quad- \gamma \nabla \nabla f_{\pi(n)}(\rvx_0^0) \frac{1}{S} \sum_{i=0}^{v(n)-1} \sum_{j=0}^{K-1} \nabla f_{\pi(i)}\left( \rvx_0^0 \right) +  \gO\left( \gamma^2K^2 N^2\frac{1}{S^2} \right).
\end{align*}
At last, we get
\begin{align*}
	\rvx_{q+1} - \rvx_{q} &= -\gamma \frac{1}{S}\sum_{n=0}^{N-1} \sum_{k=0}^{K-1}  \nabla f_{\pi_q(n)}\left(\rvx_{q,k}^{n}\right)\\
	&= -\gamma \frac{1}{S}\sum_{n=0}^{N-1} \sum_{k=0}^{K-1} \nabla f_{\pi(n)}\left( \rvx_0^0 \right) \\
	&\quad+ \gamma^2 \frac{1}{S}\sum_{n=0}^{N-1} \sum_{k=0}^{K-1} \nabla \nabla f_{\pi(n)}(\rvx_0^0) \sum_{j=0}^{k-1} \nabla f_{\pi(n)}\left( \rvx_0^0 \right)\\ 
	&\quad+ \gamma^2 \frac{1}{S}\sum_{n=0}^{N-1} \sum_{k=0}^{K-1}\nabla \nabla f_{\pi(n)}(\rvx_0^0) \frac{1}{S} \sum_{i=0}^{v(n)-1} \sum_{j=0}^{K-1} \nabla f_{\pi(i)}\left( \rvx_0^0 \right) +  \gO\left( \gamma^3K^3 N^3\frac{1}{S^3} \right).
\end{align*}
After recovering the subscripts $q$ and noting $\rvx_{q,0}^{0} = \rvx_q$, we get Eq.~\eqref{eq:order error-FL-analysis}.
\end{proof}

\subsection{Parameter Deviation in FL}

We define the maximum parameter deviation (drift) of FL in any epoch $q$, $\Delta_q$ as
\begin{align*}
	\Delta_q = \max\left\{\max_{\substack{n\in \{0, \ldots, N-1\}\\k \in \{0,\ldots, K-1\}}}\norm{ \rvx_{q,k}^{n} - \rvx_q}_{p},  \norm{\tilde\rvx_{q+1} - \rvx_q}_{p}\right\}.
\end{align*}
Here 
\begin{align*}
	\tilde\rvx_{q+1} - \rvx_q = - \gamma\hphantom{\eta} \frac{1}{S}\sum_{n=0}^{N-1}\sum_{k=0}^{K-1}\nabla f_{\pi_q(n)}\left(\rvx_{q,k}^{n}\right)	.
\end{align*}
Note that, due to the amplified updates \citep{wang2022unified},
\begin{align*}
	\rvx_{q+1} - \rvx_q = - \gamma\textcolor{blue}{\eta} \frac{1}{S}\sum_{n=0}^{N-1}\sum_{k=0}^{K-1}\nabla f_{\pi_q(n)}\left(\rvx_{q,k}^{n}\right).
\end{align*}
Then, we get the relation
\begin{align*}
	\norm{\rvx_{q+1}-\rvx_q}_{p} = \eta \norm{\tilde \rvx_{q+1}-\rvx_q}_{p} \leq \eta \Delta_q .
\end{align*}
Besides, to avoid ambiguity, we let $\tilde\rvx_{q+1} = \rvx_{q,K}^{N-1} = \rvx_{q,0}^{N}$.

\begin{lemma}
	\label{lem:FL:parameter drift}
	We first prove that if $\gamma L_p KN \frac{1}{S} \leq \frac{1}{32}$, the maximum parameter drift in FL is bounded:
	\begin{align*}
		\Delta_q &\leq \frac{32}{31} \gamma K \frac{1}{S}\bar \varphi_q +  \frac{32}{31}\gamma K N\frac{1}{S}\norm{\nabla f(\rvx_{q})}_{p} +  \frac{32}{31}\gamma K\varsigma\,,\\
		\left( \Delta_q \right)^2&\leq 4\gamma^2 K^2 \frac{1}{S^2}\left(\bar \varphi_q\right)^2 + 4\gamma^2 K^2 N^2\frac{1}{S^2}\normsq{\nabla f(\rvx_{q})}_{p} + 4\gamma^2 K^2\varsigma^2\,.
	\end{align*}
\end{lemma}
\begin{proof}
	Let $v(n) = \floor{\frac{n}{S}}\cdot S$. Then,
	\begin{align*}
		\rvx_{q,k}^{n} - \rvx_q &= \rvx_{q,k}^{n} - \rvx_{q,0}^{n} + \underbrace{\rvx_{q,0}^{n} - \rvx_{q,0}^{(v(n))}}_{=0} + \rvx_{q,0}^{(v(n))} - \rvx_q\\
		&=-\gamma\sum_{j=0}^{k-1}\nabla f_{\pi_q(n)}\left (\rvx_{q,j}^{n} \right) - \gamma \frac{1}{S}\sum_{i=0}^{v(n)-1}\sum_{j=0}^{K-1}\nabla f_{\pi_q(i)} \left( \rvx_{q,j}^{i} \right).
	\end{align*}
	For any $q>0$ and all $n \in \{0,1,\ldots, N-1\}$ and $k \in \{0,1,\ldots, K-1\}$, it follows that
	\begin{align}
		\norm{\rvx_{q,k}^{n} - \rvx_{q}}_{p} &= \norm{\gamma\sum_{j=0}^{k-1}\nabla f_{\pi_{q}(n)}(\rvx_{q,j}^{n}) + \gamma \frac{1}{S}\sum_{i=0}^{v(n)-1}\sum_{j=0}^{K-1}\nabla f_{\pi_{q}(i)}(\rvx_{q,j}^{i})}_{p}\nonumber\\
		&\leq \norm{\gamma\sum_{j=0}^{k-1}\nabla f_{\pi_{q}(n)}(\rvx_{q,j}^{n})}_{p} + \norm{\gamma \frac{1}{S}\sum_{i=0}^{v(n)-1}\sum_{j=0}^{K-1}\nabla f_{\pi_{q}(i)}(\rvx_{q,j}^{i})}_{p}.\label{eq:pf-thm-FL:param drift}
	\end{align}
	Then, we bound the two terms on the right hand side in Ineq.~\eqref{eq:pf-thm-FL:param drift} respectively.
	\begin{align*}
		\Term{1}{\eqref{eq:pf-thm-FL:param drift}} &= \norm{\gamma\sum_{j=0}^{k-1}\nabla f_{\pi_{q}(n)}(\rvx_{q,j}^{n})}_{p} \\
		&= \gamma\norm{\sum_{j=0}^{k-1}\left(\nabla f_{\pi_{q}(n)}(\rvx_{q,j}^{n}) - \nabla f_{\pi_q(n)}(\rvx_{q}) + \nabla f_{\pi_q(n)}(\rvx_{q}) - \nabla f(\rvx_{q}) + \nabla f(\rvx_{q})\right)}_{p} \\
		&\leq \gamma\norm{\sum_{j=0}^{k-1}\left(\nabla f_{\pi_{q}(n)}(\rvx_{q,j}^{n}) - \nabla f_{\pi_{q}(n)}(\rvx_{q})\right)}_{p} + \gamma\norm{\sum_{j=0}^{k-1}\left(\nabla f_{\pi_{q}(n)}(\rvx_{q})- \nabla f(\rvx_{q})\right)}_{\infty} +\gamma\norm{\sum_{j=0}^{k-1}\left(\nabla f(\rvx_{q})\right)}_{p} \\
		&\leq \gamma\sum_{j=0}^{k-1}\norm{\nabla f_{\pi_{q}(n)}(\rvx_{q,j}^{n}) - \nabla f_{\pi_{q}(n)}(\rvx_{q})}_{p} + \gamma\sum_{j=0}^{k-1}\norm{\nabla f_{\pi_{q}(n)}(\rvx_{q})- \nabla f(\rvx_{q})}_{p} +\gamma \sum_{j=0}^{k-1}\norm{\nabla f(\rvx_{q})}_{p} \\
		&\leq \gamma L_{p}\sum_{j=0}^{k-1}\norm{\rvx_{q,j}^{n} - \rvx_{q}}_{p} + \gamma\sum_{j=0}^{k-1}\varsigma +\gamma \sum_{j=0}^{k-1}\norm{\nabla f(\rvx_{q})}_{p} \\
		&\leq \gamma L_{p}K\Delta_q + \gamma K\varsigma +\gamma K \norm{\nabla f(\rvx_{q})}_{p},
	\end{align*}
	\begin{align*}
		\Term{2}{\eqref{eq:pf-thm-FL:param drift}} &= \norm{\gamma \frac{1}{S}\sum_{i=0}^{v(n)-1}\sum_{j=0}^{K-1}\nabla f_{\pi_{q}(i)}(\rvx_{q,j}^{i})}_{p}\\
		&= \gamma \frac{1}{S}\norm{\sum_{i=0}^{v(n)-1}\sum_{j=0}^{K-1}\left( \nabla f_{\pi_{q}(i)}(\rvx_{q,j}^{i}) - \nabla f_{\pi_{q}(i)}(\rvx_{q}) + \nabla f_{\pi_{q}(i)}(\rvx_{q}) - \nabla f(\rvx_{q}) + \nabla f(\rvx_{q}) \right)}_{p}\\
		&\leq \gamma \frac{1}{S}\norm{\sum_{i=0}^{v(n)-1}\sum_{j=0}^{K-1}\left( \nabla f_{\pi_{q}(i)}(\rvx_{q,j}^{i}) - \nabla f_{\pi_{q}(i)}(\rvx_{q})\right)}_{p} + \gamma \frac{1}{S}\norm{\sum_{i=0}^{v(n)-1}\sum_{j=0}^{K-1}\left( \nabla f_{\pi_{q}(i)}(\rvx_{q}) - \nabla f(\rvx_{q}) \right)}_{p} \\
		&\quad+ \gamma \frac{1}{S}\norm{\sum_{i=0}^{v(n)-1}\sum_{j=0}^{K-1}\left( \nabla f(\rvx_{q}) \right)}_{p}\\
		&\leq \gamma \frac{1}{S}\sum_{i=0}^{v(n)-1}\sum_{j=0}^{K-1}\norm{ \nabla f_{\pi_{q}(i)}(\rvx_{q,j}^{i}) - \nabla f_{\pi_{q}(i)}(\rvx_{q})}_{p} + \gamma K \frac{1}{S}\varphi_q^{v(n)}+ \gamma \frac{1}{S}\sum_{i=0}^{v(n)-1}\sum_{j=0}^{K-1}\norm{\nabla f(\rvx_{q})}_{p}\\
		&\leq \gamma L_p\frac{1}{S}\sum_{i=0}^{v(n)-1}\sum_{j=0}^{K-1}\norm{ \rvx_{q,j}^{i} - \rvx_{q}}_{p} + \gamma K \frac{1}{S}\varphi_q^{v(n)}+ \gamma  K\left( v(n) \right)\frac{1}{S}\norm{\nabla f(\rvx_{q})}_{p}\\
		&\leq \gamma L_p K \left( v(n) \right) \frac{1}{S} \Delta_q + \gamma K \frac{1}{S}\bar \varphi_q + \gamma  K\left( v(n) \right)\frac{1}{S}\norm{\nabla f(\rvx_{q})}_{p}.
	\end{align*}
	Next, we return to the upper bound of $\norm{\rvx_{q,k}^{n} - \rvx_{q}}_{p}$ for any $n, k$ such that $nK + k \leq NK$. If $k=0$, then $v(n)\leq N$ and the first term on the right hand in Ineq.~\eqref{eq:pf-thm-FL:param drift} equals zero, so we get
	\begin{align*}
		\norm{\rvx_{q,k}^{n} - \rvx_{q}}_{p} \leq \gamma L_p KN \frac{1}{S} \Delta_q + \gamma K \frac{1}{S}\bar \varphi_q + \gamma  KN\frac{1}{S}\norm{\nabla f(\rvx_{q})}_{p}.
	\end{align*}
	If $k>0$, then $v(n)\leq N-S$, so we get
	\begin{align*}
		\norm{\rvx_{q,k}^{n} - \rvx_{q}}_{p} &\leq \gamma L_{p}K\Delta_q + \gamma K\varsigma +\gamma K \norm{\nabla f(\rvx_{q})}_{p} + \gamma L_p K \left( v(n) \right) \frac{1}{S} \Delta_q + \gamma K \frac{1}{S}\bar \varphi_q + \gamma K\left( v(n) \right)\frac{1}{S}\norm{\nabla f(\rvx_{q})}_{p}\\
		&\leq \gamma L_{p}K\Delta_q + \gamma K\varsigma +\gamma K \norm{\nabla f(\rvx_{q})}_{p} + \gamma L_p K\left( N-S \right) \frac{1}{S} \Delta_q + \gamma K \frac{1}{S}\bar \varphi_q + \gamma  K\left( N-S \right)\frac{1}{S}\norm{\nabla f(\rvx_{q})}_{p}\\
		&\leq \gamma L_{p}K\Delta_q + \gamma K\varsigma +\gamma K \norm{\nabla f(\rvx_{q})}_{p} + \gamma L_p K\left( \frac{N}{S}-1 \right) \Delta_q + \gamma K \frac{1}{S}\bar \varphi_q + \gamma  K\left( \frac{N}{S}-1 \right)\norm{\nabla f(\rvx_{q})}_{p}\\
		&\leq \gamma L_{p}KN \frac{1}{S}\Delta_q +\gamma K \frac{1}{S}\bar \varphi_q + \gamma K N\frac{1}{S}\norm{\nabla f(\rvx_{q})}_{p} + \gamma K\varsigma \,.
	\end{align*}
	Therefore, for any $n, k$ such that $nK + k \leq NK$, we get
	\begin{align*}
		\Delta_q = \max_{n,k} \norm{\rvx_{q,k}^{n} - \rvx_{q}}_{p} \leq \gamma L_{p}KN \frac{1}{S}\Delta_q +\gamma K \frac{1}{S}\bar \varphi_q + \gamma K N\frac{1}{S}\norm{\nabla f(\rvx_{q})}_{p} + \gamma K\varsigma \,.
	\end{align*}
	Then, if $\gamma L_p KN \frac{1}{S} \leq \frac{1}{32}$, we get
	\begin{align*}
		&\Delta_q  \leq \gamma L_{p}KN \frac{1}{S}\Delta_q +\gamma K \frac{1}{S}\bar \varphi_q + \gamma K N\frac{1}{S}\norm{\nabla f(\rvx_{q})}_{p} + \gamma K\varsigma \\
		\implies &\left( 1- \gamma L_{p}KN \frac{1}{S}\right)\Delta_q \leq \gamma K \frac{1}{S}\bar \varphi_q + \gamma K N\frac{1}{S}\norm{\nabla f(\rvx_{q})}_{p} + \gamma K\varsigma \\
		\implies &\Delta_q \leq \frac{32}{31} \gamma K \frac{1}{S}\bar \varphi_q +  \frac{32}{31}\gamma K N\frac{1}{S}\norm{\nabla f(\rvx_{q})}_{p} +  \frac{32}{31}\gamma K\varsigma\,.
	\end{align*}
	It also implies that
	\begin{align*}
		\left( \Delta_q \right)^2\leq 4\gamma^2 K^2 \frac{1}{S^2}\left(\bar \varphi_q\right)^2 + 4\gamma^2 K^2 N^2\frac{1}{S^2}\normsq{\nabla f(\rvx_{q})}_{p} + 4\gamma^2 K^2\varsigma^2 .
	\end{align*}
\end{proof}

\subsection{Proof of Theorem~\ref{thm:FL}}

\begin{proof}[Proof of Theorem~\ref{thm:FL}]
	For FL with FL-AP (Algorithm~\ref{alg:FL}), the cumulative updates over any epoch $q$ are
	\begin{align*}
		\rvx_{q+1} - \rvx_{q} = - \gamma \eta \frac{1}{S}\sum_{n=0}^{N-1} \sum_{k=0}^{K-1} \nabla f_{\pi_q(n)}\left (\rvx_{q,k}^{n} \right).
	\end{align*}
	Since the global objective function $f$ is $L$-smooth, it follows that
	\begin{align*}
		f(\rvx_{q+1}) - f(\rvx_{q}) \leq \innerprod{\nabla f(\rvx_q), \rvx_{q+1} - \rvx_q} + \frac{1}{2}L \normsq{\rvx_{q+1} - \rvx_q}.
	\end{align*}
	After substituting $\rvx_{q+1} - \rvx_q$, we have	
	\begin{align*}
		&\innerprod{\nabla f(\rvx_q), \rvx_{q+1} - \rvx_q} \\
		&= -\gamma\eta \frac{1}{S} KN \left[\innerprod{\nabla f(\rvx_q), \frac{1}{N}\sum_{n=0}^{N-1}\frac{1}{K}\sum_{k=0}^{K-1} \nabla f_{\pi_q(n)}(\rvx_{q,k}^{n}) }\right]\\
		&= -\frac{1}{2}\gamma\eta \frac{1}{S} KN  \left[ \normsq{\nabla f(\rvx_q)} + \normsq{\frac{1}{N}\sum_{n=0}^{N-1}\frac{1}{K}\sum_{k=0}^{K-1} \nabla f_{\pi_q(n)}(\rvx_{q,k}^{n})} -  \normsq{\frac{1}{N}\sum_{n=0}^{N-1}\frac{1}{K}\sum_{k=0}^{K-1} \nabla f_{\pi_q(n)}(\rvx_{q,k}^{n}) - \nabla f(\rvx_q)} \right],
	\end{align*}
	where the second equality is due to $2 \langle \rvx,\rvy\rangle =  \norm{\rvx} + \norm{\rvy} - \norm{\rvx-\rvy}$.
	\begin{align*}
		\frac{1}{2}L \E\normsq{\rvx_{q+1} - \rvx_q}&= \frac{1}{2}L\normsq{ \gamma \eta \frac{1}{S}\sum_{n=0}^{N-1} \sum_{k=0}^{K-1} \nabla f_{\pi_q(n)}(\rvx_{q,k}^{n})}\\
		&= \frac{1}{2}\gamma^2\eta^2L\frac{1}{S^2} K^2 N^2 \normsq{\frac{1}{N}\sum_{n=0}^{N-1} \frac{1}{K}\sum_{k=0}^{K-1} \nabla f_{\pi_q(n)}(\rvx_{q,k}^{n})}.
	\end{align*}
	Next, plugging back, and using $\textcolor{blue}{\gamma \eta L  K N \frac{1}{S}\leq 1}$, we get
	\begin{align*}
		f(\rvx_{q+1}) - f(\rvx_{q})
		&\leq -\frac{1}{2}\gamma\eta \frac{1}{S} KN\normsq{\nabla f(\rvx_q)} + \frac{1}{2}\gamma\eta \frac{1}{S} KN \normsq{\frac{1}{N}\sum_{n=0}^{N-1}\frac{1}{K}\sum_{k=0}^{K-1} \nabla f_{\pi_q(n)}(\rvx_{q,k}^{n}) - \nabla f(\rvx_q)}\\
		&\leq -\frac{1}{2}\gamma\eta \frac{1}{S} KN \normsq{\nabla f(\rvx_q)} + \frac{1}{2}\gamma\eta L_{2,p}^2 \frac{1}{S}\sum_{n=0}^{N-1}\sum_{k=0}^{K-1}\normsq{ \rvx_{q,k}^{n} - \rvx_q}_{p}\\
		&\leq -\frac{1}{2}\gamma\eta \frac{1}{S} KN \normsq{\nabla f(\rvx_q)} + \frac{1}{2}\gamma\eta L_{2,p}^2 KN\frac{1}{S}\left( \Delta_q \right)^2,
	\end{align*}
	where the second inequality is because $f_{\pi_q(n)}$ is $L_{2,p}$ smooth for all $n$. Substituting $\Delta_q$ with Lemma~\ref{lem:FL:parameter drift} gives
	\begin{align*}
		f(\rvx_{q+1}) - f(\rvx_{q})
		&\leq -\frac{1}{2}\gamma\eta KN\frac{1}{S} \normsq{\nabla f(\rvx_q)} + \frac{1}{2}\gamma\eta L_{2,p}^2 KN\frac{1}{S}\left( \Delta_q \right)^2\\
		&\leq -\frac{1}{2}\gamma\eta KN\frac{1}{S} \normsq{\nabla f(\rvx_q)} + 2\gamma\eta L_{2,p}^2 KN\frac{1}{S}\left( \gamma^2 K^2 \frac{1}{S^2}\left(\bar \varphi_q\right)^2 + \gamma^2 K^2 N^2\frac{1}{S^2}\normsq{\nabla f(\rvx_{q})}_{p} + \gamma^2 K^2\varsigma^2 \right) \\
		&\leq -\gamma \eta KN\frac{1}{S}\left( \frac{1}{2} - 2\gamma^2L_{2,p}^2K^2N^2\frac{1}{S^2}\right)\normsq{\nabla f(\rvx_q)} + 2\gamma^3\eta L_{2,p}^2 K^3N\frac{1}{S^3}\left(\bar \varphi_q\right)^2  + 2\gamma^3\eta L_{2,p}^2 K^3N\frac{1}{S} \varsigma^2\\
		&\leq -\frac{255}{512}\gamma \eta KN\frac{1}{S}\normsq{\nabla f(\rvx_q)} + 2\gamma^3\eta L_{2,p}^2 K^3N\frac{1}{S^3}\left(\bar \varphi_q\right)^2  + 2\gamma^3\eta L_{2,p}^2 K^3N\frac{1}{S} \varsigma^2,
	\end{align*}
	where the last inequality is due to $\textcolor{blue}{\gamma L_{2,p} KN \frac{1}{S} \leq \frac{1}{32}}$, and $\norm{\rvx}_p \leq \norm{\rvx}$ for $p\geq 2$. Then,
	\begin{align*}
		&f(\rvx_{q+1}) - f(\rvx_q) \leq  -\frac{255}{512}\gamma \eta KN\frac{1}{S}\normsq{\nabla f(\rvx_q)} + 2\gamma^3\eta L_{2,p}^2 K^3N\frac{1}{S^3}\left(\bar \varphi_q\right)^2  + 2\gamma^3\eta L_{2,p}^2 K^3N\frac{1}{S} \varsigma^2\\
		\implies &\frac{1}{Q}\sum_{q=0}^{Q-1} \left( f(\rvx_{q+1}) - f(\rvx_q) \right) \leq -\frac{255}{512}\gamma \eta KN\frac{1}{S}\frac{1}{Q}\sum_{q=0}^{Q-1}\normsq{\nabla f(\rvx_q)} + 2\gamma^3\eta L_{2,p}^2 K^3N\frac{1}{S^3} \frac{1}{Q}\sum_{q=0}^{Q-1}\left(\bar \varphi_q\right)^2  + 2\gamma^3\eta L_{2,p}^2 K^3N\frac{1}{S} \varsigma^2\\
		\implies &\frac{1}{\gamma \eta KN \frac{1}{S}Q} \left( f(\rvx_{Q}) - f(\rvx_0) \right) \leq -\frac{255}{512} \frac{1}{Q}\sum_{q=0}^{Q-1}\normsq{\nabla f(\rvx_q)} + 2\gamma^2 L_{2,p}^2 K^2\frac{1}{S^2}\frac{1}{Q}\sum_{q=0}^{Q-1} \left(\bar \varphi_{q}\right)^2 + 2\gamma^2 L_{2,p}^2 K^2 \varsigma^2\,.
	\end{align*}

	Then, we use Assumption~\ref{asm:FL:order error}. The following steps are identical to those in Theorem~\ref{thm:FL}. Recall that $A_i= 0$ and $B_i=0$ for $i> \nu$ in this theorem. We can write it as
	\begin{align*}
		\left( \bar \varphi_{q}\right)^2
		&\leq A_{1} \left(\bar \varphi_{q-1}\right)^2 + A_{2} \left(\bar \varphi_{q-2}\right)^2 + \cdots + A_{\nu} \left(\bar \varphi_{q-\nu}\right)^2\\
		&\quad+ B_{0} \normsq{\nabla f(\rvx_q)} + B_{1} \normsq{\nabla f(\rvx_{q-1})}_{p} + \cdots + B_{\nu} \normsq{\nabla f(\rvx_{q-\nu})}_{p} + D\,.
	\end{align*}
	Then,
	\begin{align*}
		&\left( \bar \varphi_{q}\right)^2
		\leq A_{1} \left(\bar \varphi_{q-1}\right)^2 + A_{2} \left(\bar \varphi_{q-2}\right)^2 + \cdots + A_{\nu} \left(\bar \varphi_{q-\nu}\right)^2\\
		&\qquad\qquad+ B_{0} \normsq{\nabla f(\rvx_q)} + B_{1} \normsq{\nabla f(\rvx_{q-1})} + \cdots + B_{\nu} \normsq{\nabla f(\rvx_{q-\nu})} + D\\
		\implies
		&\sum_{q=\nu}^{Q-1}\left( \bar \varphi_{q}\right)^2
		\leq A_{1} \sum_{q=\nu}^{Q-1}\left(\bar \varphi_{q-1}\right)^2 + A_{2} \sum_{q=\nu}^{Q-1}\left(\bar \varphi_{q-2}\right)^2 + \cdots + A_{\nu} \sum_{q=\nu}^{Q-1}\left(\bar \varphi_{q-\nu}\right)^2\\
		&\qquad\qquad+ B_{0} \sum_{q=\nu}^{Q-1}\normsq{\nabla f(\rvx_q)} + B_{1} \sum_{q=\nu}^{Q-1}\normsq{\nabla f(\rvx_{q-1})} + \cdots + B_{\nu} \sum_{q=\nu}^{Q-1}\normsq{\nabla f(\rvx_{q-\nu})} + \sum_{q=\nu}^{Q-1} D\\
		\implies
		&\sum_{q=0}^{Q-1}\left( \bar \varphi_{q}\right)^2
		\leq \sum_{i=0}^{\nu-1}\left( \bar \varphi_{i}\right)^2 + A_{1} \sum_{q=0}^{Q-1}\left(\bar \varphi_q\right)^2 + A_{2} \sum_{q=0}^{Q-1}\left(\bar \varphi_{q}\right)^2 + \cdots + A_{\nu} \sum_{q=0}^{Q-1}\left(\bar \varphi_{q}\right)^2\\
		&\qquad\qquad+ B_{0} \sum_{q=0}^{Q-1}\normsq{\nabla f(\rvx_q)} + B_{1} \sum_{q=0}^{Q-1}\normsq{\nabla f(\rvx_{q})} + \cdots + B_{\nu} \sum_{q=0}^{Q-1}\normsq{\nabla f(\rvx_{q})} + \sum_{q=0}^{Q-1} D\\
		\implies
		&\left( 1- \sum_{i=1}^{\nu} A_i \right)\frac{1}{Q}\sum_{q=0}^{Q-1}\left( \bar \varphi_{q}\right)^2 \leq \frac{1}{Q}\sum_{i=0}^{\nu-1}\left( \bar \varphi_{i}\right)^2 + \left( \sum_{i=0}^{\nu} B_i \right)\frac{1}{Q}\sum_{q=0}^{Q-1}\normsq{\nabla f(\rvx_{q})} + D\,.
	\end{align*}
	Then, we get
	\begin{align*}
		\frac{ f(\rvx_{Q}) - f(\rvx_0) }{\gamma \eta KN\frac{1}{S}Q}  &\leq -\frac{255}{512} \frac{1}{Q}\sum_{q=0}^{Q-1}\normsq{\nabla f(\rvx_q)} + 2\gamma^2 L_{2,p}^2 K^2\frac{1}{S^2}\frac{1}{Q}\sum_{q=0}^{Q-1} \left(\bar \varphi_{q}\right)^2 + 2\gamma^2 L_{2,p}^2 K^2 \varsigma^2\\
		&\leq -\frac{255}{512} \frac{1}{Q}\sum_{q=0}^{Q-1}\normsq{\nabla f(\rvx_q)} + 2\gamma^2 L_{2,p}^2 K^2 \varsigma^2 \\
		&\quad+ \frac{2\gamma^2 L_{2,p}^2 K^2\frac{1}{S^2}}{\left( 1- \sum_{i=1}^{\nu} A_i \right)} \left( \frac{1}{Q}\sum_{i=0}^{\nu-1}\left( \bar \varphi_{i}\right)^2  + \left( \sum_{i=0}^{\nu} B_i \right)\frac{1}{Q}\sum_{q=0}^{Q-1}\normsq{\nabla f(\rvx_{q})} + D \right).
	\end{align*}
	To ensue that $\frac{255}{512} - \frac{2\gamma^2 L_{2,p}^2K^2 \frac{1}{S^2}\sum_{i=0}^{\nu} B_i }{1-\sum_{i=1}^{\nu} A_i} >0$, considering that $\textcolor{blue}{\gamma L_{2,p} KN \frac{1}{S} \leq \frac{1}{32}}$, we can use a stricter condition $ \frac{255}{512} - \frac{\sum_{i=0}^{\nu} B_i }{512N^2\left(1-\sum_{i=1}^{\nu} A_i\right)} >0$. Thus, if $\frac{255}{512} - \frac{\sum_{i=0}^{\nu} B_i }{512N^2\left(1-\sum_{i=1}^{\nu} A_i\right)} >0$,
	\begin{align*}
		\frac{1}{Q}\sum_{q=0}^{Q-1}\normsq{\nabla f(\rvx_q)} \leq c_1 \cdot \frac{ f(\rvx_{0}) - f(\rvx_Q) }{\gamma \eta KN\frac{1}{S}Q} + c_2 \cdot \gamma^2 L_{2,p}^2 K^2\frac{1}{S^2}  \frac{1}{Q}\sum_{i=0}^{\nu-1}\left( \bar \varphi_{i}\right)^2 + 2c_1\cdot \gamma^2 L_{2,p}^2 K^2 \varsigma^2 +  c_2\cdot\gamma^2 L_{2,p}^2 K^2\frac{1}{S^2} D.
	\end{align*}
	where $c_1$ and $c_2$ are numerical constants such that $c_1 \geq \nicefrac{1}{\left( \frac{255}{512} - \frac{\sum_{i=0}^{\nu} B_i }{512N^2\left(1-\sum_{i=1}^{\nu} A_i\right)}\right)}$ and $c_2 \geq \left(\frac{2}{1-\sum_{i=1}^{\nu} A_i}\right)\cdot c_1$. Let $F_0 = f(\rvx_0) -f_\ast$.
	\begin{align*}
		\min_{q\in \{0,1,\ldots,Q-1\}} \normsq{\nabla f(\rvx_q)} &\leq \frac{1}{Q}\sum_{q=0}^{Q-1}\normsq{\nabla f(\rvx_q)} \\
		&\leq c_1 \cdot \frac{ F_0 }{\gamma \eta KN\frac{1}{S}Q} + c_2 \cdot \gamma^2 L_{2,p}^2 K^2\frac{1}{S^2}  \frac{1}{Q}\sum_{i=0}^{\nu-1}\left( \bar \varphi_{i}\right)^2 + 2c_1\cdot \gamma^2 L_{2,p}^2 K^2 \varsigma^2 +  c_2\cdot\gamma^2 L_{2,p}^2 K^2\frac{1}{S^2} D,
	\end{align*}
	where the last inequality is due to $f(\rvx_{0}) - f(\rvx_Q) \leq f(\rvx_0) -f_\ast = F_0 $.
	
	At last, we summarize the constraints on the step sizes $\gamma$ and $\eta$ (they are marked in blue),
	\begin{align*}
		\gamma L_{2,p} KN \frac{1}{S} \leq \frac{1}{32},\\
		\gamma \eta L  K N \frac{1}{S}\leq 1,\\
		\gamma L_p KN \frac{1}{S} \leq \frac{1}{32}.
	\end{align*}
	Thus, a tighter constraint $\gamma \leq \min \left\{ \frac{1}{32L_{2,p} KN\frac{1}{S}}, \frac{1}{\eta L KN\frac{1}{S}}, \frac{1}{32L_p KN\frac{1}{S}} \right\}$ is used in Theorem~\ref{thm:FL}.
\end{proof}

\section{Special Cases in FL}
\label{apx-sec:FL:special cases}

In this section, we provide proofs of the examples of FL.

\subsection{FL-AP}

\begin{example}[FL-AP]
	\label{ex:FL-AP}
	For FL-AP, all the permutations $\{\pi_q\}$ in Algorithm~\ref{alg:FL} are generated arbitrarily. Under Assumption~\ref{asm:local deviation}, Assumption~\ref{asm:FL:order error} holds as
	\begin{align*}
		\left(\bar\varphi_{q}\right)^2 \leq N^2\varsigma^2\,.
	\end{align*}
	Applying Theorem~\ref{thm:FL}, we get
	\begin{align*}
		\min_{q\in \{0,1\ldots, Q-1\}}\normsq{ \nabla f(\rvx_q) } 
		&= \gO\left(  \frac{ F_0 }{\gamma \eta KN\frac{1}{S}Q} +\gamma^2 L^2 K^2 \varsigma^2+ \gamma^2 L^2K^2N^2\frac{1}{S^2}\varsigma^2 \right).
	\end{align*}
	If we set $\eta=1$ and tune the step size, the upper bound becomes $\gO\left( \frac{L F_0}{Q} + \left(\frac{LF_0 S\varsigma }{ NQ}\right)^{\frac{2}{3}} + \left(\frac{LF_0 N \varsigma }{ NQ}\right)^{\frac{2}{3}}  \right)$.
\end{example}

\begin{proof}
	For any epoch $q$,
	\begin{align*}
		\left(\bar\varphi_{q}\right)^2 = \max_{n \in [N]}\norm{\sum_{i=0}^{v(n)-1} \left( \nabla f_{\pi_{q}(i)}(\rvx_{q}) - \nabla f(\rvx_{q})  \right) }^2 \leq N^2\varsigma^2.
	\end{align*}
	
	In this case, for Assumption~\ref{asm:FL:order error}, $p=2$, $A_1 =\cdots=A_q=0$, $B_0 =B_1=\cdots=B_q=0$ and $D = N^2\varsigma^2$. Then, we verify that
	\begin{align*}
		\frac{255}{512} - \frac{1}{512N^2}\cdot \frac{\sum_{i=0}^{\nu} B_i }{1-\sum_{i=1}^{\nu} A_i} = \frac{255}{512} >0\,.
	\end{align*}
	Thus, we can set $c_1 = 3$ and $c_2 = 6$ for Theorem~\ref{thm:FL}. In addition, for Theorem~\ref{thm:FL}, $\nu=0$. These lead to the upper bound,
	\begin{align*}
		\min_{q\in \{0,1\ldots, Q-1\}}\normsq{ \nabla f(\rvx_q) } 
		&= \gO\left(  \frac{ F_0 }{\gamma \eta KN\frac{1}{S}Q} +\gamma^2 L^2 K^2 \varsigma^2+ \gamma^2 L^2K^2N^2\frac{1}{S^2}\varsigma^2 \right),
	\end{align*}
	where $F_0 = f(\rvx_0) - f_\ast$ and $L = L_{2,p} = L_{p}$ when $p=2$.
	
	Next, we summarize the constraints on the step size:
	\begin{align*}
		\gamma &\leq \min \left\{ \frac{1}{\eta L KN\frac{1}{S}}, \frac{1}{32L_{2,p} KN\frac{1}{S}}, \frac{1}{32L_p KN\frac{1}{S}} \right\}.
	\end{align*}
	It is from Theorem~\ref{thm:FL}. After we use the effective step size $\tilde \gamma \coloneqq \gamma \eta K N \frac{1}{S} $, the constraint becomes
	\begin{align*}
		\tilde \gamma &\leq \min \left\{ \frac{1}{ L }, \frac{\eta}{32L_{2,p}}, \frac{\eta}{32L_p} \right\},
	\end{align*}
	and the upper bound becomes
	\begin{align*}
		\min_{q\in \{0,1\ldots, Q-1\}}\normsq{ \nabla f(\rvx_q) } 
		&= \gO\left(  \frac{ F_0 }{\tilde\gamma Q} +\tilde\gamma^2 L^2 \frac{1}{\eta^2N^2\frac{1}{S^2}} \varsigma^2+ \tilde\gamma^2 L^2 \frac{1}{\eta^2N^2}N^2\varsigma^2 \right).
	\end{align*}
	Applying Lemma~\ref{lem:step-size}, we get
	\begin{align*}
		\min_{q\in \{0,1\ldots, Q-1\}}\normsq{ \nabla f(\rvx_q) } &=
		\gO\left( \frac{\left( 1+\eta\right)L F_0}{\eta Q} + \left(\frac{LF_0 S\varsigma }{\eta NQ}\right)^{\frac{2}{3}} + \left(\frac{LF_0 N \varsigma }{\eta NQ}\right)^{\frac{2}{3}}  \right).
	\end{align*} 
	For comparison with other algorithms, we set $\eta = 1$, and get
	\begin{align*}
		\min_{q\in \{0,1\ldots, Q-1\}}\normsq{ \nabla f(\rvx_q) } =
		\gO\left( \frac{L F_0}{Q} + \left(\frac{LF_0 S\varsigma }{ NQ}\right)^{\frac{2}{3}} + \left(\frac{LF_0 N \varsigma }{ NQ}\right)^{\frac{2}{3}}  \right).
	\end{align*}
\end{proof}

\subsection{FL-RR}

\begin{example}[FL-RR]
	\label{ex:FL-RR}
	For FL-RR, all the permutations $\{\pi_q\}$ in Algorithm~\ref{alg:FL} are generated independently and randomly. Under Assumption~\ref{asm:local deviation}, Assumption~\ref{asm:FL:order error} holds with probability at least $1-\delta$:
	\begin{align*}
		\left(\bar\varphi_{q}\right)^2 &= \max_{n \in [N]} \norm{\sum_{i=0}^{v(n)-1} \left( \nabla f_{\pi_{q}(i)}(\rvx_{q}) - \nabla f(\rvx_{q})  \right) }^2 \leq 4N \varsigma^2 \log^2 \left( \frac{8}{\delta}\right).
	\end{align*}
	Applying Theorem~\ref{thm:FL}, we get that, with probability at least $1-Q\delta$,
	\begin{align*}
		\min_{q\in \{0,1\ldots, Q-1\}}\normsq{ \nabla f(\rvx_q) } 
		&= \gO\left(  \frac{ F_0 }{\gamma \eta KN\frac{1}{S}Q} +\gamma^2 L^2 K^2 \varsigma^2+ \gamma^2 L^2K^2N\frac{1}{S^2}\varsigma^2\log^2 \left( \frac{8}{\delta}\right)  \right).
	\end{align*}
	If we set $\eta=1$ and tune the step size, the upper bound becomes $\tilde\gO\left( \frac{L F_0}{Q} + \left(\frac{LF_0 S\varsigma }{ NQ}\right)^{\frac{2}{3}} + \left(\frac{LF_0 \sqrt{N} \varsigma }{ NQ}\right)^{\frac{2}{3}}  \right)$.
\end{example}

\begin{proof}[Proof of Example~\ref{ex:FL-RR}]
	Since the permutations $\{\pi_q\}$ are independent across different epochs, for any $q$, we get that, with probability at least $1-\delta$,
	\begin{align*}
		\left(\bar\varphi_{q}\right)^2 &= \max_{n \in [N]} \norm{\sum_{i=0}^{v(n)-1} \left( \nabla f_{\pi_{q}(i)}(\rvx_{q}) - \nabla f(\rvx_{q})  \right) }^2 \leq 4N \varsigma^2 \log^2 \left( \frac{8}{\delta}\right),
	\end{align*}
	where the last inequality is due to \citet{yu2023high}'s Proposition~2.3.
	
	In this case, for Assumption~\ref{asm:FL:order error}, $p=2$, $A_1 =\cdots=A_q=0$, $B_0 =B_1=\cdots=B_q=0$ and $D = 4N \varsigma^2 \log^2 \left( \frac{8}{\delta}\right)$. Then, we verify that
	\begin{align*}
		\frac{255}{512} - \frac{1}{512N^2}\cdot \frac{\sum_{i=0}^{\nu} B_i }{1-\sum_{i=1}^{\nu} A_i} = \frac{255}{512} >0.
	\end{align*}
	Thus, we can set $c_1 = 3$ and $c_2 = 6$ for Theorem~\ref{thm:FL}. In addition, for Theorem~\ref{thm:FL}, $\nu=0$. These lead to the upper bound,
	\begin{align*}
		\min_{q\in \{0,1\ldots, Q-1\}}\normsq{ \nabla f(\rvx_q) } 
		&= \gO\left(  \frac{ F_0 }{\gamma \eta KN\frac{1}{S}Q} +\gamma^2 L^2 K^2 \varsigma^2+ \gamma^2 L^2K^2N\frac{1}{S^2}\varsigma^2\log^2 \left( \frac{8}{\delta}\right)  \right),
	\end{align*}
	where $F_0 = f(\rvx_0) - f_\ast$ and $L = L_{2,p} = L_{p}$ when $p=2$. Since \citet{yu2023high}'s Proposition~2.3 is used for each epoch (that is, for $Q$ times), so by the union bound, the preceding bound holds with probability at least $1-Q\delta$.
	
	Next, we summarize the constraints on the step size:
	\begin{align*}
		\gamma &\leq \min \left\{ \frac{1}{\eta L KN\frac{1}{S}}, \frac{1}{32L_{2,p} KN\frac{1}{S}}, \frac{1}{32L_p KN\frac{1}{S}} \right\}.
	\end{align*}
	It is from Theorem~\ref{thm:FL}. After we use the effective step size $\tilde \gamma \coloneqq \gamma \eta K N \frac{1}{S} $, the constraint becomes	\begin{align*}
		\tilde \gamma &\leq \min \left\{ \frac{1}{ L }, \frac{\eta}{32L_{2,p}}, \frac{\eta}{32L_p} \right\},
	\end{align*}
	and the upper bound becomes
	\begin{align*}
		\min_{q\in \{0,1\ldots, Q-1\}}\normsq{ \nabla f(\rvx_q) } 
		&= \tilde \gO\left(  \frac{ F_0 }{\tilde\gamma Q} +\tilde\gamma^2 L^2 \frac{1}{\eta^2N^2\frac{1}{S^2}} \varsigma^2+ \tilde\gamma^2 L^2 \frac{1}{\eta^2N^2}N\varsigma^2 \right).
	\end{align*}
	Applying Lemma~\ref{lem:step-size}, we get
	\begin{align*}
		\min_{q\in \{0,1\ldots, Q-1\}}\normsq{ \nabla f(\rvx_q) } &=
		\tilde \gO\left( \frac{\left( 1+\eta\right)L F_0}{\eta Q} + \left(\frac{LF_0 S\varsigma }{\eta NQ}\right)^{\frac{2}{3}} + \left(\frac{LF_0 \sqrt{N} \varsigma }{\eta NQ}\right)^{\frac{2}{3}}  \right).
	\end{align*} 
	For comparison with other algorithms, we set $\eta = 1$, and get
	\begin{align*}
		\min_{q\in \{0,1\ldots, Q-1\}}\normsq{ \nabla f(\rvx_q) } =
		\tilde\gO\left( \frac{L F_0}{Q} + \left(\frac{LF_0 S\varsigma }{ NQ}\right)^{\frac{2}{3}} + \left(\frac{LF_0 \sqrt{N} \varsigma }{ NQ}\right)^{\frac{2}{3}}  \right).
	\end{align*}
\end{proof}

\subsection{FL-OP}

\begin{example}[FL-OP]
	\label{ex:FL-OP}
	For FL-OP, in Algorithm~\ref{alg:FL}, the first-epoch permutation $\pi_0$ is generated arbitrarily/randomly/elaborately; the subsequent permutations are the same as the first-epoch permutation: $\pi_q = \pi_0$ for any $q\geq 1$.
	Let $\{f_n\}$ be $L$-smooth and Assumptions~\ref{asm:local deviation},~\ref{asm:parameter deviation} hold. Then, Assumption~\ref{asm:FL:order error} holds as
	\begin{align*}
		\left(\bar\varphi_{q}\right)^2 \leq 8 L^2 N^2 \theta^2 + 2 \left( \bar\varphi_0 \right)^2 \,.
	\end{align*}
	Applying Theorem~\ref{thm:FL}, we get
	\begin{align*}
		\min_{q\in \{0,1\ldots, Q-1\}}\normsq{ \nabla f(\rvx_q) } 
		&= \gO\left(  \frac{ F_0 }{\gamma \eta KN\frac{1}{S}Q} +\gamma^2 L^2 K^2 \varsigma^2+ \gamma^2 L^2K^2\frac{1}{S^2}\left( \bar\varphi_0 \right)^2+\gamma^2 L^4K^2N^2\frac{1}{S^2} \theta^2 \right).
	\end{align*}
	If we set $\eta =1$ and tune the step size, then the upper bound becomes $\gO\left( \frac{L F_0}{Q} + \left(\frac{LF_0 S\varsigma }{ NQ}\right)^{\frac{2}{3}} + \left(\frac{LF_0 \bar\varphi_0+L^2F_0 N\theta }{ NQ}\right)^{\frac{2}{3}}  \right)$. Furthermore, if $\theta \lesssim \frac{\bar \varphi_0}{LN}$, it becomes $\gO\left( \frac{L F_0}{Q} + \left(\frac{LF_0 S\varsigma }{ NQ}\right)^{\frac{2}{3}} + \left(\frac{LF_0 \bar\varphi_0 }{ NQ}\right)^{\frac{2}{3}}  \right)$.
	\begin{itemize}
		\item If the initial permutation is arbitrary (it implies $\bar\varphi_0 = \gO\left( N \varsigma \right)$), then the bound will be $\gO\left( \frac{L F_0}{Q} + \left(\frac{LF_0 S\varsigma }{ NQ}\right)^{\frac{2}{3}} + \left(\frac{LF_0 N\sigma }{ NQ}\right)^{\frac{2}{3}}  \right)$.
		\item If the initial permutation is random (it implies $\bar\varphi_0 = \tilde\gO( \sqrt{N} \varsigma )$), then the bound will be $\tilde\gO\left( \frac{L F_0}{Q} + \left(\frac{LF_0 S\varsigma }{ NQ}\right)^{\frac{2}{3}} + \left(\frac{LF_0 \sqrt{N}\sigma }{ NQ}\right)^{\frac{2}{3}}  \right)$. It holds with probability at least $1-\delta$.
		\item If the initial permutation is produced meticulously (it implies $\bar\varphi_0 = \tilde\gO\left( \varsigma \right)$), then the bound will be $\tilde \gO\left( \frac{L F_0}{Q} + \left(\frac{LF_0 S\varsigma }{ NQ}\right)^{\frac{2}{3}} + \left(\frac{LF_0 \varsigma }{ NQ}\right)^{\frac{2}{3}}  \right)$.
	\end{itemize}
\end{example}

\begin{proof}[Proof of Example~\ref{ex:FL-OP}]
	We replace the notation $v(n)$ for $n\in [N]$ with $m$ for $m\in \{S,2S, \ldots, N\}$ to avoid ambiguity. For any $q\geq 1$ and $m\in \{S,2S, \ldots, N\}$,
	\begin{align*}
		\varphi_{q}^m &= \norm{ \sum_{i=0}^{m-1} \left( \nabla f_{\pi_{q}(i)}(\rvx_{q}) - \nabla f(\rvx_{q})  \right) }  \\
		&= \norm{ \sum_{i=0}^{m-1} \left( \nabla f_{\pi_{q}(i)}(\rvx_{q}) - \nabla f(\rvx_{q})  \right) - \left( \nabla f_{\pi_{q}(i)}(\rvx_{0}) - \nabla f(\rvx_{0})  \right) + \left( \nabla f_{\pi_{q}(i)}(\rvx_{0}) - \nabla f(\rvx_{0})  \right) }\\
		&\leq \norm{ \sum_{i=0}^{m-1} \left( \nabla f_{\pi_{q}(i)}(\rvx_{q}) - \nabla f_{\pi_{q}(i)}(\rvx_{0})  \right)  } + \norm{ \sum_{i=0}^{m-1} \left( \nabla f(\rvx_{q}) - \nabla f(\rvx_{0}) \right) } + \norm{\sum_{i=0}^{m-1} \left( \nabla f_{\pi_{q}(i)}(\rvx_{0}) - \nabla f(\rvx_{0})  \right) }\\
		&\leq \sum_{i=0}^{m-1}\norm{ \nabla f_{\pi_{q}(i)}(\rvx_{q}) - \nabla f_{\pi_{q}(i)}(\rvx_{0}) } + \sum_{i=0}^{m-1}\norm{ \nabla f(\rvx_{q}) - \nabla f(\rvx_{0}) } + \norm{\sum_{i=0}^{m-1} \left( \nabla f_{\pi_{q}(i)}(\rvx_{0}) - \nabla f(\rvx_{0})  \right) }\\
		&\leq 2Lm\norm{\rvx_{q} - \rvx_{0}} + \norm{\sum_{i=0}^{m-1} \left( \nabla f_{\pi_{q}(i)}(\rvx_{0}) - \nabla f(\rvx_{0})  \right) }\\
		&\leq 2LN\theta + \bar \varphi_0\,,
	\end{align*}
	where we use the fact that the permutations are exactly the same, $\pi_q = \pi_0$ for $q\geq 1$ in this case. Since the preceding inequality holds for all $m\in \{S,2S, \ldots, N\}$, we have
	\begin{align*}
		\bar\varphi_{q} \leq 2LN \theta + \bar\varphi_0
		\implies & \left( \bar\varphi_{q}\right)^2 \leq 2 \cdot \left(2LN \theta\right)^2 + 2\cdot \left( \bar\varphi_0 \right)^2 = 8 L^2 N^2 \theta^2 + 2 \left( \bar\varphi_0 \right)^2
	\end{align*}
	
	In this case, for Assumption~\ref{asm:FL:order error}, $p=2$, $A_1 =\cdots=A_q=0$, $B_0 =B_1=\cdots=B_q=0$ and $D = 8 L^2 N^2 \theta^2 + 2 \left( \bar\varphi_0 \right)^2$. Then, we verify that
	\begin{align*}
		\frac{255}{512} - \frac{1}{512N^2}\cdot \frac{\sum_{i=0}^{\nu} B_i }{1-\sum_{i=1}^{\nu} A_i} = \frac{255}{512} >0\,.
	\end{align*}
	Thus, we can set $c_1 = 3$ and $c_2 = 6$ for Theorem~\ref{thm:FL}. In addition, for Theorem~\ref{thm:FL}, $\nu=0$. These lead to the upper bound
	\begin{align*}
		\min_{q\in \{0,1\ldots, Q-1\}}\normsq{ \nabla f(\rvx_q) } 
		&= \gO\left(  \frac{ F_0 }{\gamma \eta KN\frac{1}{S}Q} +\gamma^2 L^2 K^2 \varsigma^2+ \gamma^2 L^2K^2\frac{1}{S^2}\left( \bar\varphi_0 \right)^2+\gamma^2 L^4K^2N^2\frac{1}{S^2} \theta^2 \right),
	\end{align*}
	where $F_0 = f(\rvx_0) - f_\ast$ and $L = L_{2,p} = L_{p}$ when $p=2$.
	
	Next, we summarize the constraints on the step size:
	\begin{align*}
		\gamma &\leq \min \left\{ \frac{1}{\eta L KN\frac{1}{S}}, \frac{1}{32L_{2,p} KN\frac{1}{S}}, \frac{1}{32L_p KN\frac{1}{S}} \right\}.
	\end{align*}
	It is from Theorem~\ref{thm:FL}. After we use the effective step size $\tilde \gamma \coloneqq \gamma \eta K N \frac{1}{S} $, the constraint becomes
	\begin{align*}
		\tilde \gamma &\leq \min \left\{ \frac{1}{ L }, \frac{\eta}{32L_{2,p}}, \frac{\eta}{32L_p} \right\},
	\end{align*}
	and the upper bound becomes
	\begin{align*}
		\min_{q\in \{0,1\ldots, Q-1\}}\normsq{ \nabla f(\rvx_q) } 
		&= \gO\left(  \frac{ F_0 }{\tilde\gamma Q} +\tilde\gamma^2 L^2 \frac{1}{\eta^2N^2\frac{1}{S^2}} \varsigma^2+ \tilde\gamma^2 L^2 \frac{1}{\eta^2N^2}\left( \varphi_0\right)^2 + \tilde\gamma^2 L^4 \frac{1}{\eta^2N^2}N^2\theta^2 \right).
	\end{align*}
	Applying Lemma~\ref{lem:step-size}, we get
	\begin{align*}
		\min_{q\in \{0,1\ldots, Q-1\}}\normsq{ \nabla f(\rvx_q) } &=
		\gO\left( \frac{\left( 1+\eta\right)L F_0}{\eta Q} + \left(\frac{LF_0 S\varsigma }{\eta NQ}\right)^{\frac{2}{3}} + \left(\frac{LF_0 \bar\varphi_0 + L^2F_0 N\theta}{\eta NQ}\right)^{\frac{2}{3}} \right) .
	\end{align*} 
	For comparison with other algorithms, we set $\eta = 1$, and get
	\begin{align*}
		\min_{q\in \{0,1\ldots, Q-1\}}\normsq{ \nabla f(\rvx_q) } =
		\gO\left( \frac{L F_0}{Q} + \left(\frac{LF_0 S\varsigma }{ NQ}\right)^{\frac{2}{3}} + \left(\frac{LF_0 \bar\varphi_0+L^2F_0 N\theta }{ NQ}\right)^{\frac{2}{3}}  \right).
	\end{align*}
	Furthermore, if $\theta \lesssim \frac{\bar\varphi_0}{LN}$, then
	\begin{align*}
		\min_{q\in \{0,1\ldots, Q-1\}}\normsq{ \nabla f(\rvx_q) } =
		\gO\left( \frac{L F_0}{Q} + \left(\frac{LF_0 S\varsigma }{ NQ}\right)^{\frac{2}{3}} + \left(\frac{LF_0 \bar\varphi_0 }{ NQ}\right)^{\frac{2}{3}}  \right).
	\end{align*}
	
	Next, let us deal with $\bar\varphi_0$, depending on the initial permutation.
	\begin{itemize}
		\item If the initial permutation $\pi_0$ is generated arbitrarily, we get
		\begin{align*}
			\left(\bar\varphi_{0}\right)^2 = \max_{m \in \{S,2S,\ldots, N\}} \norm{\sum_{i=0}^{m-1} \left( \nabla f_{\pi_{0}(i)}(\rvx_{0}) - \nabla f(\rvx_{0})  \right) }^2 \leq \max_{m \in \{S,2S,\ldots, N\}} \left( m^2 \varsigma^2 \right) = N^2\varsigma^2\,.
		\end{align*}
		Then,
		\begin{align*}
			\min_{q\in \{0,1\ldots, Q-1\}}\normsq{ \nabla f(\rvx_q) } =
			\gO\left( \frac{L F_0}{Q} + \left(\frac{LF_0 S\varsigma }{ NQ}\right)^{\frac{2}{3}} + \left(\frac{LF_0 N\varsigma }{ NQ}\right)^{\frac{2}{3}}  \right).
		\end{align*}
		\item If the initial permutation $\pi_0$ is generated randomly, we get that with probability at least $1-\delta$,
		\begin{align*}
			\left(\bar\varphi_{0}\right)^2 = \max_{m \in \{S,2S,\ldots, N\}} \norm{\sum_{i=0}^{m-1} \left( \nabla f_{\pi_{0}(i)}(\rvx_{0}) - \nabla f(\rvx_{0})  \right) }^2 \leq 4N \varsigma^2 \log^2 \left( \frac{8}{\delta}\right),
		\end{align*}
		where the last inequality is due to \citet{yu2023high}'s Proposition~2.3.
		Then,
		\begin{align*}
			\min_{q\in \{0,1\ldots, Q-1\}}\normsq{ \nabla f(\rvx_q) } =
			\tilde \gO\left( \frac{L F_0}{Q} + \left(\frac{LF_0 S\varsigma }{ NQ}\right)^{\frac{2}{3}} + \left(\frac{LF_0 \sqrt{N}\varsigma }{ NQ}\right)^{\frac{2}{3}}  \right).
		\end{align*}
		It holds with probability at least $1-\delta$, because \citet{yu2023high}'s Proposition~2.3 is only used for the initial epoch.
		\item If the initial permutation $\pi_0$ is a nice permutation such that $\bar\varphi_0 = \tilde \gO\left( \varsigma^2 \right)$,
		\begin{align*}
			\min_{q\in \{0,1\ldots, Q-1\}}\normsq{ \nabla f(\rvx_q) } =
			\tilde\gO\left( \frac{L F_0}{Q} + \left(\frac{LF_0 S\varsigma }{ NQ}\right)^{\frac{2}{3}} + \left(\frac{LF_0 \varsigma }{ NQ}\right)^{\frac{2}{3}}  \right).
		\end{align*}
		In fact, we can generate such a nice permutation by GraBs \citep[Section 6. Ablation Study: are good permutations fixed?]{lu2022grab}.
	\end{itemize}
\end{proof}

\subsection{Prototype of FL-GraB}
\label{apx-subsec:balancing-prototype}

Prototype of FL-GraB: Use \texttt{PairBR} (Algorithm~\ref{alg:PairBR}) as the \texttt{Permute} function in Algorithm~\ref{alg:FL}, with the inputs of $\pi_q$, $\{ \nabla f_{\pi_q(n)}(\rvx_q) \}_{n=0}^{N-1}$ and $\nabla f(\rvx_q)$, for each epoch $q$.

Thus, the key idea of our proof is as follows:
\begin{align*}
	\bar \varphi_{q+1} &\rightarrow \max_{m \in \{S,2S,\ldots,N\}} \norm{\sum_{i=0}^{m-1} \left( \nabla f_{\pi_{q+1}(i)}(\rvx_{q}) - \nabla f(\rvx_{q})  \right) }_{\infty}  \overset{\text{Lemma~\ref{lem:FL:pair-balancing-reordering}}}{\rightarrow}
	\bar \varphi_{q}\,.
\end{align*}

\begin{example}[Prototype of FL-GraB]
	\label{ex:FL-GraB-proto}
	Let $\{f_n\}$ be $L_{\infty}$-smooth and Assumption~\ref{asm:local deviation} hold. Assume that $N \mod S =0$ and $S\mod 2=0$. Then, if $\gamma \leq \frac{1}{32 \eta L_{\infty} KN \frac{1}{S}}$, Assumption~\ref{asm:FL:order error} holds with probability at least $1-\delta$:
	\begin{align*}
		\left(\bar \varphi_{q}\right)^2 \leq \frac{3}{4}\left( \bar \varphi_{q-1} \right)^2 + \frac{1}{40} N^2 \normsq{\nabla f(\rvx_{q-1})}_{} +\frac{1}{40}S^2 \varsigma^2 + 6 C^2 \varsigma^2\,,
	\end{align*}
	where $C = \gO \left(\log\left(\frac{d N}{\delta}\right) \right) = \tilde\gO\left( 1\right)$. Applying Theorem~\ref{thm:FL} (with a tighter constraint $\gamma \leq \min \left\{ \frac{1}{\eta L KN\frac{1}{S}}, \frac{1}{32L_{2,\infty} KN\frac{1}{S}}, \frac{1}{32(1+\eta)L_\infty KN\frac{1}{S}} \right\}$), we get that, with probability at least $1-Q\delta$,
	\begin{align*}
		\min_{q\in \{0,1\ldots, Q-1\}}\normsq{ \nabla f(\rvx_q) } 
		&= \tilde\gO\left(  \frac{ F_0 }{\tilde\gamma Q} + \tilde\gamma^2 L_{2,\infty}^2 \frac{1}{\eta^2}\frac{1}{Q} \varsigma^2 +\tilde\gamma^2 L_{2,\infty}^2 \frac{1}{\eta^2N^2\frac{1}{S^2}} \varsigma^2+ \tilde\gamma^2 L_{2,\infty}^2 \frac{1}{\eta^2N^2}C^2\varsigma^2 \right).
	\end{align*}
	If we set $\eta =1$ and tune the step size, the upper bound becomes $\gO\left( \frac{\tilde L  F_0+\left (L_{2,\infty}F_0 \varsigma \right)^{\frac{2}{3}}}{Q} + \left(\frac{L_{2,\infty}F_0 S\varsigma }{ NQ}\right)^{\frac{2}{3}} + \left(\frac{L_{2,\infty}F_0 C \varsigma }{ NQ}\right)^{\frac{2}{3}} \right)$, where $\tilde L = L+L_{2,\infty}+L_{\infty}$.
\end{example}

\begin{proof}
	We need to find the relation between $\varphi_{q+1}$ and $\varphi_{q}$. We replace the notation $v(n)$ for $n\in [N]$ with $m$ for $m\in \{S,2S, \ldots, N\}$ to avoid ambiguity. Furthermore, unless otherwise stated, the notation $\max_m$ means $\max_{m\{S,2S,\ldots, N\}}$. For any $m\in \{S,2S, \ldots, N\}$,
	\begin{align*}
		\varphi_{q+1}^m &= \norm{ \sum_{i=0}^{m-1} \left( \nabla f_{\pi_{q+1}(i)}(\rvx_{q+1}) - \nabla f(\rvx_{q+1})  \right) }_{\infty}  \\
		&= \norm{ \sum_{i=0}^{m-1} \left( \nabla f_{\pi_{q+1}(i)}(\rvx_{q+1}) - \nabla f(\rvx_{q+1})  \right) - \left( \nabla f_{\pi_{q+1}(i)}(\rvx_{q}) - \nabla f(\rvx_{q})  \right) + \left( \nabla f_{\pi_{q+1}(i)}(\rvx_{q}) - \nabla f(\rvx_{q})  \right) }_{\infty}\\
		&\leq \norm{ \sum_{i=0}^{m-1} \left( \nabla f_{\pi_{q+1}(i)}(\rvx_{q+1}) - \nabla f_{\pi_{q+1}(i)}(\rvx_{q})  \right)  }_{\infty} + \norm{ \sum_{i=0}^{m-1} \left( \nabla f(\rvx_{q+1}) - \nabla f(\rvx_{q}) \right) }_{\infty} + \norm{\sum_{i=0}^{m-1} \left( \nabla f_{\pi_{q+1}(i)}(\rvx_{q}) - \nabla f(\rvx_{q})  \right) }_{\infty}\\
		&\leq \sum_{i=0}^{m-1}\norm{ \nabla f_{\pi_{q+1}(i)}(\rvx_{q+1}) - \nabla f_{\pi_{q+1}(i)}(\rvx_{q}) }_{\infty} + \sum_{i=0}^{m-1}\norm{ \nabla f(\rvx_{q+1}) - \nabla f(\rvx_{q}) }_{\infty} + \norm{\sum_{i=0}^{m-1} \left( \nabla f_{\pi_{q+1}(i)}(\rvx_{q}) - \nabla f(\rvx_{q})  \right) }_{\infty}\\
		&\leq 2L_{\infty}m\norm{\rvx_{q+1} - \rvx_{q}}_{\infty} + \norm{\sum_{i=0}^{m-1} \left( \nabla f_{\pi_{q+1}(i)}(\rvx_{q}) - \nabla f(\rvx_{q})  \right) }_{\infty}.
	\end{align*}
	Since the above inequality holds for all $m \in \{S,2S, \ldots, N\}$, we have
	\begin{align*}
		\bar \varphi_{q+1} \leq 2L_{\infty}N\norm{\rvx_{q+1} - \rvx_{q}}_{\infty} + \max_{m \in \{S,2S,\ldots,N\}} \norm{\sum_{i=0}^{m-1} \left( \nabla f_{\pi_{q+1}(i)}(\rvx_{q}) - \nabla f(\rvx_{q})  \right) }_{\infty}.
	\end{align*}
	In this example, since
	\begin{align*}
		&\norm{\nabla f_i (\rvx_q) - f(\rvx_q)} \leq \varsigma\,,\quad \forall i \in \{0,1,\ldots,N-1\},\\
		&\norm{ \sum_{i=0}^{N-1} \left( \nabla f_i (\rvx_q) - f(\rvx_q)\right) }_{\infty} = 0\,,
	\end{align*}
	we apply Lemma~\ref{lem:FL:pair-balancing-reordering} with $a = \varsigma$ and $b = 0$,
	\begin{align*}
		\max_{m \in \{S,2S,\ldots,N\}} \norm{\sum_{i=0}^{m-1} \left( \nabla f_{\pi_{q+1}(i)}(\rvx_{q}) - \nabla f(\rvx_{q})  \right) } &\leq \frac{1}{2} \max_{m \in \{S,2S,\ldots,N\}} \norm{\sum_{i=0}^{m-1} \left( \nabla f_{\pi_{q}(i)}(\rvx_{q}) - \nabla f(\rvx_{q})  \right) } +  C \varsigma\\
		&= \frac{1}{2} \bar \varphi_q + C\varsigma\,.
	\end{align*}
	Using Lemma~\ref{lem:FL:parameter drift} that $\Delta_q \leq \frac{32}{31} \gamma K \frac{1}{S}\bar \varphi_q + \frac{32}{31}\gamma K N\frac{1}{S}\norm{\nabla f(\rvx_{q})}_{\infty} + \frac{32}{31}\gamma K\varsigma $, we get	
	\begin{align*}
		\bar \varphi_{q+1} &\leq 2L_{\infty}N\norm{\rvx_{q+1} - \rvx_{q}}_{\infty} + \max_{m \in \{S,2S,\ldots,N\}}  \norm{\sum_{i=0}^{m-1} \left( \nabla f_{\pi_{q+1}(i)}(\rvx_{q}) - \nabla f(\rvx_{q})  \right) }_{\infty}\\
		&\leq 2\eta L_{\infty}N \Delta_q + \max_{m \in \{S,2S,\ldots,N\}}  \norm{\sum_{i=0}^{m-1} \left( \nabla f_{\pi_{q+1}(i)}(\rvx_{q}) - \nabla f(\rvx_{q})  \right) }_{\infty}\\
		&\leq \frac{64}{31}\eta L_{\infty}N  \left(\gamma K \frac{1}{S}\bar \varphi_q + \gamma K N\frac{1}{S}\norm{\nabla f(\rvx_{q})}_{\infty} + \gamma K\varsigma \right) + \left( \frac{1}{2} \bar \varphi_q + C\varsigma\right)\\
		&\leq \left(\frac{1}{2}+\frac{64}{31}\gamma \eta L_{\infty} KN\frac{1}{S}\right) \bar\varphi_q + \frac{64}{31} \gamma \eta L_{\infty}KN^2\frac{1}{S} \norm{\nabla f(\rvx_q)}_{\infty} + \frac{64}{31}\gamma \eta L_{\infty} K N \varsigma + C\varsigma\\
		&\leq \frac{35}{62} \bar\varphi_q + \frac{2}{31} N \norm{\nabla f(\rvx_q)}_{\infty} + \frac{2}{31}S \varsigma + C\varsigma\,,
	\end{align*}
	where the last inequality is due to $\textcolor{blue}{\gamma \eta L_{\infty} KN \frac{1}{S} \leq \frac{1}{32}}$. Then, using $\norm{\rvx}_p \leq \norm{\rvx}_2$ for $p>2$, we can get
	\begin{align*}
		\left(\bar \varphi_{q+1}\right)^2
		&\leq \left( \frac{35}{62}\bar \varphi_q + \frac{2}{31}N \norm{\nabla f(\rvx_q)}_{} + \frac{2}{31}S \varsigma + C\varsigma \right)^2 \\
		&\leq 2\cdot \left( \frac{35}{62} \bar \varphi_q \right)^2 + 6 \cdot \left( \frac{2}{31}N \norm{\nabla f(\rvx_q)}_{} \right)^2 + 6 \cdot \left( \frac{2}{31}S \varsigma \right)^2 + 6 \cdot \left( C\varsigma \right)^2\\
		&\leq \frac{3}{4}\left( \bar \varphi_q \right)^2 + \frac{1}{40} N^2 \normsq{\nabla f(\rvx_q)}_{} +\frac{1}{40}S^2 \varsigma^2 + 6 C^2 \varsigma^2\,.
	\end{align*}
	So the relation between $\bar \varphi_q$ and $\bar \varphi_{q-1}$ is 
	\begin{align*}
		\left(\bar \varphi_{q}\right)^2 \leq \frac{3}{4}\left( \bar \varphi_{q-1} \right)^2 + \frac{1}{40} N^2 \normsq{\nabla f(\rvx_{q-1})}_{} +\frac{1}{40}S^2 \varsigma^2 + 6 C^2 \varsigma^2.
	\end{align*}
	for $q\geq 1$. Besides, we need to get the bound of $\left(\bar\varphi_{0}\right)^2$:
	\begin{align*}
		\left(\bar\varphi_{0}\right)^2 = \max_{m \in \{S,2S,\ldots,N\} } \norm{\sum_{i=0}^{m-1} \left( \nabla f_{\pi_{0}(i)}(\rvx_{0}) - \nabla f(\rvx_{0})  \right) }^2 \leq N^2 \varsigma^2.
	\end{align*}

	In this case, for Assumption~\ref{asm:FL:order error}, $p=\infty$, $A = \left[\frac{3}{4},0, \ldots,0\right]$, $B = \left[0,\frac{1}{40}N^2,0,\ldots, 0\right]$ and $D = \frac{1}{40}S^2\varsigma^2 + 6C^2 \varsigma^2$. Then, we verify that
	\begin{align*}
		\frac{255}{512} - \frac{1}{512N^2}\cdot \frac{\sum_{i=0}^{\nu} B_i }{1-\sum_{i=1}^{\nu} A_i} \geq \frac{255}{512} - 2\cdot 4 \cdot \left( \frac{1}{32}\right)^2 \cdot \frac{1}{40} \geq \frac{254}{512}>0\,.
	\end{align*}
	Thus, we can set $c_1 = 3$ and $c_2 = 24$ for Theorem~\ref{thm:FL}. In addition, for Theorem~\ref{thm:FL}, $\nu=1$ and $\left(\bar\varphi_{0}\right)^2 \leq N^2 \varsigma^2$. These lead to the upper bound,
	\begin{align*}
		\min_{q\in \{0,1\ldots, Q-1\}}\normsq{ \nabla f(\rvx_q) } 
		&= \gO\left(  \frac{ F_0 }{\gamma \eta KN\frac{1}{S}Q} + \gamma^2 L_{2,\infty}^2 K^2N^2\frac{1}{S^2}\frac{1}{Q} \varsigma^2 +\gamma^2 L_{2,\infty}^2 K^2 \varsigma^2+ \gamma^2 L_{2,\infty}^2K^2C^2\frac{1}{S^2}\varsigma^2 \right).
	\end{align*}
	where $F_0 = f(\rvx_0) - f_\ast$. Since Lemma~\ref{lem:FL:pair-balancing-reordering} is used for each epoch (that is, for $Q$ times), so by the union bound, the preceding bound holds with probability at least $1-Q\delta$.
	
	Next, we summarize the constraints on the step size:
	\begin{align*}
		\gamma &\leq \min \left\{ \frac{1}{\eta L KN\frac{1}{S}}, \frac{1}{32L_{2,\infty} KN\frac{1}{S}}, \frac{1}{32L_\infty KN\frac{1}{S}} \right\},\\
		\gamma  &\leq \frac{1}{32 \eta L_{\infty} KN \frac{1}{S}}.
	\end{align*}
	The first one is from Theorem~\ref{thm:FL} and the others are from the derivation of the relation. Then, a tighter constraint will be
	\begin{align*}
		\gamma &\leq \min \left\{ \frac{1}{\eta L KN\frac{1}{S}}, \frac{1}{32L_{2,\infty} KN\frac{1}{S}}, \frac{1}{32(1+\eta)L_\infty KN\frac{1}{S}} \right\}.
	\end{align*}
	After we use the effective step size $\tilde \gamma \coloneqq \gamma \eta K N \frac{1}{S} $, the constraint becomes
	\begin{align*}
		\gamma &\leq \min \left\{ \frac{1}{ L }, \frac{\eta}{32L_{2,\infty} }, \frac{\eta}{32(1+\eta)L_\infty } \right\},
	\end{align*}
	and the upper bound becomes
	\begin{align*}
		\min_{q\in \{0,1\ldots, Q-1\}}\normsq{ \nabla f(\rvx_q) } 
		&= \gO\left(  \frac{ F_0 }{\tilde\gamma Q} + \tilde\gamma^2 L_{2,\infty}^2 \frac{1}{\eta^2}\frac{1}{Q} \varsigma^2 +\tilde\gamma^2 L_{2,\infty}^2 \frac{1}{\eta^2N^2\frac{1}{S^2}} \varsigma^2+ \tilde\gamma^2 L_{2,\infty}^2 \frac{1}{\eta^2N^2}C^2\varsigma^2 \right).
	\end{align*}
	Applying Lemma~\ref{lem:step-size}, we get
	\begin{align*}
		&\min_{q\in \{0,1\ldots, Q-1\}}\normsq{ \nabla f(\rvx_q) } \\&=
		\gO\left( \frac{\left( \eta L+L_{2,\infty} + (1+\eta)L_{\infty}\right)  F_0}{\eta Q}  + \frac{ (L_{2,\infty}F_0 \varsigma)^{\frac{2}{3}} }{\eta^{\frac{2}{3}}Q} + \left(\frac{L_{2,\infty}F_0 S\varsigma }{\eta NQ}\right)^{\frac{2}{3}} + \left(\frac{L_{2,\infty}F_0 C \varsigma }{\eta NQ}\right)^{\frac{2}{3}}  \right).
	\end{align*} 
	For comparison with other algorithms, we set $\eta = 1$, and get
	\begin{align*}
		&\min_{q\in \{0,1\ldots, Q-1\}}\normsq{ \nabla f(\rvx_q) } \\&=
		\gO\left( \frac{\left( L +L_{2,\infty}+ L_{\infty}\right)  F_0}{Q}  + \frac{ (L_{2,\infty}F_0 \varsigma)^{\frac{2}{3}} }{Q} + \left(\frac{L_{2,\infty}F_0 S\varsigma }{ NQ}\right)^{\frac{2}{3}} + \left(\frac{L_{2,\infty}F_0 C \varsigma }{ NQ}\right)^{\frac{2}{3}}  \right).
	\end{align*}
\end{proof}

\subsection{FL-GraB}
\label{apx-subsec:balancing}

\begin{example}[FL-GraB]
	\label{ex:FL-GraB}
	Let each $f_n$ be $L_{2,\infty}$-smooth and $L_{\infty}$-smooth, and Assumption~\ref{asm:local deviation} hold. Assume that $N \mod S =0$ and $S\mod 2=0$. Then, if $\gamma \leq \min\{ \frac{1}{128L_{2, \infty}K C\frac{1}{S}}, \frac{1}{128(1+\eta) L_{\infty}K N\frac{1}{S}} \}$, Assumption~\ref{asm:FL:order error} holds with probability at least $1-\delta$:
	\begin{align*}
			\left(\bar \varphi_{q}\right)^2 \leq \frac{3}{5}\left( \bar \varphi_{q-1} \right)^2 + \frac{1}{96}N^2 \normsq{\nabla f(\rvx_{q-1})}_{} + \frac{1}{96}S^2\varsigma^2 + 6C^2 \varsigma^2,
		\end{align*}
	where $C = \gO \left(\log\left(\frac{d N}{\delta}\right) \right) = \tilde\gO\left( 1\right)$. Applying Theorem~\ref{thm:FL} (with a tighter constraint $\gamma \leq \min \left\{ \frac{1}{\eta L KN\frac{1}{S}}, \frac{1}{128L_{2,\infty} K(N+C)\frac{1}{S}}, \frac{1}{128(1+\eta)L_\infty KN\frac{1}{S}} \right\}$), we get that, with probability at least $1-Q\delta$,
	\begin{align*}
			\min_{q\in \{0,1\ldots, Q-1\}} \normsq{ \nabla f(\rvx_q) } 
			&= \gO\left(  \frac{ F_0 }{\gamma \eta KN\frac{1}{S}Q} + \gamma^2 L_{2,\infty}^2 K^2N^2\frac{1}{S^2}\frac{1}{Q} \varsigma^2 +\gamma^2 L_{2,\infty}^2 K^2 \varsigma^2+ \gamma^2 L_{2,\infty}^2K^2C^2\frac{1}{S^2}\varsigma^2 \right).
		\end{align*}
	After we set $\eta =1$ and tune the step size, the upper bound becomes $\gO\left( \frac{\tilde L  F_0+\left (L_{2,\infty}F_0 \varsigma \right)^{\frac{2}{3}}}{Q} + \left(\frac{L_{2,\infty}F_0 S\varsigma }{ NQ}\right)^{\frac{2}{3}} + \left(\frac{L_{2,\infty}F_0 C \varsigma }{ NQ}\right)^{\frac{2}{3}} \right)$ where $\tilde L = L+L_{2,\infty}\left(1+\frac{C}{N}\right)+L_{\infty}$.
\end{example}

FL-GraB. Use \texttt{PairBR} (Algorithm~\ref{alg:PairBR}) as the \texttt{Permute} function in Algorithm~\ref{alg:FL}, with the inputs of $\pi_q$, $\{ \rvp_q^n \}_{n=0}^{N-1}$ and $\frac{1}{N}\sum_{n=0}^{N-1}\rvp_q^n$, for each epoch $q$.

Thus, the key idea of our proof is as follows:
\begin{align*}
	\bar \varphi_{q+1} &\rightarrow \max_{m \in \{S,2S,\ldots, N\}} \norm{ \sum_{i=0}^{m-1} \left( \sum_{j=0}^{K-1} \nabla f_{\pi_q(i)}\left(\rvx_{q,j}^i \right) - \frac{1}{N} \sum_{l=0}^{N-1} \sum_{j=0}^{K-1} \nabla f_{\pi_{q}(l)} \left( \rvx_{q,j}^l\right) \right) }_{\infty}   \\
	&\overset{\text{Lemma~\ref{lem:FL:pair-balancing-reordering}}}{\rightarrow} \max_{m \in \{S,2S,\ldots, N\}}\norm{ \sum_{i=0}^{m-1} \left( \sum_{j=0}^{K-1} \nabla f_{\pi_{q+1}(i)}\left(\rvx_{q,j}^{\pi_q^{-1} (\pi_{q+1}(i))} \right) - \frac{1}{N} \sum_{l=0}^{N-1} \sum_{j=0}^{K-1} \nabla f_{\pi_{q}(l)} \left( \rvx_{q,j}^l\right) \right) }_{\infty}  \rightarrow
	\bar \varphi_{q}\,.
\end{align*}

\begin{proof}
	We need to find the relation between $\bar \varphi_{q+1}$ and $\varphi_q$. For all $m\in \{S,2S, \ldots, N\}$,
	\begin{align}
		&\varphi_{q+1}^{m} \nonumber\\
		&= \norm{ \sum_{i=0}^{m-1} \left( \nabla f_{\pi_{q+1}(i)}\left(\rvx_{q+1} \right) - \nabla f\left(\rvx_{q+1}\right)\right) }_{\infty}\nonumber\\
		&= \frac{1}{K}\norm{ \sum_{i=0}^{m-1} \sum_{j=0}^{K-1}\left( \nabla f_{\pi_{q+1}(i)}\left(\rvx_{q+1} \right) - \nabla f\left(\rvx_{q+1}\right)\right) }_{\infty}\nonumber\\
		&= \frac{1}{K}\norm{ \sum_{i=0}^{m-1} \sum_{j=0}^{K-1}\left( \nabla f_{\pi_{q+1}(i)}\left(\rvx_{q+1} \right) - \nabla f\left(\rvx_{q+1}\right)\right)\pm \sum_{i=0}^{m-1} \left( \sum_{j=0}^{K-1} \nabla f_{\pi_{q+1}(i)}\left(\rvx_{q,j}^{\pi_q^{-1} (\pi_{q+1}(i))} \right) - \frac{1}{N} \sum_{l=0}^{N-1} \sum_{j=0}^{K-1} \nabla f_{\pi_{q}(l)} \left( \rvx_{q,j}^l\right) \right) }_{\infty}\nonumber\\
		&\leq \frac{1}{K}\norm{ \sum_{i=0}^{m-1} \sum_{j=0}^{K-1} \nabla f_{\pi_{q+1}(i)}\left(\rvx_{q+1} \right) - \sum_{i=0}^{m-1} \sum_{j=0}^{K-1} \nabla f_{\pi_{q+1}(i)}\left(\rvx_{q,j}^{\pi_q^{-1} (\pi_{q+1}(i))} \right) }_{\infty} \nonumber\\
		&\quad+ \frac{1}{K}\norm{ \sum_{i=0}^{m-1} \sum_{j=0}^{K-1}\frac{1}{N} \sum_{l=0}^{N-1} \nabla f_{\pi_{q}(l)}\left(\rvx_{q+1}\right) - \sum_{i=0}^{m-1}\frac{1}{N} \sum_{l=0}^{N-1} \sum_{j=0}^{K-1} \nabla f_{\pi_{q}(l)} \left( \rvx_{q,j}^l\right) }_{\infty}\nonumber\\
		&\quad+ \frac{1}{K}\norm{ \sum_{i=0}^{m-1} \left( \sum_{j=0}^{K-1} \nabla f_{\pi_{q+1}(i)}\left(\rvx_{q,j}^{\pi_q^{-1} (\pi_{q+1}(i))} \right) - \frac{1}{N} \sum_{l=0}^{N-1} \sum_{j=0}^{K-1} \nabla f_{\pi_{q}(l)} \left( \rvx_{q,j}^l\right) \right) }_{\infty}\label{eq:pf-ex-fed-grab:order error},
	\end{align}
	where the last inequality is due to $\nabla f(\rvx_{q+1}) = \frac{1}{N}\sum_{l=0}^{N-1} \nabla f_{\pi_q(l)}\left (\rvx_{q+1} \right)$. Then,
	\begin{align*}
		\Term{1}{\eqref{eq:pf-ex-fed-grab:order error}} &= \frac{1}{K}\norm{ \sum_{i=0}^{m-1} \sum_{j=0}^{K-1} \nabla f_{\pi_{q+1}(i)}\left(\rvx_{q+1} \right) - \sum_{i=0}^{m-1} \sum_{j=0}^{K-1} \nabla f_{\pi_{q+1}(i)}\left(\rvx_{q,j}^{\pi_q^{-1} (\pi_{q+1}(i))} \right) }_{\infty} \\
		&\leq \frac{1}{K}\sum_{i=0}^{m-1} \sum_{j=0}^{K-1}\norm{  \nabla f_{\pi_{q+1}(i)}\left(\rvx_{q+1} \right) -  \nabla f_{\pi_{q+1}(i)}\left(\rvx_{q,j}^{\pi_q^{-1} (\pi_{q+1}(i))} \right) }_{\infty} \\
		&\leq L_{\infty}\frac{1}{K}\sum_{i=0}^{m-1} \sum_{j=0}^{K-1}\norm{  \rvx_{q+1} -  \rvx_{q,j}^{\pi_q^{-1} (\pi_{q+1}(i))} }_{\infty} \\
		&\leq L_{\infty}\frac{1}{K}\sum_{i=0}^{m-1} \sum_{j=0}^{K-1}\left( \norm{  \rvx_{q+1} -  \rvx_{q} }_{\infty} + \norm{  \rvx_{q} -  \rvx_{q,j}^{\pi_q^{-1} (\pi_{q+1}(i))} }_{\infty}\right) \\
		&\leq L_{\infty}\frac{1}{K}\sum_{i=0}^{m-1} \sum_{j=0}^{K-1}\left( \eta \Delta_q + \Delta_q\right)\\
		&\leq L_{\infty}N\left( \eta \Delta_q + \Delta_q\right),
	\end{align*}
	\begin{align*}
		\Term{2}{\eqref{eq:pf-ex-fed-grab:order error}} &=\frac{1}{K}\norm{ \sum_{i=0}^{m-1} \sum_{j=0}^{K-1}\frac{1}{N} \sum_{l=0}^{N-1} \nabla f_{\pi_{q}(l)}\left(\rvx_{q+1}\right) - \sum_{i=0}^{m-1}\frac{1}{N} \sum_{l=0}^{N-1} \sum_{j=0}^{K-1} \nabla f_{\pi_{q}(l)} \left( \rvx_{q,j}^l\right) }_{\infty}\\
		&\leq \frac{1}{K}\sum_{i=0}^{m-1} \sum_{j=0}^{K-1}\frac{1}{N} \sum_{l=0}^{N-1} \norm{ \nabla f_{\pi_{q}(l)}\left(\rvx_{q+1}\right) -  \nabla f_{\pi_{q}(l)} \left( \rvx_{q,j}^l\right) }_{\infty}\\
		&\leq L_{\infty}\frac{1}{K}\sum_{i=0}^{m-1} \sum_{j=0}^{K-1}\frac{1}{N} \sum_{l=0}^{N-1} \norm{ \rvx_{q+1} -  \rvx_{q,j}^l }_{\infty}\\
		&\leq L_{\infty}\frac{1}{K}\sum_{i=0}^{m-1} \sum_{j=0}^{K-1}\frac{1}{N} \sum_{l=0}^{N-1} \left( \norm{ \rvx_{q+1} -  \rvx_{q} }_{\infty} + \norm{ \rvx_{q} -  \rvx_{q,j}^l }_{\infty} \right)\\
		&\leq L_{\infty}\frac{1}{K}\sum_{i=0}^{m-1} \sum_{j=0}^{K-1}\frac{1}{N} \sum_{l=0}^{N-1} \left( \eta\Delta_{q} + \Delta_{q} \right)\\
		&\leq L_{\infty}N \left( \eta\Delta_{q} + \Delta_{q} \right).
	\end{align*}
	Since it holds for any $m \in \{S, 2S, \ldots, N\}$, we have
	\begin{align}
		\bar \varphi_{q+1} &\leq 2L_{\infty}N\left( \eta \Delta_q + \Delta_q\right)\nonumber \\
		&\quad+  \frac{1}{K}\max_{m \in \{S, 2S, \ldots, N\}} \norm{ \sum_{i=0}^{m-1} \left( \sum_{j=0}^{K-1} \nabla f_{\pi_{q+1}(i)}\left(\rvx_{q,j}^{\pi_q^{-1} (\pi_{q+1}(i))} \right) - \frac{1}{N} \sum_{l=0}^{N-1} \sum_{j=0}^{K-1} \nabla f_{\pi_{q}(l)} \left( \rvx_{q,j}^l\right) \right) }_{\infty}.\label{eq:pf-ex:FL-balancing:order error^2}
	\end{align}
	
	Note that 
	\begin{align*}
		\sum_{j=0}^{K-1} \nabla f_{\pi_q(i)}\left(\rvx_{q,j}^i \right) - \frac{1}{N} \sum_{l=0}^{N-1} \sum_{j=0}^{K-1} \nabla f_{\pi_{q}(l)} \left( \rvx_{q,j}^l\right)
	\end{align*}
	and
	\begin{align*}
		\sum_{j=0}^{K-1} \nabla f_{\pi_{q+1}(i)}\left(\rvx_{q,j}^{\pi_q^{-1} (\pi_{q+1}(i))} \right) - \frac{1}{N} \sum_{l=0}^{N-1} \sum_{j=0}^{K-1} \nabla f_{\pi_{q}(l)} \left( \rvx_{q,j}^l\right)
	\end{align*}
	correspond to $\rvz_{\pi(i)}$ and $\rvz_{\pi' (i)}$ in Lemma~\ref{lem:FL:pair-balancing-reordering}, respectively. We next get the upper bounds of 
	\begin{align*}
		\norm{\rvz_{\pi(i)}}_2, \norm{\sum_{i=0}^{N-1} \rvz_{\pi(i)}}_{\infty} \text{ and } \max_{n\in [N]} \norm{\sum_{i=0}^{n-1} \rvz_{\pi(i)}}_{\infty},
	\end{align*}
	and then apply Lemma~\ref{lem:FL:pair-balancing-reordering} to the last term on the right hand side in Ineq.~\eqref{eq:pf-ex:FL-balancing:order error^2}.
	\begin{align*}
		&\norm{\rvz_{\pi(i)}}_2 \\
		&= \norm{  \sum_{j=0}^{K-1} \nabla f_{\pi_q(i)}\left(\rvx_{q,j}^i \right) - \frac{1}{N} \sum_{l=0}^{N-1} \sum_{j=0}^{K-1} \nabla f_{\pi_{q}(l)} \left( \rvx_{q,j}^l\right) }_{2}\\
		&= \norm{  \left(\sum_{j=0}^{K-1} \nabla f_{\pi_q(i)}\left(\rvx_{q,j}^i \right) - \frac{1}{N} \sum_{l=0}^{N-1} \sum_{j=0}^{K-1} \nabla f_{\pi_{q}(l)} \left( \rvx_{q,j}^l\right)\right) \pm \left(\sum_{j=0}^{K-1} \nabla f_{\pi_q(i)}\left(\rvx_{q} \right) - \frac{1}{N} \sum_{l=0}^{N-1} \sum_{j=0}^{K-1} \nabla f_{\pi_{q}(l)} \left( \rvx_{q}\right)\right) }_{2}\\
		&\leq \norm{ \sum_{j=0}^{K-1} \nabla f_{\pi_q(i)}\left(\rvx_{q,j}^i \right) - \sum_{j=0}^{K-1} \nabla f_{\pi_{q}(i)} \left( \rvx_{q}\right) }_{2} + \norm{ \frac{1}{N} \sum_{l=0}^{N-1} \sum_{j=0}^{K-1} \nabla f_{\pi_{q}(l)} \left( \rvx_{q,j}^l\right) - \frac{1}{N} \sum_{l=0}^{N-1} \sum_{j=0}^{K-1} \nabla f_{\pi_{q}(l)} \left( \rvx_{q}\right) }_{2} + K \varsigma\\
		&\leq \sum_{j=0}^{K-1}\norm{ \nabla f_{\pi_q(i)}\left(\rvx_{q,j}^i \right) - \nabla f_{\pi_{q}(i)} \left( \rvx_{q}\right) }_{2} + \frac{1}{N} \sum_{l=0}^{N-1} \sum_{j=0}^{K-1}\norm{  \nabla f_{\pi_{q}(l)} \left( \rvx_{q,j}^l\right) - \nabla f_{\pi_{q}(l)} \left( \rvx_{q}\right) }_{2} + K \varsigma\\
		&\leq L_{2,\infty}\sum_{j=0}^{K-1}\norm{ \rvx_{q,j}^i - \rvx_{q} }_{\infty} + L_{2,\infty}\frac{1}{N} \sum_{l=0}^{N-1} \sum_{j=0}^{K-1}\norm{  \rvx_{q,j}^l - \rvx_{q} }_{\infty} + K \varsigma\\
		&\leq L_{2,\infty}\sum_{j=0}^{K-1}\Delta_q + L_{2,\infty}\frac{1}{N} \sum_{l=0}^{N-1} \sum_{j=0}^{K-1} \Delta_{q} + K \varsigma\\
		&\leq 2L_{2,\infty}K\Delta_q + K \varsigma\,,
	\end{align*}
	\begin{align*}
		\norm{ \sum_{i=0}^{N-1} \rvz_{\pi_q(i)} }_{\infty} = \norm{ \sum_{i=0}^{N-1} \left( \sum_{j=0}^{K-1} \nabla f_{\pi_q(i)}\left(\rvx_{q,j}^i \right) - \frac{1}{N} \sum_{l=0}^{N-1} \sum_{j=0}^{K-1} \nabla f_{\pi_{q}(l)} \left( \rvx_{q,j}^l\right) \right) }_{\infty}= 0\,.
	\end{align*}	
	In addition, for any $m \in \{S,2S,\ldots, N\}$, we have
	\begin{align*}
		&\norm{\sum_{i=0}^{m-1} \rvz_{\pi(i)} }_{\infty}\\
		&= \norm{ \sum_{i=0}^{m-1} \left( \sum_{j=0}^{K-1} \nabla f_{\pi_q(i)}\left(\rvx_{q,j}^i \right) - \frac{1}{N} \sum_{l=0}^{N-1} \sum_{j=0}^{K-1} \nabla f_{\pi_{q}(l)} \left( \rvx_{q,j}^l\right) \right) }_{\infty} \\
		&= \norm{ \sum_{i=0}^{m-1} \left( \sum_{j=0}^{K-1} \nabla f_{\pi_q(i)}\left(\rvx_{q,j}^i \right) - \frac{1}{N} \sum_{l=0}^{N-1} \sum_{j=0}^{K-1} \nabla f_{\pi_{q}(l)} \left( \rvx_{q,j}^l\right) \right) \pm \sum_{i=0}^{m-1} \left(\sum_{j=0}^{K-1} \nabla f_{\pi_q(i)}\left(\rvx_{q} \right) - \frac{1}{N} \sum_{l=0}^{N-1} \sum_{j=0}^{K-1} \nabla f_{\pi_{q}(l)} \left( \rvx_{q}\right)\right)}_{\infty} \\
		&\leq \norm{ \sum_{i=0}^{m-1}  \sum_{j=0}^{K-1} \nabla f_{\pi_q(i)}\left(\rvx_{q,j}^i \right) - \sum_{i=0}^{m-1} \sum_{j=0}^{K-1} \nabla f_{\pi_{q}(i)} \left( \rvx_{q}\right)  }_{\infty} \\
		&\quad+ \norm{ \sum_{i=0}^{m-1}  \frac{1}{N} \sum_{l=0}^{N-1} \sum_{j=0}^{K-1} \nabla f_{\pi_{q}(l)} \left( \rvx_{q,j}^l\right) - \sum_{i=0}^{m-1}\frac{1}{N} \sum_{l=0}^{N-1} \sum_{j=0}^{K-1} \nabla f_{\pi_{q}(l)} \left( \rvx_{q}\right)}_{\infty}\\
		&\quad+ \norm{ \sum_{i=0}^{m-1} \left(\sum_{j=0}^{K-1} \nabla f_{\pi_q(i)}\left(\rvx_{q} \right) - \frac{1}{N} \sum_{l=0}^{N-1} \sum_{j=0}^{K-1} \nabla f_{\pi_{q}(l)} \left( \rvx_{q}\right)\right)}_{\infty}\\
		&\leq \sum_{i=0}^{m-1}  \sum_{j=0}^{K-1}\norm{ \nabla f_{\pi_q(i)}\left(\rvx_{q,j}^i \right) - \nabla f_{\pi_{q}(i)} \left( \rvx_{q}\right)  }_{\infty} + \sum_{i=0}^{m-1}  \frac{1}{N} \sum_{l=0}^{N-1} \sum_{j=0}^{K-1}\norm{  \nabla f_{\pi_{q}(l)} \left( \rvx_{q,j}^l\right) - \nabla f_{\pi_{q}(l)} \left( \rvx_{q}\right)}_{\infty} + K \bar \varphi_q\\
		&\leq L_{\infty}\sum_{i=0}^{m-1}  \sum_{j=0}^{K-1}\norm{ \rvx_{q,j}^i - \rvx_{q}  }_{\infty} + L_{\infty}\sum_{i=0}^{m-1}  \frac{1}{N} \sum_{l=0}^{N-1} \sum_{j=0}^{K-1}\norm{  \rvx_{q,j}^l - \rvx_{q}}_{\infty} + K \bar \varphi_q\\
		&\leq L_{\infty}\sum_{i=0}^{m-1}  \sum_{j=0}^{K-1}\Delta_q + L_{\infty}\sum_{i=0}^{m-1}  \frac{1}{N} \sum_{l=0}^{N-1} \sum_{j=0}^{K-1}\Delta_{q} + K \bar \varphi_q\\
		&\leq 2L_{\infty}KN\Delta_q + K \bar \varphi_q\,.
	\end{align*}
	Since it holds for all $m \in \{S,2S,\ldots, N\}$, we have
	\begin{align*}
		\max_{m \in \{S,2S,\ldots, N\}} \norm{\sum_{i=0}^{m-1} \rvz_{\pi(i)}}_{\infty} \leq 2L_{\infty}KN\Delta_q + K \bar \varphi_q\,.
	\end{align*}
	Now, applying Lemma~\ref{lem:FL:pair-balancing-reordering} to the last term on the right hand side in Ineq.~\eqref{eq:pf-ex:FL-balancing:order error^2}, we can get
	\begin{align*}
		\bar \varphi_{q+1} %&\leq 2L_{\infty}N\left( \eta \Delta_q + \Delta_q\right) + \frac{1}{K}\max_{m \in \{S, 2S, \ldots, N\}} \norm{\sum_{i=0}^{m-1} \rvz_{\pi_{q+1}(i)} }_{\infty}\\
		&\leq 2L_{\infty}N\left( \eta \Delta_q + \Delta_q\right) + \frac{1}{2}\left( 2L_{\infty}N\Delta_q +  \bar \varphi_q \right) + C \left( 2L_{2,\infty}\Delta_q + \varsigma\right)\\
		&\leq \left( \left( 3+2\eta\right)L_{\infty} N + 2L_{2,\infty}C \right)\Delta_q + \frac{1}{2}\bar\varphi_q + C \varsigma\,.
	\end{align*}
	Applying Lemma~\ref{lem:FL:parameter drift} that $\Delta_q \leq \frac{32}{31} \left(\gamma K \frac{1}{S}\bar \varphi_q + \gamma K N\frac{1}{S}\norm{\nabla f(\rvx_{q})}_{\infty} + \gamma K\varsigma \right)$, we get
	\begin{align*}
		\bar \varphi_{q+1}
		&\leq \left( \left( 3+2\eta\right)L_{\infty} N + 2L_{2,\infty}C \right) \Delta_q + \frac{1}{2} \bar\varphi_q + C \varsigma\\
		&\leq \left( \left( 3+2\eta\right)L_{\infty} N + 2L_{2,\infty}C \right) \cdot \frac{32}{31} \gamma K\frac{1}{S}\left( \bar \varphi_q + N\norm{\nabla f(\rvx_{q})}_{\infty} + S\varsigma \right) + \frac{1}{2} \bar\varphi_q + C \varsigma\\
		&\leq \frac{13}{24} \bar\varphi_q + \frac{1}{24} N \norm{\nabla f(\rvx_q)}_{\infty} + \frac{1}{24} S\varsigma + C \varsigma\,.
	\end{align*}
	where the last inequality is due to $\textcolor{blue}{\gamma (1+\eta) L_{\infty}K N\frac{1}{S} \leq \frac{1}{128}}$ and $\textcolor{blue}{\gamma L_{2, \infty}K C\frac{1}{S} \leq \frac{1}{128}}$. Then, using $\norm{\rvx}_p \leq \norm{\rvx}_2$ for $p>2$, we can get
	\begin{align*}
		\left(\bar \varphi_{q+1}\right)^2
		&\leq \left( \frac{13}{24} \bar\varphi_q + \frac{1}{24} N \norm{\nabla f(\rvx_q)}_{} + \frac{1}{24} S\varsigma + C \varsigma \right)^2 \\
		&\leq 2\cdot \left( \frac{13}{24} \bar\varphi_q \right)^2 + 6 \cdot \left( \frac{1}{24}N \norm{\nabla f(\rvx_q)}_{} \right)^2+ 6 \cdot \left( \frac{1}{24} S\varsigma \right)^2+ 6 \cdot \left( C\varsigma \right)^2\\
		&\leq \frac{3}{5}\left( \bar \varphi_q \right)^2 + \frac{1}{96}N^2 \normsq{\nabla f(\rvx_q)}_{} + \frac{1}{96}S^2\varsigma^2 + 6C^2 \varsigma^2.
	\end{align*}
	So the relation between $\bar \varphi_q$ and $\bar \varphi_{q-1}$ is 
	\begin{align*}
		\left(\bar \varphi_{q}\right)^2 \leq \frac{3}{5}\left( \bar \varphi_{q-1} \right)^2 + \frac{1}{96}N^2 \normsq{\nabla f(\rvx_{q-1})}_{} + \frac{1}{96}S^2\varsigma^2 + 6C^2 \varsigma^2\,.
	\end{align*}
	for $q\geq 1$. Besides, we have $\left(\bar\varphi_{0}\right)^2 \leq N^2 \varsigma^2$.

	In this case, for Assumption~\ref{asm:FL:order error}, $p=\infty$, $A_1= \frac{3}{5}$, $A_2=\cdots=A_q=0$, $B_0=0$, $B_1= \frac{1}{96}N^2$, $B_2 = \cdots B_q = 0$ and $D = \frac{1}{96}S^2\varsigma^2 + 6C^2 \varsigma^2$. Then, we verify that
	\begin{align*}
		\frac{255}{512} - \frac{1}{512N^2}\cdot \frac{\sum_{i=0}^{\nu} B_i }{1-\sum_{i=1}^{\nu} A_i} \geq \frac{255}{512} - 2\cdot \frac{5}{2} \cdot \left( \frac{1}{128}\right)^2 \cdot \frac{1}{96} \geq \frac{254}{512}>0\,.
	\end{align*}
	Thus, we can set $c_1 = 3$ and $c_2 = 15$ for Theorem~\ref{thm:FL}. In addition, for Theorem~\ref{thm:FL}, $\nu=1$ and $\left(\bar\varphi_{0}\right)^2 \leq N^2 \varsigma^2$. These lead to the upper bound
	\begin{align*}
		\min_{q\in \{0,1\ldots, Q-1\}}\normsq{ \nabla f(\rvx_q) } 
		&= \gO\left(  \frac{ F_0 }{\gamma \eta KN\frac{1}{S}Q} + \gamma^2 L_{2,\infty}^2 K^2N^2\frac{1}{S^2}\frac{1}{Q} \varsigma^2 +\gamma^2 L_{2,\infty}^2 K^2 \varsigma^2+ \gamma^2 L_{2,\infty}^2K^2C^2\frac{1}{S^2}\varsigma^2 \right),
	\end{align*}
	where $F_0 = f(\rvx_0) - f_\ast$. Since Lemma~\ref{lem:FL:pair-balancing-reordering} is used for each epoch (that is, for $Q$ times), so by the union bound, the preceding bound holds with probability at least $1-Q\delta$.
	
	Next, we summarize the constraints on the step size:
	\begin{align*}
		\gamma &\leq \min \left\{ \frac{1}{\eta L KN\frac{1}{S}}, \frac{1}{32L_{2,\infty} KN\frac{1}{S}}, \frac{1}{32L_\infty KN\frac{1}{S}} \right\},\\
		\gamma  &\leq \frac{1}{128(1+\eta) L_{\infty}K N\frac{1}{S}},\\ 
		\gamma  &\leq \frac{1}{128L_{2, \infty}K C\frac{1}{S}}.
	\end{align*}
	The first one is from Theorem~\ref{thm:FL} and the others are from the derivation of the relation. For simplicity, we can use a tighter constraint
	\begin{align*}
		\gamma &\leq \min \left\{ \frac{1}{\eta L KN\frac{1}{S}}, \frac{1}{128L_{2,\infty} K(N+C)\frac{1}{S}}, \frac{1}{128(1+\eta)L_\infty KN\frac{1}{S}} \right\}.
	\end{align*}
	After we use the effective step size $\tilde \gamma \coloneqq \gamma \eta K N \frac{1}{S} $, the constraint becomes
	\begin{align*}
		\tilde \gamma &\leq \min \left\{ \frac{1}{ L }, \frac{\eta}{128L_{2,\infty} \left(1+\frac{C}{N}\right)}, \frac{\eta}{128(1+\eta)L_\infty } \right\},
	\end{align*}
	and the upper bound becomes
	\begin{align*}
		\min_{q\in \{0,1\ldots, Q-1\}}\normsq{ \nabla f(\rvx_q) } 
		&= \gO\left(  \frac{ F_0 }{\tilde\gamma Q} + \tilde\gamma^2 L_{2,\infty}^2 \frac{1}{\eta^2}\frac{1}{Q} \varsigma^2 +\tilde\gamma^2 L_{2,\infty}^2 \frac{1}{\eta^2N^2\frac{1}{S^2}} \varsigma^2+ \tilde\gamma^2 L_{2,\infty}^2 \frac{1}{\eta^2N^2}C^2\varsigma^2 \right).
	\end{align*}
	Applying Lemma~\ref{lem:step-size}, we get
	\begin{align*}
		&\min_{q\in \{0,1\ldots, Q-1\}}\normsq{ \nabla f(\rvx_q) } \\&=
		\gO\left( \frac{\left( \eta L+L_{2,\infty} \left (1+\frac{C}{N} \right) + (1+\eta)L_{\infty}\right)  F_0}{\eta Q}  + \frac{ (L_{2,\infty}F_0 \varsigma)^{\frac{2}{3}} }{\eta^{\frac{2}{3}}Q} + \left(\frac{L_{2,\infty}F_0 S\varsigma }{\eta NQ}\right)^{\frac{2}{3}} + \left(\frac{L_{2,\infty}F_0 C \varsigma }{\eta NQ}\right)^{\frac{2}{3}}  \right).
	\end{align*} 
	For comparison with other algorithms, we set $\eta = 1$, and get
	\begin{align*}
		&\min_{q\in \{0,1\ldots, Q-1\}}\normsq{ \nabla f(\rvx_q) } \\&=
		\gO\left( \frac{\left( L +L_{2,\infty} \left (1+\frac{C}{N} \right)+ L_{\infty}\right)  F_0}{Q}  + \frac{ (L_{2,\infty}F_0 \varsigma)^{\frac{2}{3}} }{Q} + \left(\frac{L_{2,\infty}F_0 S\varsigma }{ NQ}\right)^{\frac{2}{3}} + \left(\frac{L_{2,\infty}F_0 C \varsigma }{ NQ}\right)^{\frac{2}{3}}  \right).
	\end{align*}
\end{proof}

\section{Experiments}
\label{apx-sec:exps}

In this section, we provide the experimental results of FL on real data sets. Refer to \citet{lu2022grab,cooper2023coordinating} for the experimental results of SGD on real data sets.

\subsection{Setups}

\textit{Algorithms.} We consider the three algorithms in (regularized-participation) FL in the main body: FL-RR, FL-OP and FL-GraB. For FL-OP, its first-epoch permutation is generated randomly; in other words, it corresponds SO in SGD.

\textit{Datasets and models.} We consider the datasets CIFAR-10 \citep{krizhevsky2009learning}, CIFAR-100 \citep{krizhevsky2009learning} and CINIC-10 \citep{darlow2018cinic}. We use the convolutional neural network (CNN) from \citep{acar2021federated} and ResNet-10 \citep{he2016deep}.

\textit{Hyperparameters.} We partition the data examples by the way in \citet{mcmahan2017communication} among $N=1000$ clients, ensuring that each client contains data examples from about one label. We use SGD as the local solver with the learning rate being constant, the momentum being $0$ and
weight decay being $0$. We set the global step size to $\eta = 1$. We set the total number of training rounds to $20000$ (that is, $Q=200$ epochs). For other setups, following those in \citet{wang2022unified}, we set the number of participating clients in each training round to $S=10$, the number of local update steps to $K=5$, the mini-batch size to $16$.

\textit{Two-stage grid search.} We use a two-stage grid search for tuning the step size. Specifically, we first perform a \textit{coarse-grained} search over a broad range of step sizes to identify a best step size at a high level. After that, based on the best step size found, we perform a \textit{fine-grained} search around it by testing neighboring step sizes to find a more precise value. For instance, in the first stage, we can use a grid of $\{10^{-2}, 10^{-1}, 10^{0}\}$ to find the coarse-grained best step size; in the second stage, if the coarse-grained best step size is $10^{-1}$, we use the grid of $\{10^{-1.5}, 10^{-1}, 10^{0.5}\}$ to find the fine-grained best step size. Notably, we tune the step size by the two-stage grid search for FL-RR, and reuse the best step size for the other two algorithms. Tables~\ref{tab:grid search:cnn} shows the processes of the grid searches. We get that the best step size is $10^{-1}=0.1$ for CNN; in the same way, we get that the best step size is $10^{-0.5}\approx 0.316$ for ResNet-10 (the processes are omitted).

\begin{table}[ht]
	\caption{The results of the grid searches for training CNN on various datasets with FL-RR. The best step size is marked with $\ast$. The ``lr'' in the legend means the learning rate or the step size. We use $10^{-1.5} \approx 0.0316$ and $10^{-0.5} \approx 0.316$ as done in \citet{wang2024lightweight}.}
	\label{tab:grid search:cnn}
	\centering{\small
			\resizebox{\linewidth}{!}{
					\begin{tabular}{lllm{0.5\textwidth}}
						\toprule
						{\bf \footnotesize Dataset}  & {\bf \footnotesize Coarse-grained} & {\bf \footnotesize Fine-grained}  &{\bf \footnotesize Result} \\\midrule
						
						CIFAR-10 &$\{0.01, 0.1^\ast, 1.0\}$ &$\{0.0316, 0.1^*, 0.316\}$ &\includegraphics[width=\linewidth]{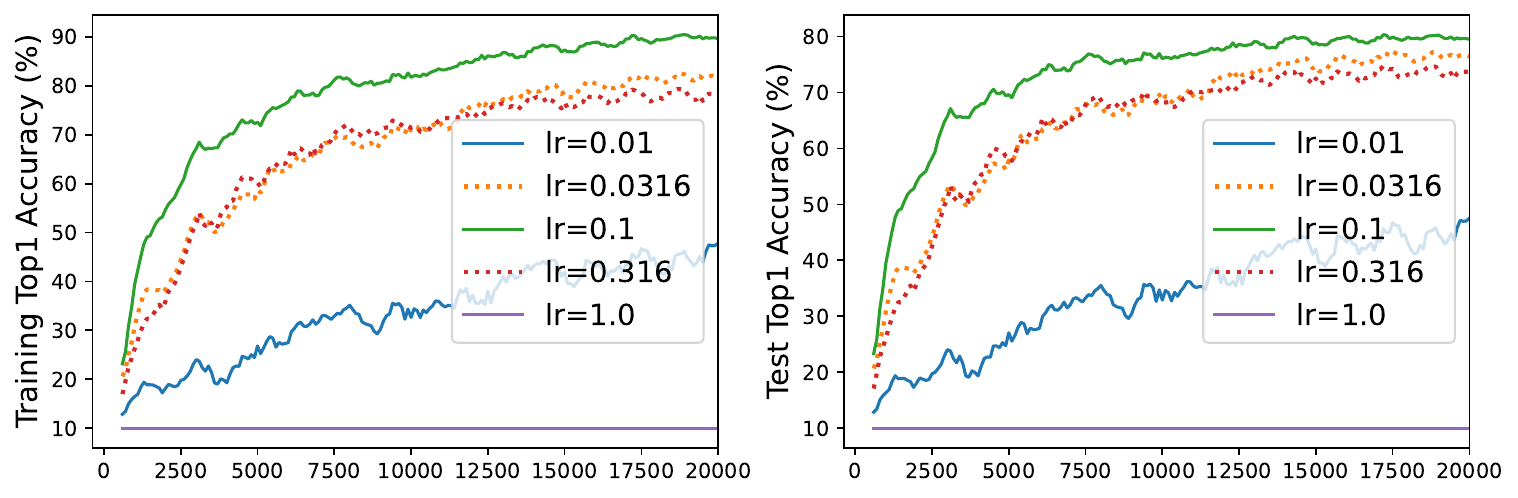} \\\midrule
						
						CIFAR-100 &$\{0.01, 0.1^\ast, 1.0\}$ &$\{0.0316, 0.1^*, 0.316\}$ &\includegraphics[width=\linewidth]{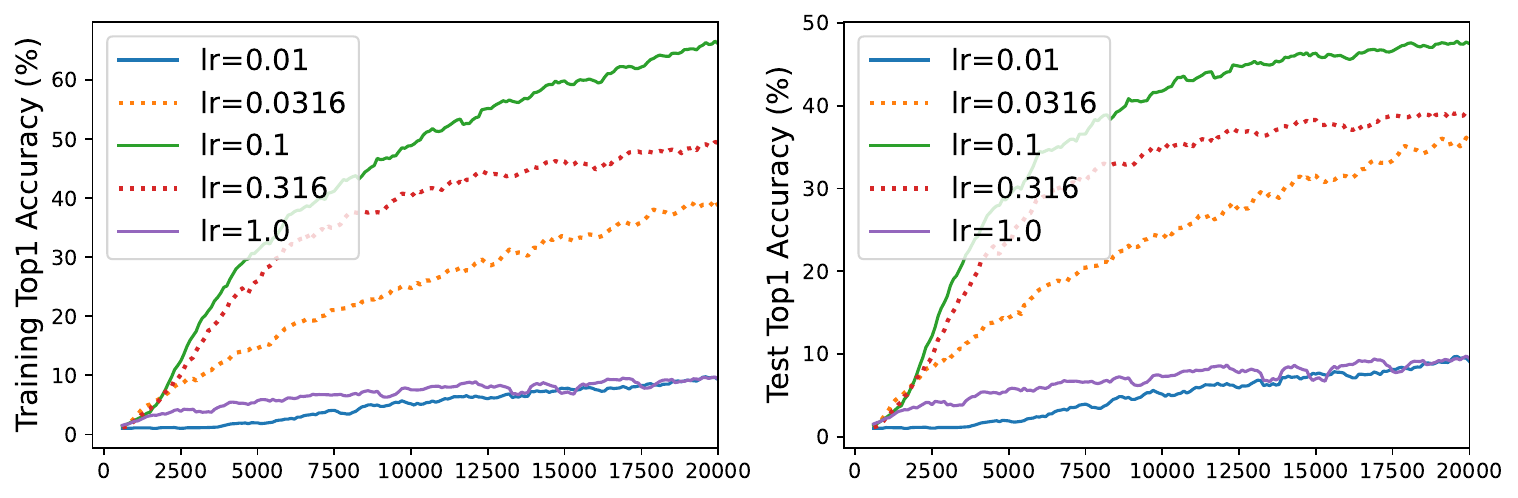} \\\midrule
						
						CINIC-10 &$\{0.01, 0.1^\ast, 1.0\}$ &$\{0.0316, 0.1^*, 0.316\}$ &\includegraphics[width=\linewidth]{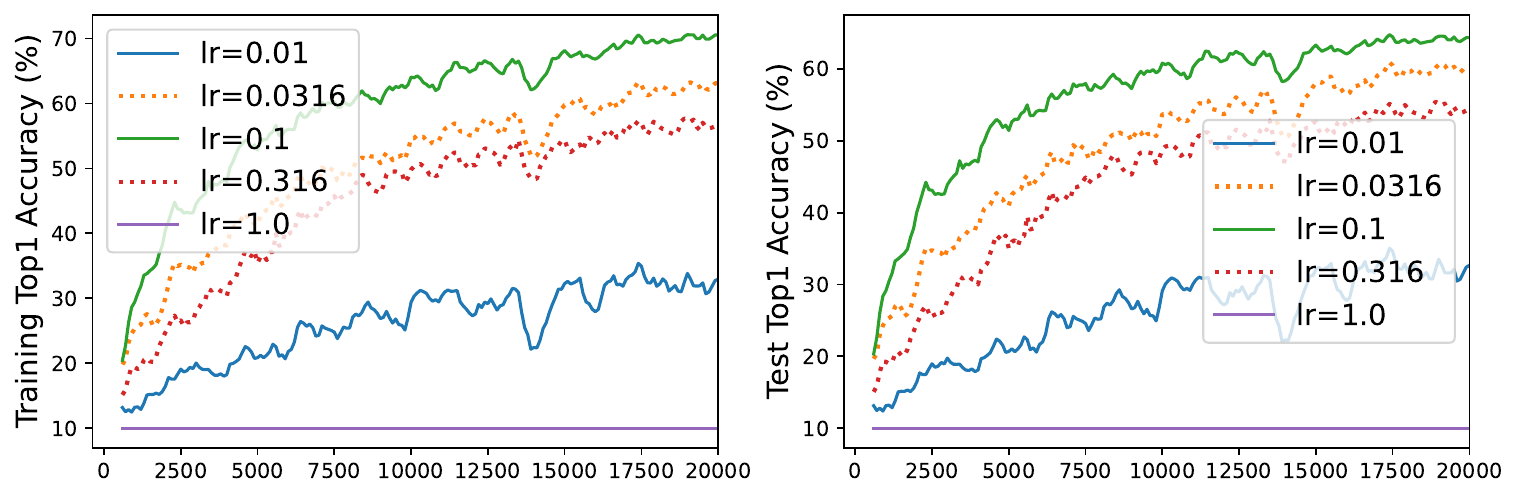} \\
						\bottomrule
					\end{tabular}
			}
		}
\end{table}

\subsection{Experimental Results}

The experimental results are in Figures~\ref{fig:exp-cnn} and \ref{fig:exp-resnetii10}. Some observations are as follows. First, FL-GraB shows the best performance across all tasks, especially in the early stages. This is aligned with our theory that the convergence rate of FL-GraB is the best. Second, FL-OP shows close performance to that of FL-RR on CIFAR-10 and CINIC-10, while it shows worse performance than that of FL-RR on CIFAR-100. This is aligned with our theory that the convergence rate of FL-OP can be the same as that of FL-RR when the change of the parameter is not too large and it will be worse when the change is too large.

\begin{figure}[h]
	\centering
	\includegraphics[width=0.9\linewidth]{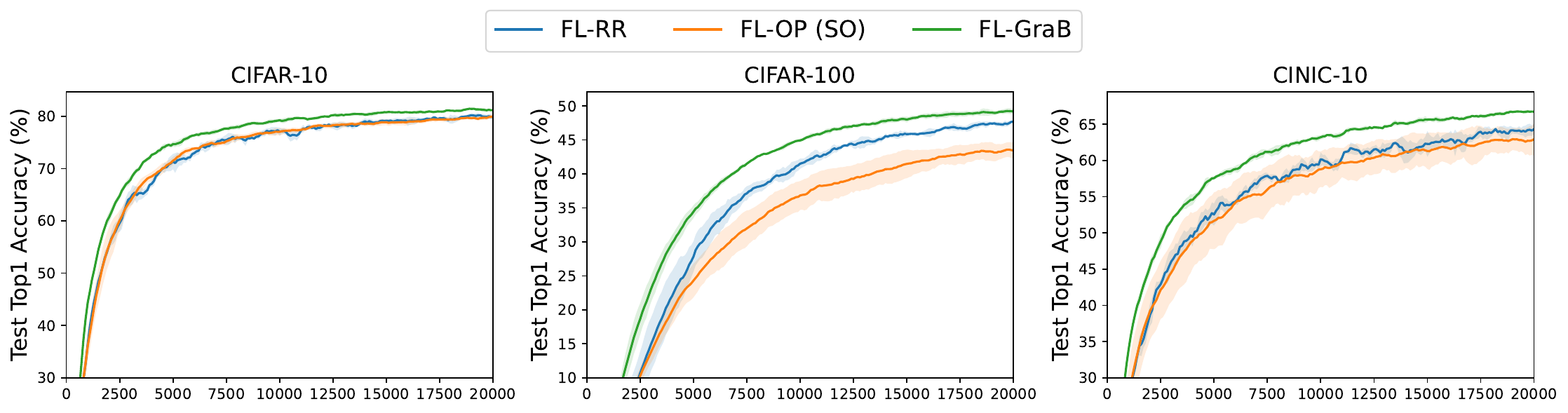}
	\caption{Test accuracy results for training CNN on CIFAR-10, CIFAR-100 and CINIC-10. As done in \citet{wang2022unified}, we applied moving average on the recorded data points with a window length of 6; note that we record the results every 100 rounds (that is, one epoch). The shaded areas show the standard deviation across 5 random seeds.}
	\label{fig:exp-cnn}
\end{figure}

\begin{figure}[h]
	\centering
	\includegraphics[width=0.9\linewidth]{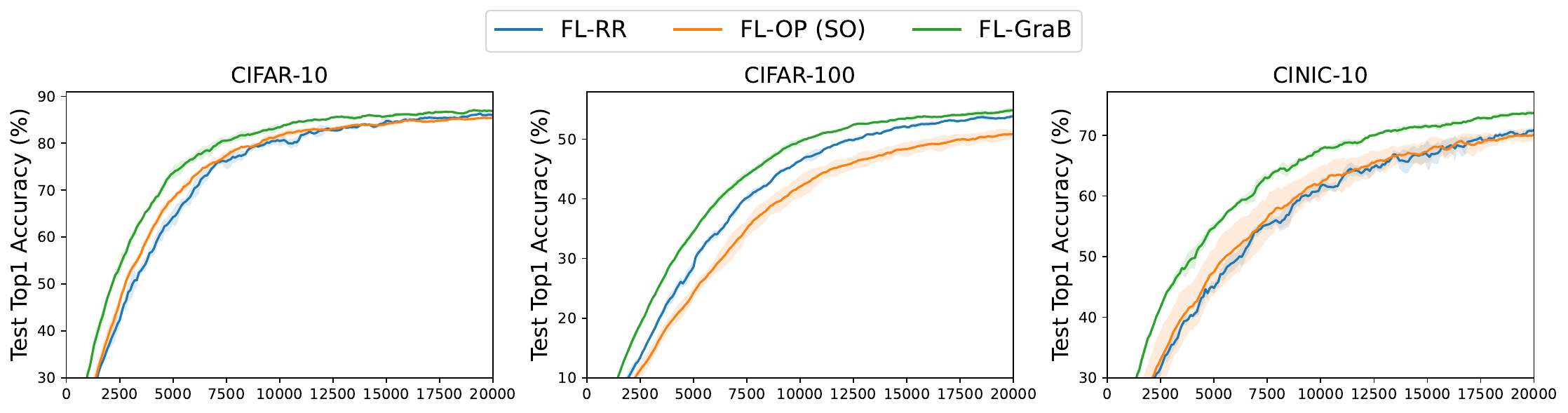}
	\caption{Test accuracy results for training ResNet-10 on CIFAR-10, CIFAR-100 and CINIC-10. As done in \citet{wang2022unified}, we applied moving average on the recorded data points with a window length of 6; note that we record the results every 100 rounds (that is, one epoch). The shaded areas show the standard deviation across 5 random seeds.}
	\label{fig:exp-resnetii10}
\end{figure}

\end{document}